\documentclass[dvips,preprint]{imsart}

\RequirePackage[OT1]{fontenc}
\usepackage{latexsym}
\usepackage{datetime}
\usepackage{alltt}
\usepackage{amsmath,amsthm,amssymb,amsxtra,amsfonts,euscript,bm}
\usepackage{times}
\usepackage{color}

\allowdisplaybreaks

\def\ints{{{\rm Z} \kern -.35em {\rm Z} }}  
\def\smallints{{{\rm Z} \kern -.3em {\rm Z} }}  
\def\pints{{{\rm I} \kern -.15em {\rm N} }}      
\newcommand{\reals}{\mathbb R}

\DeclareMathOperator{\tr}{tr}
\DeclareMathOperator{\Sp}{\mathrm Sp}

\def\cplx{{{\rm I} \kern -.45em {\rm C} }}       
\def\l2{\rm {\mathcal L}^{2}(\reals)}            

\providecommand{\normF}[1]{\left\| #1 \right\|_\mathrm{F}}
\DeclareMathOperator{\ball}{{B}}

\DeclareMathOperator{\ballG}{{{B}_\mathrm{G}}}

\newtheorem{nad}{Notation and Definitions}[section]

\newtheorem{theorem}{Theorem}[section]
\newtheorem{lemma}[theorem]{Lemma}

\newtheorem{proposition}{Proposition}[section]

\newtheorem{example}{Example}

\newcommand{\be}{\begin{equation}}
\newcommand{\ee}{\end{equation}}
\newcommand{\bea}{\begin{eqnarray}}
\newcommand{\eea}{\end{eqnarray}}
\newcommand{\beaa}{\begin{eqnarray*}}
\newcommand{\eeaa}{\end{eqnarray*}}
\newcommand{\bnad}{\begin{nad}}
\newcommand{\enad}{\end{nad}}
\DeclareMathOperator{\diag}{diag}

\newcommand{\shrtexp}{\mathbb{E}}

\newcommand{\di}{{\,\mathrm{d}}}
\newcommand{\latop}[2]{\genfrac{}{}{0pt}{}{#1}{#2}}

\newcommand{\eps}{\epsilon}

\newcommand{\supp}{{\mbox{\rm supp}}}
\newcommand{\dist}{{\mbox{\rm dist}}}

\newcommand{\nin}{\in\!\!\!\!\!/\,}

\newcommand{\sinc}{{\rm sinc\,}}

\renewcommand{\ell}{l}

\def\dG{\mathrm{dist_G}}

\def\area{\mathrm{Area}}
\def\GDd{\mathrm{G}(D,d)}
\def\tt{{t}}
\def\bQ{\mathbf{Q}}

\def\bSigma{\mathbf{\Sigma}}
\def\bx{\mathbf{x}}
\def\by{\mathbf{y}}
\def\bz{\mathbf{z}}
\def\bC{\mathbf{C}}
\def\bH{\mathbf{H}}
\def\bv{\mathbf{v}}
\def\bu{\mathbf{u}}
\def\bU{\mathbf{U}}
\def\b0{\mathbf{0}}
\def\bV{\mathbf{V}}
\def\bW{\mathbf{V}}
\def\bB{\mathbf{B}}

\def\bD{\mathbf{D}}
\def\be{\mathbf{e}}

\def\bX{\mathbf{X}}

\def\rmL{{L}}
\def\ddtp{\frac{\di} {\di t^p}}

\def\rmF{{F}}
\def\rmG{{G}}
\def\rmX{{\cal{X}}}

\def\bz{\mathbf{z}}

\def\Or{{\mathrm{O}}}
\def\Scp{{\mathrm{D}_{+}}}
\def\NScp{{\mathrm{ND}_{+}}}

\DeclareMathOperator*{\argmin}{arg \, min}

\begin{document}

\begin{frontmatter}
\title{{${l_p}$}-Recovery of the Most Significant Subspace
among Multiple Subspaces with Outliers
\protect\thanksref{T1}}
\runtitle{{${l_p}$}-Recovery of the Most Significant Subspace}
\thankstext{T1}{This work was supported by NSF grants
DMS-09-15064 and DMS-09-56072.
It is inspired by our collaboration with Arthur Szlam on
efficient and fast algorithms for hybrid linear modeling, which
apply geometric $\ell_1$ minimization.
We thank the anonymous reviewer for many insightful comments and suggestions that
significantly improved the presentation of this work, John Wright for referring us to~\cite{lp_ransac98, lp_ransac00} as well as for relevant questions which we address in
\S\ref{sec:conclusion} and Vic Reiner, Stanislaw Szarek and J.~Tyler Whitehouse
for commenting on an earlier version of this manuscript. Thanks to
the Institute for Mathematics and its Applications (IMA)
for holding a workshop on
multi-manifold modeling that GL co-organized and TZ participated in.
}

\begin{aug}
\author{\fnms{Gilad} \snm{Lerman}\ead[label=e1]{lerman@umn.edu}}
\and
\author{\fnms{Teng} \snm{Zhang}\ead[label=e2]{zhang620@umn.edu}}
\runauthor{G. Lerman and T. Zhang}
\affiliation{University of Minnesota }
\end{aug}
\address{Department of Mathematics, University of Minnesota\\
127 Vincent Hall, 206 Church Street SE, Minneapolis, MN
55455\\
\printead{e1},
\printead*{e2}
}
\begin{abstract}
We assume data sampled from a mixture of
$d$-dimensional linear subspaces with spherically symmetric distributions within each subspace and an additional outlier component with spherically symmetric distribution within the ambient space
(for simplicity we may assume that all distributions are uniform on
their corresponding unit spheres). We also assume mixture weights for the different components.
We say that one of the underlying subspaces of the model is most significant if its mixture weight is higher than the sum of the mixture weights of all other subspaces.
We study the recovery of the most significant subspace by minimizing the $\ell_p$-averaged
distances of data points from $d$-dimensional subspaces of
$\reals^D$, where $0 < p \in \reals$.
Unlike other $\ell_p$ minimization problems, this
minimization is non-convex for all $p>0$ and thus requires different
methods for its analysis.
We show that if $0<p \leq 1$, then for any fraction of outliers the most significant subspace can be
recovered by $\ell_p$ minimization with overwhelming probability
(which depends on the generating distribution and its parameters).
We show that when adding small noise around the underlying subspaces the most significant subspace can be nearly recovered by $\ell_p$ minimization for any $0<p \leq 1$
with an error proportional to the noise level.
On the other hand, if $p>1$ and there is more than one underlying subspace,
then with overwhelming probability the most significant subspace cannot be
recovered or nearly recovered.
This last result does not require spherically symmetric outliers.
\end{abstract}

%
%
\end{frontmatter}
%

\section{Introduction}
\label{sec:intro}

Principal Component Analysis (PCA) is arguably the most common tool in high dimensional data analysis.
It approximates a given data set by a lower-dimensional subspace obtained from solving an $\ell_2$ optimization problem.
While such an $\ell_2$ minimization can be easily implemented to run fast for moderate-size data, it is not robust to outliers.
That is, the estimated subspace can significantly change when adding points sampled from a very different distribution.
This obstacle motivated the developments of many algorithms for robust PCA, where some of them are based on $\ell_1$ minimization. Their robustness is often theoretically guaranteed when restricting both the distribution and the fraction of outliers.

Here, we study the robustness to outliers of a ``geometric $\ell_1$ minimization'' for subspace recovery. In fact, we discuss the robustness of the following {\em geometric $\ell_p$ minimization} for all $p>0$: For a data set $\rmX \subset \reals^D$, it tries to minimize among all $d$-dimensional subspaces, $\rmL \subseteq \reals^D$, the quantity:
\begin{equation}\label{eq:def_error_1sub} e_{\ell_p}(\rmX,\rmL)=\sum_{\bx \in
\rmX}\dist(\bx,\rmL)^p,\end{equation}
where $\dist(\bx,\rmL)$ denotes the Euclidean distance between a
data point $\bx$ and the subspace $\rmL$.
In this paper, we restrict this minimization to $d$-dimensional linear subspaces (instead of affine), which we refer to as $d$-subspaces.

The geometric $\ell_1$ minimization is related to some of the recent attempts for robust PCA~\cite{Xu2010, Xu2012, robust_mccoy,robust_pca_ZL, LMTZ2012}. However, it is hard to implement it directly since it is not convex (the set of $d$-subspaces, over which the $\ell_1$ energy is minimized, is not convex). Nevertheless, the question of its robustness is fundamentally interesting. While the analysis in~\cite{lp_recovery_part2_11} implies such robustness when restricting the fraction of outliers, here we ask a more challenging question for the recovery of a single subspace: Can it be recovered by a sufficiently large sample when having no restriction on the fraction of outliers but on their distribution? One possible instance is when the outliers are spherically symmetric, i.e., invariant to rotations (or for simplicity uniformly distributed on the sphere). We make the problem more interesting by assuming points sampled from several multiple subspaces and outliers (where the distributions of all components are spherically symmetric) and we study the recovery of the most significant subspace by geometric $\ell_1$ (or $\ell_p$) minimization.

\subsection{The Most Significant Subspace and its Difference from the Global $\boldsymbol{\ell_0}$ Subspace}
\label{sec:mostsignificant}

Ideally one may wish to recover the global $\ell_0$ subspace, that is, the subspace with the largest number of points,
by geometric $\ell_1$ minimization (or $\ell_p$ geometric minimization with any $p \leq 1$).
This will be a nice geometric generalization of the well-known results of basis pursuit,
where $\ell_1$ minimization can be used to solve an $\ell_0$ minimization under some conditions~\cite{candes_romberg_tao06,donoho_most_large1,donoho_most_large2,candes_romberg_tao_cpam06}.

However, there is a crucial difference between the two problems. In basis pursuit one tries to recover the support of a finite sparse vector and there is a uniform positive lower bound on the distances of all possible support vectors. In our geometric setting we try to recover $d$-subspaces and we do not have any restriction on the relative orientation of the underlying subspaces of our model; therefore two subspaces in our model can be arbitrarily close to each other. Unlike the $\ell_0$ energy (that is, number of points on the complement of a given subspace), the $\ell_p$ energy with $p>0$ is a continuous function of the vector of the following elements $\{ \dist(\bx,\rmL) \}_{\bx \in \rmX}$.
Therefore, any two arbitrarily close subspaces can be perceived as the same one with respect to this energy and when uniting the two subspaces
one can get an ``approximate global $\ell_0$ subspace''.
To clarify this point, let us assume for simplicity that
$\rmL^*_1$, $\rmL^*_2$ and $\rmL^*_3$ are $d$-subspaces in $\reals^D$,
where $40\%$ of the points are on $\rmL^*_1$, $30\%$ on $\rmL^*_2$ and $30\%$ on $\rmL^*_3$.
Clearly $\rmL^*_1$ is the global $\ell_0$ subspace. However, if $p > 0$ is fixed and $\rmL^*_2$ and $\rmL^*_3$ are sufficiently close to each other, then $\ell_p(\rmX,\rmL^*_2)<\ell_p(\rmX,\rmL^*_1)$ and thus $\rmL^*_1$ is not the global $\ell_p$ subspace. Indeed, since $\rmL^*_2$ and $\rmL^*_3$ are sufficiently close to each other, we may identify their union as the ``approximate global $\ell_0$ subspace'' with $60\%$ of the points.

As opposed to this example, we will not talk about the exact number of points on a subspace (or around it in a noisy setting), but assume an i.i.d.~sample from a mixture measure of $K+1$ components: $K$ of them along $d$-subspaces $\{\rmL^*_i\}_{i=1}^K$  with weights $\{\alpha_i\}_{i=1}^K$ and another component of outliers with weight $\alpha_0$; more details of the distributions themselves are in \S\ref{subsec:precise}.
We say that $\rmL^*_1$ is the {\em most significant subspace} if
\begin{equation}\label{eq:cond_alpha}  \alpha_1 > \sum_{i=2}^K \alpha_i. \end{equation}
Unlike the condition of the global $\ell_0$ subspace, which translates here to
$\alpha_1 > \max_{i=2}^K \alpha_i$, condition~\eqref{eq:cond_alpha}
is still valid for $\ell_p$ subspace recovery if $\{\rmL^*_i\}_{i=2}^K$ are arbitrarily close to each other.

\subsection{Background and Related Work}
\label{subsec:background}

The $\ell_1$ norm has been widely used to form robust statistics~\cite{huber_book, robust_stat_book2006, Rousseeuw_stat_book2}.
The early principle of least absolute deviations for robust regression minimizes the sum of absolute values of residuals. For example, in linear regression it minimizes the sum of the absolute values of the deviations of the dependent variable observations from the fitted linear estimator based on the independent variable observations. It is a natural robust alternative for the least squares regression and actually emerged independently of least squares regression (see e.g., historical review in~\cite{Harter_review1, Harter_review2, dodge_L1_87}).

The sum of absolute values of residuals can also be used in total
regression problems, where observational errors of both dependent and
independent variables are taken into account. This is a robust
alternative for the total least squares problem, which can be described
geometrically as minimizing~\eqref{eq:def_error_1sub} with $p=2$.
The robust version with sum of absolute values is equivalent to
minimizing~\eqref{eq:def_error_1sub} with $p=1$.
Osborne and Watson~\cite{osborne_watson85},
Sp\"{a}th and Watson~\cite{spath_watson1987} and Nyquist~\cite{Nyquist_l1_88}
suggested a procedure for solving the latter minimization problem over hyperplanes, that is, when the codimension of the subspaces is 1 (see also~\cite{Bargiela_Hartley93}).
Watson~\cite{watson2001orth_l1, Watson02} even suggested an orthogonal $\ell_1$ procedure for fitting a surface to data.
David and Semmes~\cite{DS91} proposed the minimization of~\eqref{eq:def_error_1sub}
for $p \geq 1$ for a pure analytic setting, which is free of outliers. In the context of machine learning and data mining,
Ding et al.~\cite{Ding+06} proposed the minimization
of~\eqref{eq:def_error_1sub} with $p=1$ as rotation-invariant robust
PCA. They also proposed a numerical strategy for approximating a minimizer of~\eqref{eq:def_error_1sub} when $p=1$, but without valid theoretical guarantees for such an approximation.
Zhang et al.~\cite{MKF_workshop09} have formulated an online procedure for this minimization, which can even approximate data by multiple subspaces.

In~\cite{lp_recovery_part2_11}, which followed this work, we analyzed the recovery of
all underlying subspaces within outliers by minimizing a modified version of~\eqref{eq:def_error_1sub} (adapted to multiple subspaces). In that work, the outlier distribution is rather general, but the fraction of outliers is restricted. Since we continued developing the current work, it includes improved estimates for some of the constants of~\cite{lp_recovery_part2_11}.

Recently, several convex algorithms for robust PCA (with provable exact recovery) have been suggested~\cite{wright_robust_pca09, Xu2010, Xu2012, robust_mccoy, robust_pca_ZL, LMTZ2012}.
In~\cite{Xu2010, Xu2012, robust_mccoy, robust_pca_ZL, LMTZ2012} the
problem of fitting a subspace to data is translated into fitting a low-rank matrix to a given matrix, whose columns represent the data points, where outliers correspond to grossly corrupted columns.
Both~\cite{robust_pca_ZL} and~\cite{LMTZ2012} propose a convex relaxation
of the minimization in~\eqref{eq:def_error_1sub}. In the case of a pure inliers-outliers model (inliers lie exactly on a subspace and outliers in its complement) that satisfies certain combinatorial conditions, the subspace outputted by either~\cite{robust_pca_ZL} or~\cite{LMTZ2012} is the minimizer of~\eqref{eq:def_error_1sub} when $p=1$.
We also view one of the terms in the energy of~\cite{Xu2010, Xu2012, robust_mccoy} (namely, the sum of $\ell_2$ norms of column vectors) as an analogue of the energy~\eqref{eq:def_error_1sub} when the columns of the corresponding matrix for this term are the orthogonal complement of the data points with respect to the subspace.
In the case of spherically symmetric outliers with no restriction of their fraction,
it is currently unknown if exact recovery is guaranteed for any of the algorithms in~\cite{Xu2010, Xu2012, robust_mccoy, robust_pca_ZL, LMTZ2012}, though we conjecture it is impossible.
On the other hand, we show here that such guarantees exist for the geometric $\ell_1$ minimization.
To make the problem more challenging (so that the underlying subspace cannot be nearly recovered by PCA due to the spherically symmetric outliers), we find it interesting to ask about the geometric $\ell_1$ recovery of the most significant subspace among multiple subspaces within spherically-symmetric outliers.

Hardt and Moitra~\cite{moitra_pca2012} showed that it is small set expansion hard to exactly
recover a $d$-subspace in $\reals^D$ with fraction of outliers larger than $(D-d)/D$ for all scenarios satisfying a rather general combinatorial condition.
They also developed deterministic and random algorithms for achieving subspace recovery that can handle outliers with fraction at most $(D-d)/D$ for data satisfying their combinatorial condition.
Our current work suggests a higher fraction of outliers (arbitrarily close to $100\%$), however, it does not contradict~\cite{moitra_pca2012}. First of all, in the case of a single subspace ($K=1$) the recovery in our work only applies to spherically symmetric outliers and not to all scenarios satisfying the combinatorial condition
of~\cite{moitra_pca2012}.
Second of all, our work verifies exact recovery in probability, while the result of~\cite{moitra_pca2012}
requires deterministic satisfaction of all scenarios of their combinatorial condition.
Third of all, the combinatorial condition of~\cite{moitra_pca2012} may not be satisfied for our setting
when $K>1$, that is, when having multiple subspaces. At last and most importantly, if the small set expansion problem has no efficient algorithm (which is unknown), then the result of Hardt and Moitra~\cite{moitra_pca2012} implies the following fact: Any estimator that can exactly recover a subspace in all settings specified by their combinatorial condition with percentage of outliers larger than $(D-d)/D$ cannot be efficient. Since our optimization problem is non-convex, it is possible that there is no efficient algorithm for approximating it.

There are several other non-convex methods for subspace recovery that seem to work well with high percentages of outliers, in particular, higher than the ones guaranteed for convex methods like~\cite{Xu2010, Xu2012, robust_mccoy, robust_pca_ZL, LMTZ2012} or are highly common among practitioners (without theoretical guarantees) and we thus review some of them.

In the computer vision literature a common procedure for subspace fitting uses
the Random Sample Consensus (RANSAC)~\cite{Fischler81RANSAC} heuristic. In theory,
it may not exactly recover subspaces for any positive ratio of outliers. However, in practice it often
nearly recovers subspaces when the ambient dimension is sufficiently small. It is possible that the random strategy of Hardt and Moitra~\cite{moitra_pca2012}
may serve as a good theoretically-guaranteed alternative for RANSAC.
The RANSAC strategy repeatedly applies the following two steps: 1. randomly select a set
of $d$ independent vectors; 2. count the number of data points
within a strip of width $\eps$ around the $d$-subspace spanned by
those $d$ vectors (both $\eps$ and the number of iterations of these
two steps are parameters set by the user). The final output of this
algorithm is the $d$-subspace maximizing the quantity computed in
step 2.

Torr and Zisserman~\cite{lp_ransac98, lp_ransac00} suggested a
RANSAC-type strategy, which selects a subspace (among the random set of candidates) by minimizing a supposedly robust variant of the $\ell_2$
distance from a subspace. This variant uses the square function
until a fixed threshold and a constant function for larger values.
However, by following the proof of Theorem~\ref{thm:phase_hlm} in this work (in particular, \eqref{eq:monotone})
one can show that when $K>1$ the subspace obtained by the minimizer of this variant is sufficiently far from the most significant subspace with probability 1.

There are non-convex methods for removing outliers (or detecting the hidden low-dimensional structures) that can handle arbitrarily large fraction of outliers.
For example, Arias-Castro et al.~\cite{Arias-Castro05connect} proved that the scan statistics may detect points sampled uniformly from a $d$-dimensional graph in $\reals^D$
of an $m$-differentiable function among uniform outliers in a cube in $\reals^D$ with fraction of order $1-O(N^{-m(D-d)/(d+m(D-d))})$.
Arias-Castro et al.~\cite{higher-order} used higher order spectral clustering affinities
to remove outliers and thus detect differentiable surfaces (or certain
unions of such surfaces) among uniform outliers, whose maximal fraction can be of a similar asymptotic order as that
of the scan statistics.

Soltanolkotabi and Cand{\`e}s~\cite{Soltanolkotabi2012}, which appeared after the online release of the first version of this work, assumed a similar model to the one assumed here but without noise and including another assumption (when $N$ approaches infinity it becomes $d<\frac{1}{96}D$);
they established the exact recovery of all underlying subspaces (and not the most significant subspace) by the sparse subspace clustering (SSC) algorithm~\cite{ssc13}. They also proposed an additional step for removing outliers, which is not convex, and analyzed its performance when $N$ lies in a certain interval.
In fact, they analyzed the part of the SSC algorithm that forms an affinity matrix, where the affinities are obtained via convex optimization. The second part of SSC involves clustering the subspaces by this affinity matrix and is not convex.
Soltanolkotabi et al.~\cite{Soltanolkotabi2013} also analyzed the stability to noise of the first and convex part of a modified version of the SSC algorithm without outliers and when
$d<c_0 D/\log (N)$.

Zhang et al.~\cite{LBF_cvpr10, LBF_journal12} proposed a method for recovering multiple subspaces by globally incorporating information from several local best-fit subspaces. It can also be adapted for finding only the most significant subspace. Nevertheless, the full guarantees for recovering either the most significant subspace or all underlying subspaces have not been established yet.

\subsection{Basic Conventions and Notation}

We denote by $\GDd$ the Grassmannian space, i.e., the set of all
$d$-subspaces of $\mathbb{R}^D$ with a manifold structure. The geodesic distance
between $\rmF$ and $\rmG$ in $\GDd$ is
\begin{equation}
\label{eq:dist_grassman}
\dG(\rmF,\rmG)=\sqrt{\sum_{i=1}^{d}\theta_i^2},
\end{equation}
where $\{\theta_i\}_{i=1}^d$ are the principal angles between $\rmF$ and $\rmG$ (we review these angles
and their relation to geodesics in \S\ref{sec:prin_angles}).
Following \S3.9 of~\cite{Mat95}, we denote by $\gamma_{D,d}$ the ``uniform distribution on $\GDd$''.
We designate a ball in $\GDd$ by $\ballG(\rmL,r)$ as opposed to a Euclidean ball in $\reals^D$, $\ball(\bx,r)$.
We refer to any of the global minimizers of~\eqref{eq:def_error_1sub} among $\rmL \in \GDd$ as a {\it global $\ell_p$ subspace}.
Similarly, local minimizers of~\eqref{eq:def_error_1sub} among $\rmL \in \GDd$ are {\it local $\ell_p$ subspaces}.

We use ``w.p.'' as a shorthand for ``with probability''.
By saying ``with overwhelming probability'', or in short ``w.o.p.'', we mean that the
underlying probability is at least $1-Ce^{-N/C}$, where $N$ is the size of the data set $\rmX$ and $C$ is a
constant independent of $N$.
When using this terminology we will make sure to estimate the asymptotic dependence of $C$ on $D$ and $d$ and
use it to infer the asymptotic dependence of the minimal sample size $N$ on $D$ and $d$; this way we make sure that
the probabilistic estimate is not completely useless.

\subsection{Setting of This Paper}
\label{subsec:precise}

We assume $K$ distinct $d$-subspaces in $\reals^D$, which we denote by $\{\rmL^*_i\}_{i=1}^K$.
Furthermore, we assume an i.i.d.~data set $\rmX \subseteq \reals^D$ of size $N$
sampled from a mixture distribution $\mu_\epsilon$ with components supported on each of the $d$-subspaces
$\{\rmL^*_i\}_{i=1}^K$ as well as an outlier component and noise level $\eps \geq 0$.
Our typical setting assumes spherically symmetric distributions within $\{\rmL^*_i\}_{i=1}^K$ and (for most of the discussion) spherically
symmetric distribution of the outliers (within $\reals^D$).

For simplicity of our presentation we replace spherically symmetric distributions with uniform distributions on the sphere, though
our analysis can be easily extended to the former distributions.
Furthermore, one can always normalize the data to the sphere so that spherically symmetric distributions (or even more general distributions)
are mapped to uniform distributions onto the sphere. Normalization of data to the unit sphere is a common practice
for robust PCA algorithms~\cite{sphericalPCA99, LMTZ2012} as well as algorithms for modeling data by multiple subspaces~\cite{Yan06LSA, MKF_workshop09}.

In the noiseless case ($\eps=0$), we denote the $K+1$ components of the mixture measure by $\{\mu_i\}_{i=0}^K$, where $\mu_0$ is the uniform
distribution on $\mathbb{S}^{D-1}$ (the $(D-1)$-dimensional unit sphere) that represents outliers and
for $1\leq i\leq K$, $\mu_i$ is the uniform distribution on $\mathbb{S}^{D-1}\cap \rmL^*_i$.


For the noisy case, we assume that $\{\mu_i\}_{i=1}^K$ are contaminated by the noise distributions $\{\nu_{i,\eps}\}_{i=1}^K$
such that $\supp(\mu_{i} + \nu_{i,\eps})\subseteq \mathbb{S}^{D-1}$ (that is, all points sampled from this noisy distribution also lie on the unit sphere),
and for technical reasons we assume that the $p$th moments of $\{\nu_{i,\eps}\}_{i=1}^K$ are smaller than $\eps^p$ for all $p \leq 1$ (when considering geometric $\ell_p$ minimization with $p \geq 1$ we only need this condition with $p=1$ and when considering geometric $\ell_p$ minimization with $p < 1$ we only need this condition with the relevant value of $p$). If $\epsilon=0$, then the latter model is consistent with the former one by letting $\{\nu_{i,0}\}_{i=1}^K$ be the Dirac $\delta$ distributions at $\b0$.

For any noise level $\eps \geq 0$, the mixture distribution $\mu_\epsilon$ has the form
\begin{equation}
\label{eq:mu_eps}
\mu_\eps=\alpha_0 \mu_0 + \sum_{i=1}^{K}\alpha_i (\mu_{i} + \nu_{i,\eps}),
\end{equation}
where
$\alpha_0 \geq 0$, $\alpha_i >0 \ \forall \, 1 \leq i \leq K$ and $\sum_{i=0}^K \alpha_i =1$.
If $\eps=0$, then for convenience we replace the notation $\mu_\eps$ by $\mu$, i.e.,
\begin{equation}
\mu = \sum_{i=0}^{K}\alpha_i \mu_{i}.
\end{equation}

We refer to $\mu_\eps$ created according to this model as {\em spherically uniform HLM (hybrid linear modeling) measure with noise level
$\eps$} (sometimes we also add ``w.r.t.~$\{\rmL_i^*\}_{i=1}^K$'').
In part of our setting, the assumption on $\mu_0$ can be completely removed, while still assuming that
$\{\mu_i\}_{i=1}^K$ are the same.
In this case we refer to $\mu_\eps$ as {\em weakly spherically uniform HLM measure with noise level $\eps$}.

\subsection{Mathematical Problems of This Paper}
We address here two mathematical problems. The simpler one is implicit in this introduction, though clear from the proofs. It asks whether the most significant subspace  $\rmL^*_1$ can be recovered when $\eps = 0$ by minimizing $\shrtexp_\mu (\dist^p(\bx, \rmL))$ over all $\rmL \in \GDd$.
The main problem can be formulated using the empirical distribution $\mu_N$ of i.i.d.~sample of size $N$ from $\mu$.
It asks whether $\rmL^*_1$ can be recovered (w.o.p.) by minimizing $\shrtexp_{\mu_N} (\dist^p(\bx, \rmL))$, which is equivalent to minimizing~\eqref{eq:def_error_1sub}. In the noisy case, we extend these problems to near recovery.

\subsection{Main Theorems}

In the noiseless case and $0 < p \leq 1$, we can exactly recover the most significant subspace by $\ell_p$ minimization as follows.
\begin{theorem}
\label{thm:hlm} If $\mu$ is a spherically uniform HLM measure on $\reals^D$
with $K$ $d$-subspaces $\{\rmL^*_i\}_{i=1}^{K} \subset\GDd$ and
mixture coefficients $\{\alpha_i\}_{i=0}^{K}$
satisfying~\eqref{eq:cond_alpha}, $\rmX$ is a data set of $N$ points
identically and independently sampled from $\mu$ and $0 < p \leq 1$, then the
probability that $\rmL^*_1$ is a global $\ell_p$ subspace is at least
$1-C'\exp(-N/C)$, where $C$ and $C'$ are constants depending on $D$, $d$, $K$,
$p$, $\alpha_0$, $\alpha_1$, and $\min_{2\leq i\leq
K}(\dG(\rmL^*_1,\rmL^*_i))$.
The asymptotic dependence of $C$ and $C'$ on $d$ and $D$ (when the rest of the parameters are fixed)
can be expressed as follows: $C=O(d^{\max(13p,2)}D^{3p})$ and $C'=O(d^{d(d+1)/2}+d^{6.5d(D-d)}D^{1.5d(D-d)})$.
\end{theorem}

The theorem guarantees exact recovery of $\rmL_1^*$ w.o.p.~for any percentage of outliers $\alpha_0<1$.
However the probability of this event depends (through the constants $C$ and $C'$) on the model parameters.
Due to the non-convexity of the underlying minimization, it is too complicated to estimate the parameters $C$ and $C'$, even for very special cases.
However, the theorem states their asymptotic dependence on $d$ and $D$, which is later verified in \S\ref{sec:dependence_on_D}.
We also show in \S\ref{sec:dependence_on_D} that these estimates imply that $N= \Omega(d^{\max(13p,2)+1}D^{3p}\max(D-d,d+1)\log(D))$\footnote{We recall that $f=\Omega(g)$ if and only if $g=O(f)$.}.
This indicates some unnecessary oversampling for the single subspace recovery, but we believe that we may improve this estimate. Nevertheless, we currently view this estimate as ``a sanity check'' ensuring that the minimal $N$ has polynomial dependence on $D$ and $d$, where the polynomial in $D$ is of low order and the
polynomial in $d$ is of moderate order at most.

Even though we cannot fully estimate the probability for a global minimum, we can still estimate the probability that $\rmL_1^*$ is a local minimum when $K=1$.
For example, it follows from Theorem~\ref{thm:dist_1subs_loc1} (which appears later in \S\ref{sec:local_global_lp}) that if there are $q N$ i.i.d.~samples from $\mu_1$ and
$(1-q) N$ i.i.d.~samples from $\mu_0$, then $\rmL_1^*$ is a local $\ell_1$ subspace with probability at least
\[ 1-2d^2\exp\left(-\frac{q\cdot N}{8.01 \cdot d^4} \right)-2dD
\exp\left(-\frac{q^2\cdot N}{8 \cdot (1-q)\cdot d^4 \cdot  D  }\right)\,. \]
We further discuss this estimate in \S\ref{sec:local_global_lp}.

In the noisy case, exact asymptotic recovery is not possible in general (as we explain in \S\ref{sec:no_asymptotic}), but we can extend the above formulation to near recovery.
\begin{theorem}
\label{thm:noisy_hlm} If $\eps >0$, $\mu_{\eps}$ is a spherically uniform
HLM measure on $\reals^D$ of noise level $\eps$  with $K$
$d$-subspaces $\{\rmL^*_i\}_{i=1}^K \subset\GDd$
and mixture coefficients
$\{\alpha_i\}_{i=0}^{K}$ satisfying~\eqref{eq:cond_alpha}, $\rmX$ is a data set of $N$ points sampled
identically and independently from $\mu_\eps$ and $0<p \leq 1$, then the global
$\ell_p$ subspace for $\mu_\eps$ is in the ball $\ballG(\rmL^*_1,f)$, where
\begin{equation}f\equiv
f(\epsilon,K,d,p,\alpha_0,\alpha_1,\mu_1)=\frac{ \sqrt{d+p} \cdot \pi^{\frac{2p+1}{2p}}\cdot 4.55^{\frac{1}{p}} \cdot \eps}
{\left(\alpha_0 + 2 \cdot \alpha_1 - 1 \right)^{\frac{1}{p}}\cdot 2^{\frac{3}{2}} }\,,\label{eq:definition_of_f}
\end{equation}
w.p.~at least \begin{equation}1-\exp(-N\epsilon^{2p}/2)(C_2\sqrt{d})^{d(D-d)/p}/(2\epsilon^p)^{d(D-d)}.\label{eq:thm2_prob}\end{equation} 

If $K=1$, then the above statement extends to $1<p<\infty$ with
\begin{equation}
\label{eq:definition_of_f2}
f\equiv
f(\eps,K,d,p,\alpha_1, \mu_1)
=\begin{cases}\frac{ \sqrt{d+p} \cdot \pi^{\frac{2p+1}{2p}}\cdot 4.55^{\frac{1}{p}} \cdot \eps^{\frac{1}{p}}\cdot p^{\frac{1}{p}}}
{\left(\alpha_0 + 2 \cdot \alpha_1 - 1 \right)^{\frac{1}{p}}\cdot 2^{\frac{3}{2}} }, &\text{if $1< p\leq 2$};\\
 \sqrt{d} \cdot (4\eps)^{\frac{1}{p}}\cdot \pi/2, &\text{if $p>2$}\end{cases}
\end{equation}
and probability $1-\exp(-Np^2 \epsilon^2/2) (C_2\sqrt{d})^{d(D-d)/p}/( 2 p \epsilon)^{d(D-d)}.$
\end{theorem}

We note that
if $f \geq \frac{\pi\sqrt{d}}{2}$, then all principle angles are at most
$\pi/2$ and thus $\ballG(\rmL^*_1,f)=\GDd$. The theorem is thus only interesting when $\eps$ is sufficiently small, in particular, when it satisfies the following bound, which ensures that $f<\frac{\pi\sqrt{d}}{2}$:
\begin{equation}\label{eq:eps_low_bound1}
\eps<
\begin{cases}
\frac{\sqrt{2d}\cdot\left(\alpha_0 + 2 \cdot \alpha_1 - 1 \right)^{\frac{1}{p}}}{ \sqrt{d+p} \cdot \pi^{\frac{1}{2p}}\cdot 4.55^{\frac{1}{p}}},
&\text{if $p\leq 1$};\\
\frac{(2d)^{\frac{p}{2}}\cdot\left(\alpha_0 + 2 \cdot \alpha_1 - 1 \right) }{ (d+p)^{\frac{p}{2}} \cdot \sqrt{\pi}\cdot 4.55\cdot p}
,
&\text{if $1<p\leq 2$ and $K=1$;}\\
\frac{1}{4},
&\text{if $p >  2$ and $K=1$.}
\end{cases}
\end{equation}

At last, we formulate the impossibility of $\ell_p$
recovery when $p>1$ and $K>1$ and thus demonstrate a phase
transition at $p=1$ when $K>1$. This result does not require $\mu_0$ to be uniform on the sphere (or spherically symmetric).
\begin{theorem}
\label{thm:phase_hlm} Assume that $\{\rmL^*_i\}_{i=1}^{K}$ are $K$
$d$-subspaces in $\reals^D$, which are identically and independently distributed
according to $\gamma_{D,d}$. For each $\eps \geq 0$ and a random
sample of $\{\rmL^*_i\}_{i=1}^{K}$, let $\mu_{\eps}$ be a weakly spherically uniform
HLM measure on $\reals^D$ (w.r.t.~$\{\rmL^*_i\}_{i=1}^K$) of noise level $\eps$  and let $\rmX$ be a data set of $N$ points sampled
identically and independently from $\mu_\eps$. If $K>1$ and $p> 1$, then for almost
every $\{\rmL^*_i\}_{i=1}^{K}$ (w.r.t.~$\gamma_{D,d}^K$), there exist
positive constants $\delta_0$ and $\kappa_0$,
independent of $N$, such that for any $0 \leq \epsilon<\delta_0$ the
global $\ell_p$ subspace of $\rmX$ is not in the ball
$\ballG\left(\rmL^*_1,\kappa_0\right)$ with overwhelming probability.
\end{theorem}
The overwhelming probability of Theorem~\ref{thm:phase_hlm}
is not of practical interest, but for completeness
we specify it later in \eqref{eq:failure_prob}.
More importantly, in \S\ref{sec:delta_kappa} we provide estimates for $\delta_0$ and $\kappa_0$, which are independent of $\epsilon$.
They require some technical definitions, which we would rather avoid here. Instead, we exemplify them for the special case where $K=2$, $d=1$, $D=2$ and $\mu_1$ and
$\mu_2$ are uniform distributions on line segments centered on the origin and of length $2$. Denoting by $\theta$ the angle between $\rmL^*_1$ and $\rmL^*_2$, the analysis in \S\ref{sec:delta_kappa} implies the following lower bound for both $\kappa_0$ and $\delta_0$ in this special case:
\begin{equation}\label{eq:gamma1_example}
\delta_0, \kappa_0 \geq \begin{cases}
\frac{1}{8(p+1)^2}\cdot \alpha_2^2 \cdot \cos^2(\theta)\cdot \sin^{2(p-1)}(\theta), &\text{if $p\geq 2$;}
\\
2^\frac{p-4}{p-1}(p-1)p^\frac{1}{p-1}(p+1)^\frac{p-1}{p}\cdot\alpha_2^\frac{p}{p-1}\cdot \sin^p(\theta)\cdot
\cos^\frac{p}{p-1}(\theta),    &\text{if $1< p< 2$}.
\end{cases}\end{equation}

These lower bounds for $\delta_0$ and $\kappa_0$ approach zero when $\alpha_2$ approaches zero or when $\theta$ approaches 0 or $\pi/2$. We expect such a behavior since if $\alpha_2=0$, $\theta =0$ or $\theta=\pi/2$, then for any $p>1$, $\rmL_1^*$ is the unique global $\ell_p$ minimizer w.o.p.
We also comment that these bounds are not sharp (in particular, their discontinuity at $p=2$ is artificial).

\subsection{Relevance of Theory}
\label{sec:implications}

As discussed in \S\ref{subsec:background}, the geometric $\ell_1$ minimization is a prototype
for other robust and convex PCA algorithms~\cite{Xu2010, Xu2012, robust_mccoy,robust_pca_ZL, LMTZ2012}.
Without any control on the fraction of outliers, no guarantees are known for the exact recovery of the other algorithms.
We thus find it interesting to analyze the robustness of the geometric $\ell_1$ minimization
to spherically uniform outliers (or spherically symmetric outliers) with no restriction of their fraction and with possibly other underlying subspaces.
It is also interesting for us to quantify the phase transition of exact recovery at $p=1$ (different phase transitions at $p=1$ and $p=0$ are discussed later in \S\ref{sec:p1_min} and \S\ref{sec:mostsignificant} respectively).
The analysis of the geometric $\ell_1$ minimization of this paper has inspired the different analysis of~\cite{robust_pca_ZL, LMTZ2012}
and is also directly used in~\cite{lp_recovery_part2_11}.
Nevertheless, our setting is non-convex and we are not aware of efficient
and theoretically guaranteed strategies to approximate the global minimizer.
It is possible that the ability to theoretically recover the global minimizer with an arbitrarily large fraction of outliers is closely related to the possible inefficiency of any algorithm that aims to compute this minimizer (see \S\ref{sec:non_convex}).

\subsection{Additional Results and Structure of the Paper}

Additional theory is reviewed in \S\ref{sec:more_theory}. In particular,
\S\ref{sec:combinatorics} establishes some necessary and sufficient deterministic conditions for a $d$-subspace to be a local $\ell_p$ minimizer for a given data set;
\S\ref{sec:local_global_lp} uses these conditions to show that if one samples $N_0$ i.i.d.~outliers from $\mu_0$ and $N_1$ i.i.d.~inliers from $\mu_1$ and if $N_0=o(N_1^2)$, then the global $\ell_0$ subspace (which is also the most significant subspace in this case) is a local $\ell_1$ subspace. On the other hand, it shows that in a general setting of a single underlying subspace with outliers, the global $\ell_0$ subspace is a local $\ell_p$ subspace w.p.~0 when $p>1$ and w.p.~1 when $0<p<1$;
\S\ref{sec:counter} demonstrates natural instances,
distinct from the case of spherically uniform outliers (or spherically symmetric outliers), where the most significant subspace is
neither a local $\ell_p$ subspace (even for $p=1$) nor global one
(even for $0<p<1$).
We separately include all mathematical details verifying the theory of this paper
in \S\ref{sec:proofs}, while leaving some auxiliary verifications
to the appendix.
At last, \S\ref{sec:conclusion} concludes this paper and discusses extensions of
its results as well as open problems.

\section{Additional Theory}
\label{sec:more_theory}

\subsection{Combinatorial Conditions for $\boldsymbol{\ell_0}$ Subspaces Being Local $\boldsymbol{\ell_p}$ Subspaces}
\label{sec:combinatorics}

\subsubsection{Preliminary Notation}
We denote the orthogonal
group of $n\times n$ matrices by $\Or(n)$ and the semigroup of
$n\times n$ nonnegative diagonal matrices by $\Scp(n)$. We designate the projection from
$\mathbb{R}^D$ onto the $d$-subspace $\rmL$ by $P_{\rmL}$ and
the corresponding orthogonal projection by $P^\perp_{\rmL}$. We represent them by $d \times D$ and $(D-d) \times D$ matrices respectively.
Only in few places in the text we use $D \times D$ matrix representations instead and thus denote them by
$\hat{P}_{\rmL}$ and $\hat{P}^\perp_{\rmL}$ instead (where $P_{\rmL}^T P_{\rmL} =\hat{P}_{\rmL}$ and
${P^\perp_{\rmL}}^T {P^\perp_{\rmL}} =\hat{P}^\perp_{\rmL}$).
The nuclear norm of
$\mathbf{A}$, which is denoted by $\|\mathbf{A}\|_{*}$, is the sum of singular values of $\mathbf{A}$.
We define the scaled outlying ``correlation'' matrix
$\mathbf{B}_{\rmL,\rmX}$ of a data set $\rmX$ and a $d$-subspace
$\rmL$ as follows
\begin{equation}\label{eq:def_B}\mathbf{B}_{\rmL,\rmX}=\sum_{\bx \in
\rmX\setminus \rmL}P_{\rmL}(\bx)P^\perp_{\rmL}
(\bx)^{T}/\dist(\bx,\rmL). \end{equation}%
That is, unlike the covariance matrix, which sums over all data points the rank one matrices $\bx \bx^{T}$,
$\mathbf{B}_{\rmL,\rmX}$ sums over all outlying data points (i.e., $\bx \in \rmX$ not lying on $\rmL$), the
restriction of $\bx \bx^{T}$ to matrices with column space in $\rmL$ and row space in the orthogonal complement of $\rmL$, while scaling this product by the distance of $\bx$ to $\rmL$, i.e., $\|P^\perp_{\rmL}(\bx)\|$, where throughout the paper $\| \cdot \|$ denotes the Euclidean norm.

We exemplify $\mathbf{B}_{\rmL,\rmX}$ for a typical counterexample of robust recovery, which we discuss later in \S\ref{sec:counter}.
\begin{example}
\label{example:calculate_B}
Let $D=2$, $d=1$, $\bz=(t_0\cos(\theta_0),t_0\sin(\theta_0))^T$, where $t_0>0$ and $0< \theta_0 \leq \frac{\pi}{2}$ and $\rmX=\{(a_1,0)^T,(a_2,0)^T,\cdots,(a_{N_1},0)^T,\bz\}$.
That is, $\rmX$ is a set of $N_1+1$ points, where $N_1$ of them lie
on the $x$-axis with magnitudes $\{|a_i|\}_{i=1}^{N_1}$ and one of them has an angle $\theta_0$ with the $x$-axis and magnitude $t_0$. We denote the $x$-axis by $\rmL_x$ and the line passing through the origin and $\bz$ by $\rmL_{\bz}$.

We note that
\begin{align}
&\mathbf{B}_{\rmL_x,\rmX}=\sum_{\bx\in\rmX\setminus \rmL_x}
P_{\rmL_x}(\bx) \, P^\perp_{\rmL_x} (\bx)^{T} \dist(\bx,\rmL_x)^{-1}\nonumber\\&=
P_{\rmL_x}((t_0\cos(\theta_0),t_0\sin(\theta_0))^T) \, \frac{ P^\perp_{\rmL_x}((t_0\cos(\theta_0),t_0\sin(\theta_0))^T)^{T}}{\dist((t_0\cos(\theta_0),t_0\sin(\theta_0))^T,\rmL_x)}
\nonumber\\&=t_0\cos(\theta_0) \, t_0\sin(\theta_0)/t_0\sin(\theta_0)= t_0\cos(\theta_0)\label{eq:example}
\end{align}
and
\begin{align}
&\mathbf{B}_{\rmL_{\bz},\rmX}=\sum_{\bx\in\rmX\setminus \rmL_{\bz}}
P_{\rmL_{\bz}}(\bx) \, P^\perp_{\rmL_{\bx}} (\bx)^{T} \dist(\bx,\rmL_{\bz})^{-1}\nonumber\\&=
\sum_{i=1}^{N_1}P_{\rmL_{\bz}}((a_i,0)^T) \, P^\perp_{\rmL_{\bz}}((a_i,0)^T)^{T}/\dist((a_i,0)^T,\rmL_{\bz})
\nonumber\\&=\sum_{i=1}^{N_1}a_i\cos(\theta_0)\, a_i \sin(\theta_0) / |a_i\sin(\theta_0)|= \cos(\theta_0)\sum_{i=1}^{N_1}|a_i|.\label{eq:example2}
\end{align}

\end{example}

\subsubsection{Conditions for a Local \boldsymbol{$\ell_p$} Minimizer}
\label{sec:conditions}
We formulate conditions for an arbitrary $d$-subspace $\dot{\rmL}$ to be a local
$\ell_p$ subspace, while distinguishing between three cases: $p=1$,
$0<p<1$ and $p>1$. 
\begin{theorem}
\label{thm:comb_cond} If $\dot{\rmL} \in \GDd$, $\rmX_1=
\{\bx_i\}_{i=1}^{N_1}\subset \dot{\rmL}$, $\rmX_0 =
\{\mathbf{y_i}\}_{i=1}^{N_0}\subset \mathbb{R}^D\setminus \dot{\rmL}$ and
$\rmX=\rmX_0 \cup \rmX_1$, then a sufficient condition for $\dot{\rmL}$
to be a local ${\ell_1}$ $d$-subspace is that for any $\bV \in \Or(d)$ and
$\bC \in \Scp(d)$:
\begin{equation}\label{eq:main}
\sum_{i=1}^{N_1}\|\bC\bV P_{\dot{\rmL}}(\bx_i)\|>\|\bC \bV
\mathbf{B}_{\dot{\rmL},\rmX}\|_{*}\,.
\end{equation} Furthermore, a necessary condition is that for any $\bV \in \Or(d)$ and
$\bC \in \Scp(d)$:
\begin{equation}\label{eq:main_necessary}
\sum_{i=1}^{N_1}\|\bC\bV P_{\dot{\rmL}}(\bx_i)\|\geq \|\bC \bV
\mathbf{B}_{\dot{\rmL},\rmX}\|_{*}\,.
\end{equation}
\end{theorem}

\begin{proposition}
\label{prop:comb_cond2} If $\dot{\rmL} \in \GDd$, $\rmX_1 =
\{\bx_i\}_{i=1}^{N_1}\subset \dot{\rmL}$, $\rmX_0 =
\{\mathbf{y_i}\}_{i=1}^{N_0}\subset \mathbb{R}^D\setminus \dot{\rmL}$,
$\Sp(\{\bx_i\}_{i=1}^{N_1})=\dot{\rmL}$, $\rmX=\rmX_0 \cup \rmX_1$ and $p<1$, then $\dot{\rmL}$ is a
local minimum of $e_{\ell_p}(\rmX,\rmL)$ among all $\rmL\in\GDd$.
\end{proposition}

\begin{proposition}
\label{prop:comb_cond3} If $\dot{\rmL} \in \GDd$, $\rmX_1 =
\{\bx_i\}_{i=1}^{N_1}\subset \dot{\rmL}$, $\rmX_0 =
\{\mathbf{y_i}\}_{i=1}^{N_0}\subset \mathbb{R}^D\setminus \dot{\rmL}$, $\rmX=\rmX_0 \cup \rmX_1$ and
$p>1$, then a necessary condition for $\dot{\rmL}$ to be a local minimum
of $e_{\ell_p}(\rmX,\rmL)$  among all $\rmL\in\GDd$ is
\begin{equation}
\sum_{i=1}^{N_0}P_{\dot{\rmL}}(\by_i)P^{\perp}_{\dot{\rmL}}(\by_i)^T\dist(\by_i,\dot{\rmL})^{p-2}=
\b0.\label{eq:necessary}
\end{equation}
This statement is also true when $\rmX_1=\emptyset$ and $0<p\leq 1$.
\end{proposition}

The above conditions follow from differentiating the corresponding energy function (along geodesics) and using the resulting derivative to
form necessary and sufficient conditions for local minimum (see their proof in \S\ref{sec:proofs_counter+cond}). However, intuitively it is hard to explain their expressions without going through all calculations. Instead, we exemplify them as follows.
\begin{example}
\label{example:local_min}
We simplify the conditions of Theorem~\ref{thm:comb_cond} and Propositions~\ref{prop:comb_cond2} and~\ref{prop:comb_cond3}
for the special case of Example~\ref{example:calculate_B}.\\
\noindent
\textbf{The Case } $\boldsymbol{p=1}$\textbf{:}\\
\noindent
Let us first simplify~\eqref{eq:main} (or equivalently~\eqref{eq:main_necessary}) in this example.
If $\dot{\rmL} = \rmL_x$, then the set of inliers and outliers are $\rmX_1=\{(a_i,0)^T \}_{i=1}^{N_1}$ and $\rmX_0=\{\bz\}$ respectively. Since $d=1$ then $\bV\in\Or(d)$ is either $1$ or $-1$ and $\bC$ is a positive constant $c$.
The LHS of~\eqref{eq:main} thus has the form
\[
\sum_{\bx \in \rmX_1} \|\bC\bV P_{\dot{\rmL}}(\bx)\|=c \, \sum_{i=1}^{N_1}|a_i|
\]
and computing $\mathbf{B}_{\dot{\rmL},\rmX}$ as in~\eqref{eq:example}, the RHS has the form
\[
\|\bC \bV
\mathbf{B}_{\dot{\rmL},\rmX}\|_{*}=c \, t_0\cos(\theta_0).
\]
Therefore, a sufficient condition for $\rmL_x$ to be a local $\ell_1$ line is
\[
\sum_{i=1}^{N_1}|a_i|>t_0\cos(\theta_0).
\]

If $\dot{\rmL} = \rmL_{\bz}$, then $\rmX_1=\left\{\bz\right\}$ and $\rmX_0=\{(a_i,0)^T \}_{i=1}^{N_1}$. Applying~\eqref{eq:example2} and following similar calculations as above we have that a sufficient condition for $\rmL_{\bz}$ to be a local $\ell_1$ line is
\[
\cos(\theta_0)\sum_{i=1}^{N_1}|a_i|<t_0.
\]

If on the other hand $\dot{\rmL}$ does not pass through any point in $\rmX$, then $\rmX_1=\emptyset$ and $\rmX_0=\rmX$. Therefore the LHS of \eqref{eq:main} is 0 and thus \eqref{eq:main} never holds.

All the above conditions are also necessary when their  inequalities are not strict
(see~\eqref{eq:main_necessary}).

We thus note that if $\theta_0 = \pi/2$, then both $\rmL_x$ and $\rmL_{\bz}$ are the only two local $\ell_1$ lines (assuming the obvious conditions: $t_0>0$ and $\sum_{i=1}^{N_1}|a_i|>0$).  If on the other hand $0 < \theta_0 < \pi/2$, then $\rmL_x$ is a local $\ell_1$ line if
$\sum_{i=1}^{N_1}|a_i|/t_0 > \cos(\theta_0)$ 
and $\rmL_{\bz}$ is a local $\ell_1$ line if
$\sum_{i=1}^{N_1}|a_i|/t_0 < 1/\cos(\theta_0)$ (we also recall that for necessary conditions we relax the strict inequalities).
Therefore, for fixed $0<\theta_0<\pi/2$ at least one of $\rmL_x$ or $\rmL_{\bz}$ is a local $\ell_1$ line and there are no other local minimizers.
If $t_0$ is sufficiently large, then $\rmL_{\bz}$ is the global $\ell_1$ line and if $t_0$ is sufficiently small, then $\rmL_x$ is the global $\ell_1$ line.

\noindent
\textbf{The Case } $\boldsymbol{0<p<1}$\textbf{:}\\
\noindent
We note that Proposition~\ref{prop:comb_cond2} implies that both $\rmL_x$ and $\rmL_{\bz}$ are local $\ell_p$ lines (as long as $N_1 \neq 0$ and one of the $a_i$'s is not zero).

\noindent
\textbf{The Case } $\boldsymbol{p>1}$\textbf{:}\\
\noindent
We express the necessary condition of Proposition~\ref{prop:comb_cond3} in our setting.
If $\dot{\rmL} = \rmL_x$, then the LHS of \eqref{eq:necessary} is
$t_0^p\cos(\theta_0) (\sin(\theta_0))^{p-1}$.
Therefore, \eqref{eq:necessary} holds in this case only when $\theta_0=\pi/2$
(recall that $0 < \theta_0 \leq \pi/2$).
Similarly, if $\dot{\rmL}=\rmL_{\bz}$, then the LHS of \eqref{eq:necessary} is
$\sum |a_i|^p \cos(\theta_0)\sin(\theta_0)^{p-1}$ and thus also in this case
\eqref{eq:necessary} holds only when $\theta_0=\pi/2$.

At last, if $\dot{\rmL}$ has an angle $\theta$ with the $x$-axis, where $-\pi/2 < \theta \neq 0, \theta_0 < \pi/2$,
that is, $\dot{\rmL}$ is any line but not $\rmL_x$ or $\rmL_{\bz}$, then \eqref{eq:necessary} holds only when
\begin{equation}
\label{eq:example2_pbig2}
\cos(\theta) \, \sin(\theta) \, |\sin(\theta)|^{p-2}
\sum_{i=1}^{N_1}|a_i|
+t_0 \, \cos(\theta-\theta_0) \,
\sin(\theta-\theta_0) \, |\sin(\theta-\theta_0)|^{p-2}=0.
\end{equation}
We first note that if $\theta_0=\pi/2$, then the LHS of~\eqref{eq:example2_pbig2} is either positive or negative and thus $\dot{\rmL}$ is not a local minimum. If on the other hand $\theta_0 \neq \pi/2$, then since both $\rmL_{\bz}$
and $\rmL_{x}$ are not local minimizers (see above), then there exists $\theta$ such that
$\dot{\rmL} \equiv \dot{\rmL} (\theta)$ is a local minimizer (a continuous function over the Grassmannian has at least one local minimizer). If $\theta_0<\theta<\pi/2$ or $-\pi/2<\theta<0$, then the LHS of~\eqref{eq:example2_pbig2} is either positive or negative.
It is thus necessary that $0<\theta<\theta_0$. That is, a local minimizer $\dot{\rmL}$  must lie between $\rmL_{x}$ and $\rmL_{\bz}$. Furthermore, $\dot{\rmL} \in \GDd$ is a local $\ell_p$ minimum w.p.~0 (w.r.t.~$\gamma_{D,d}$),
since $0<\theta<\theta_0$ satisfies~\eqref{eq:example2_pbig2} w.p.~0.

We emphasize that for $p>1$ we only specified a necessary condition. In particular, when $\theta_0=\pi/2$ we suspect that almost always only one of the subspaces $\rmL_{\bz}$ and $\rmL_x$ can be a local subspace. Indeed, when $p=2$ (and $\theta_0=\pi/2$) it follows from basic eigenvalue analysis of the covariance matrix that the following holds: If $t_0$ is sufficiently small, then $\rmL_x$ is the only global (or local) $\ell_2$ subspace; if $t_0$ is sufficiently large, then $\rmL_{\bz}$ is the only global (or local) $\ell_2$ subspace; and for a unique choice of $t_0$ (given the other parameters) both $\rmL_x$ and $\rmL_{\bz}$ are the global minimizers.


\end{example}


\subsection{Local \boldsymbol{$\ell_p$} Subspaces for Probabilistic Settings with a Single Subspace}
\label{sec:local_global_lp} We exemplify how to use the conditions of \S\ref{sec:conditions}
in a probabilistic setting of i.i.d.~samples from a uniform HLM measure with a single underlying subspace (i.e., $K=1$). More precisely we assume that $\mu_0$ and $\mu_1$ are uniform on $\mathbb{S}^{D-1}$ and $\mathbb{S}^{D-1}\cap \rmL^*_1$ respectively, where $\rmL^*_1 \in \GDd$ is fixed, and sample i.i.d.~inliers from $\mu_1$ and i.i.d.~outliers from $\mu_0$ (instead of using mixture weights).

Since $K=1$, $\rmL^*_1$ is both the most significant subspace and the global $\ell_0$ subspace w.o.p. For any $p>0$, we determine whether $\rmL^*_1$
is also a local $\ell_p$ subspace w.o.p. Our proofs appear in
\S\ref{sec:proofs_local_global_lp}.

We first claim that for $p=1$ the global $\ell_0$ subspace is a local
$\ell_p$ subspace w.o.p.~as long as the fraction of inliers is
 larger than $0$ (assuming that $N$ is sufficiently large).
\begin{theorem}\label{thm:dist_1subs_loc1} If $\rmL^*_1\in \GDd$, $\rmX$ is a data
set in $\reals^D$ of $N_0+N_1$ points, where $N_0$ of them are uniformly sampled from $\mathbb{S}^{D-1}$
and $N_1$ of them are uniformly sampled from $\mathbb{S}^{D-1}\cap \rmL^*_1$.
Then $\rmL^*_1$ is a local $\ell_1$ subspace of $\rmX$
w.p.~at least
%
%
\begin{equation}\label{eq:dist_1subs_loc1} 1-2d^2\exp\left(-\frac{N_1}{8.01 \cdot d^4} \right)-2dD
\exp\left(-\frac{N_1^2}{8 \cdot d^4 \cdot  D  \cdot N_0 }\right)\,. \end{equation}%
\end{theorem}

We note that if $N_0=o(N_1^2)$, then $\rmL^*_1$ is a local $\ell_1$ subspace of $\rmX$  w.o.p. However, when
$N_1 \ll  \sqrt{N_0}$,
then the lower bound for the probability in \eqref{eq:dist_1subs_loc1} is actually negative and thus meaningless.

We observe that the asymptotic requirement $N_0=o(N_1^2)$ allows any fraction of outliers lower than 1 when $N\rightarrow\infty$. Indeed, if $0 \leq \alpha <1$ is fixed, $N_0=\alpha N$ and $N_1=(1-\alpha) N$, then
$N_0=o(N_1^2)$ is equivalent with $\alpha = o(N\cdot (1-\alpha)^2)$, which is
satisfied when $N\rightarrow\infty$.

We emphasize, however, that this recovery of local minima with arbitrarily high percentage of outliers requires a significantly large number of inliers. Indeed, the first exponent in~\eqref{eq:dist_1subs_loc1} implies that $N_1=\Omega(d^4)$. Moreover, the second exponent in~\eqref{eq:dist_1subs_loc1} implies that
$N_1=\Omega(d^2 \sqrt{D} \sqrt{N_0})$.


For comparison, the S-REAPER algorithm~\cite{LMTZ2012} can recover the global $\ell_1$ minimizer when $d<(D-1)/2$  with $N_1= \Omega(d)$ and $N_0= \Omega(D)$, which are significantly smaller (see~\cite[Theorem 1.1]{LMTZ2012}). However, in this case the asymptotic fraction of outliers (when $N$ approaches infinity) is restricted as follows: $N_0/N< D/(D+30d)$ (it is possible that 30 can be reduced to a number closer to 1). We remark that in this case with no noise, the minimizer of S-REAPER is an orthogonal projector and not a relaxation of it and thus it reveals the global $\ell_1$ minimizer. Furthermore, while~\cite[Theorem 1.1]{LMTZ2012} assumes normal distributions for the inliers and outliers, the S-REAPER normalizes the data points to the sphere and thus it also applies to our case of spherically uniform distributions.

Next we discuss the case where $p\neq 1$. If $p>1$, then Proposition~\ref{prop:comb_cond3} implies that under a rather general setting, the global $\ell_0$ subspace is not a local $\ell_p$ subspace w.p.~1. Indeed, it is rather unlikely to satisfy~\eqref{eq:necessary}.
We clarify this idea by showing in \S\ref{sec:prop:p_geq_1}
that if $p>1$, the inliers are sampled from the  single subspace $\rmL^*_1$, the outlier distribution does not concentrate on any subspace and $D>d-1$, then w.p.~1 $\rmL^*_1$ is not a local $\ell_p$ minimizer.

If on the other hand $0<p<1$, then
Proposition~\ref{prop:comb_cond2} implies that w.o.p.~$\rmL^*_1$ is a local $\ell_p$ subspace.
In fact, this proposition suggests the weakest condition one would expect for a subspace to be a local minimizer, that is, being spanned by the points it contains.

The phase transition
phenomenon demonstrated
above at $p=1$  for the global $\ell_0$ subspace (or most significant subspace) to be a local $\ell_p$ is rather artificial in the current setting with $K=1$. Indeed,
when $p>1$ the distance between the
global $\ell_0$ subspace and the global $\ell_p$ subspace (which is also a local $\ell_p$ subspace) approaches 0
as $N$ approaches infinity.
Moreover, Theorem~\ref{thm:noisy_hlm} shows that this formal phase transition also
breaks down with noise. Nevertheless,
Theorems~\ref{thm:hlm} and~\ref{thm:phase_hlm} indicate that there is a clear phase transition for a spherically uniform HLM model with $K>1$. 


\subsection{Counterexamples for Robustness of Best \boldsymbol{$\ell_p$} Subspaces}
\label{sec:counter}
We discuss here basic situations, where global $\ell_p$ $d$-subspaces are not robust to outliers for all $0 < p < \infty$. More precisely, we show how a
single outlier can completely change the underlying subspace.
These cases differ from our underlying model of spherically uniform outliers (or spherically symmetric outliers).
In all examples below we assume a single underlying subspace
and thus discuss the global
$\ell_0$ subspace instead of the most significant subspace.
While we describe a probabilistic setting to sample
the data, we only care about a single counterexample sampled this way. We thus do not bother about statements in high probability (even though they are correct), but a positive statement for at least one of the sampled data sets.

A typical example includes $N_1$ points sampled identically and independently from a uniform
distribution on $\ball(\b0,\eps)\cap \rmL^* \subseteq \reals^D$, where $\rmL^*$ is a $d$-subspace
of $\reals^D$,
and an additional outlier located on a
unit vector orthogonal to $\rmL^*$.
By choosing $\eps$ sufficiently small, e.g., $\eps \leq N_1^{-1/p}$, the global
$\ell_p$ subspace passes through the single outlier and is
orthogonal to the initial $d$-subspace for all $p>0$, which is the global $\ell_0$ $d$-subspace.

If $p=1$, then the global $\ell_0$ $d$-subspace in this example is
still a local $\ell_1$ subspace (as explained in Example~\ref{example:local_min} for the special cases $d=1$ and $D=2$).
Nevertheless, if the outlier is
located instead on a unit vector having elevation angle with the
original $d$-subspace less than $\pi/2$, then $\eps$ can be chosen
so that the global $\ell_0$ subspace is even not a local
$\ell_1$ subspace (see again Example~\ref{example:local_min}).
However, if $0<p<1$, then Proposition~\ref{prop:comb_cond2} implies that the global $\ell_0$
subspace is still a local $\ell_p$ subspace in both examples.

Similarly, it is not hard to produce examples of data points on
the unit sphere of $\reals^D$ where the global $\ell_0$ subspace is
still not a global $\ell_p$ subspace for all $p>0$.
It is important for us to point it out since for simplicity we formulated the theory for data lying on the unit sphere and by normalizing the data sets in the examples above
to the unit sphere, they may not form counterexamples any more.
For simplicity we give a counterexample when $D=3$ and $d=2$. We uniformly sample $N_1$ inliers ($N_1>2$) from an arc on the great circle in the $xy$-plane with the following parametrization: $(\cos \theta,\sin\theta,0)$, where $\theta \in [-\eps,\eps]$. We also fix an outlier  $(x_0,y_0,z_0) \in \mathbb{S}^2$ such that $z_0\neq 0$. For any fixed $p >0$ and $\eps$ sufficiently small, the 2-subspace spanned by $(x_0,y_0,z_0)$ and $(1,0,0)$ (which is the center of the arc) results in a smaller $\ell_p$ energy than that of the global $\ell_0$ subspace (i.e., the $xy$-plane). That is, the global $\ell_0$ subspace is not the global $\ell_p$ subspace.

\section{Verification of Theory}
\label{sec:proofs}
We describe here the proofs of the
theorems and propositions of this paper according to the following order
of sections: \S\ref{sec:combinatorics}, \S\ref{sec:local_global_lp} and \S\ref{sec:intro}.

\subsection{Preliminaries}
\label{sec:prelim}
\subsubsection{Basic Notation and Conventions}
\label{sec:notation}

We denote the Frobenius dot product and norm by $\langle \mathbf{A}, \mathbf{B} \rangle_F$ and $\|\mathbf{A}\|_F$, that is,
$\langle \mathbf{A}, \mathbf{B} \rangle_F = \tr(\mathbf{A}^T \mathbf{B})$ and $\|\mathbf{A}\|_F = \sqrt{\langle \mathbf{A}, \mathbf{A} \rangle_F}$.
The $n\times n$ identity
matrix is written as $\mathbf{I}_n$.
 We denote the
subset of $\Scp(n)$ with Frobenius norm 1 by $\NScp(n)$. If $m>n$ we
let $\Or(m,n)=\{ \bX \in \reals^{m \times n} :  \bX^T \bX =
\mathbf{I}_n \}$, whereas if $n>m$, $\Or(m,n)=\{ \bX \in \reals^{m
\times n}  :  \bX \bX^T = \mathbf{I}_m \}$.

We sometimes apply the energy~\eqref{eq:def_error_1sub}
to a single point $\bx$, while using
the notation: $e_{\ell_p}(\bx,\rmL) \equiv e_{\ell_p}(\{\bx\},\rmL)$.

\subsubsection{Auxiliary Lemmata}
\label{sec:lemmas_state} We formulate several technical lemmata,
which will be proved in Appendices~\ref{app:lemma_proof2}-\ref{app:lemma_proof}.

\begin{lemma}\label{lemma:ell_dist_est}
If $\rmL_1,\hat{\rmL}_1\in\GDd$, $p>0$ and $\mu_1$ is a uniform measure on $\rmL_1\cap \mathbb{S}^{D-1}$, then
\begin{align*}
E_{{\mu}_1}\left(e_{\ell_p}(\bx,\hat{\rmL}_1)\right)
>
\begin{cases}\pi^{-p} \cdot 2^{p} \cdot d^{-\frac{p}{2}} \cdot
\dG(\rmL_1,\hat{\rmL}_1)^p, &\text{if $p\geq 2$};\\
0.88 \cdot 2^{\frac{3p}{2}} \cdot \pi^{-\frac{(2p+1)}{2}}\cdot ({d+p})^{-p/2} \cdot
{\dG(\rmL_1,\hat{\rmL}_1)^p}, &\text{if $p<2$}.
\end{cases}
\end{align*}
\end{lemma}

\begin{lemma}\label{lemma:distance}
For any $\bx\in \reals^D$ and $\rmL_1,\rmL_2\in\GDd$:
$$|\dist(\bx,\rmL_1)-\dist(\bx,\rmL_2)|\leq\|\bx\|\,\dG(\rmL_1,\rmL_2).
$$
\end{lemma}

\begin{lemma}\label{lemma:mean}
If $\rmL_1,\rmL_2\in\GDd$, $\mu_1$ and $\mu_2$ are uniform measures on $\rmL_1\cap \mathbb{S}^{D-1}$ and $\rmL_2\cap \mathbb{S}^{D-1}$ respectively and $p\leq 1$, then
for any $\hat{\rmL}\in\GDd$:
\begin{align}&\mathbb{E}_{\mu_1}(\dist(\bx_1,\hat{\rmL})^p)+\mathbb{E}_{\mu_2}(\dist(\bx_2,\hat{\rmL})^p)\nonumber\\\geq&
\mathbb{E}_{\mu_1}(\dist(\bx_1,\rmL_i)^p)+\mathbb{E}_{\mu_2}(\dist(\bx_2,\rmL_i)^p)
\ \ \text{\rm    for } i=1,2.\label{eq:lemma2}\end{align}
\end{lemma}

\subsection{Proofs for the Theory of \S\ref{sec:combinatorics}: Combinatorial Conditions via Calculus on the Grassmannian}

\label{sec:proofs_counter+cond}

\subsubsection{Preliminaries: Principal Angles, Principal Vectors, Representation of the Grassmannian and Geodesics on the Grassmannian}

\label{sec:prin_angles}

We frequently use here principal
angles and for completeness we present one of their equivalent definitions (\S12.4.3 of~\cite{Golub96} provides additional background on principal angles).
For two $d$-subspaces $\rmF$ and $\rmG$ with corresponding orthonormal bases stored as columns of the matrices $\bQ_{\rmF}$,
$\bQ_{\rmG}$ $\in \reals^{D \times d}$ respectively, the principal angles
$\pi/2 \geq \theta_1\geq\theta_2\geq\cdots\geq\theta_d \geq 0$,
are obtained by
\begin{equation}
\label{eq:principal_angles}
\theta_i = \arccos (\sigma_{d-i}(\bQ_{\rmG}^T \bQ_{\rmF}) ), \quad i=1,\ldots,d,
\end{equation}
where $\sigma_{d-i}(\bQ_{\rmG}^T \bQ_{\rmF})$ is the $(d-i)$th singular value of the matrix $\bQ_{\rmG}^T \bQ_{\rmF}$.
We remark that we order the principal angles decreasingly, unlike the common agreement~\cite{Golub96} (\S12.4.3), where
$\sigma_{d-i}$ in~\eqref{eq:principal_angles} is replaced by $\sigma_{i}$.

We denote by $k=k(\rmF,\rmG)$ the
largest number such that $\theta_k\neq 0$, so that $\theta_1 \geq
\ldots \geq \theta_k > \theta_{k+1}=\ldots=\theta_d=0$. We refer to
this number as the interaction dimension and reserve the index $k$ for
denoting it (the subspaces $\rmF$ and $\rmG$ will be clear from the
context). We recall that the principal vectors
$\{\mathbf{v}_i\}_{i=1}^{d}$ and $\{\mathbf{v}'_i\}_{i=1}^{d}$ of
$\rmF$ and $\rmG$ respectively are two orthogonal bases for $\rmF$
and $\rmG$ satisfying
\begin{equation*}
\left\langle \mathbf{v}_i,\mathbf{v}'_i\right\rangle
=\cos(\theta_i), \quad \text{for  } i=1,\ldots,d,
\end{equation*}
and
\begin{equation*}
\mathbf{v}_i\perp \mathbf{v}'_j, \quad \text{for all  } 1\leq i \neq
j \leq k.
\end{equation*}

We define the complementary orthogonal system $\{\bu_i\}_{i=1}^d$
for $\rmG$ with respect to $\rmF$ by the
formula:\begin{equation}\begin{cases}
\mathbf{v}'_i=\cos(\theta_i)\mathbf{v}_i+\sin(\theta_i)
\mathbf{u}_i,  &i=1,2,\cdots,k; \label{eq:trans_repres}\\
\mathbf{u}_i=\mathbf{v}_i, &i=k+1,\cdots,d.\end{cases}\end{equation}
Clearly,
$$\mathbf{u}_i\perp\mathbf{v}_j \ \text{ for all } \ 1\leq i , j
\leq k \,.$$

We note that $\rmF + \rmG$ can be decomposed using these principal vectors as follows:
\[
\rmF + \rmG= \Sp(\bv_1,\bu_1) \bigoplus \Sp(\bv_2,\bu_2) \bigoplus \cdots \bigoplus \Sp(\bv_k,\bu_k)
\bigoplus (\rmF \cap \rmG),
\] 
where $\bigoplus$ denotes an orthogonal sum (i.e., any two subspaces of the sum are orthogonal).
Therefore, the interaction between $\rmF$ and $\rmG$ can be
described only within these 2-dimensional subspaces $\Sp(\bv_i,\bu_i)$ (equivalently, $\Sp(\bv_i,\bv_i')$) 
via the principal angles. This
idea is also motivated by purely geometric intuition
in \S2 of \cite{Grassmannian}.


It follows from~\cite[Theorem 9]{Grassmannian} that if the largest
principal angle between $\rmF$ and $\rmG$ is less than $\pi/2$, then
there is a unique geodesic line between them.
Following~\cite[Theorem 2.3]{Edelman98thegeometry}, we can
parametrize this line from $\rmF$ to $\rmG$ by the following
function $\rmL$: [0,1]$\rightarrow \GDd$, which is expressed in
terms of the principal angles $\{\theta_i\}_{i=1}^d$ of $\rmF$ and
$\rmG$, the principal vectors  $\{\bv_i\}_{i=1}^d$ of $\rmF$ and the
complementary orthogonal system $\{\bu\}_{i=1}^d$ of $\rmG$ with
respect to $\rmF$:
\begin{equation}\label{eq:param_geod}\rmL(t)=\Sp(\{\cos(t\theta_i)\mathbf{v}_i+\sin(t\theta_i)\mathbf{u}_i\}_{i=1}^{d}).
\end{equation}
The length of this geodesic line is clearly expressed by the distance $\dG$ of~\eqref{eq:dist_grassman}.
We remark that \eqref{eq:param_geod} only holds when equipping the Grassmannian with this distance.

\subsubsection{Proof of Theorem~\ref{thm:comb_cond}}

In order to establish quantitative conditions for $\dot{\rmL}$ to be
a local minimizer of
$e_{\ell_1}(\rmX,\rmL)$ among all $d$-subspaces in $\GDd$, we
arbitrarily fix a $d$-subspace $\hat{\rmL}\in \ballG(\dot{\rmL},1)$ and
check the sign of the derivative of the $\ell_1$ energy when restricted to
the geodesic line from $\dot{\rmL}$ to $\hat{\rmL}$.
If this derivative is positive then $\dot{\rmL}$ is a local $\ell_1$ subspace.
Similarly, if $\dot{\rmL}$ is a local $\ell_1$ subspace then this derivative is nonnegative.

The restriction of $\hat{\rmL}$ to $\ballG(\dot{\rmL},1)$ implies that
$\theta_1 \leq 1$ and thus by~\cite[Theorem 9]{Grassmannian} this
geodesic line (connecting $\dot{\rmL}$ and $\hat{\rmL}$) is unique. We
parametrize it by the function $\rmL$: [0,1]$\rightarrow \GDd$
of~\eqref{eq:param_geod}, where here $\{\theta_i\}_{i=1}^d$ are the
principal angles between $\dot{\rmL}$ and $\hat{\rmL}$,
$\{\bv_i\}_{i=1}^d$ are the principal vectors of $\dot{\rmL}$ and
$\{\bu\}_{i=1}^d$ are the complementary orthogonal system for
$\hat{\rmL}$ with respect to $\dot{\rmL}$.
The necessary and sufficient conditions for $\dot{\rmL}$ to be a
local $\ell_1$ subspace will be formulated
in terms of the sign of the
derivative of
$e_{\ell_1}(\rmX,\rmL(t))$: [0,1]$\rightarrow \mathbb{R}$ at $t=0$.
Clearly, this derivative only exists from the right, however, our notation throughout the paper 
does not emphasize it (also when $\ell_1$ is replaced with $\ell_p$).

We follow by simplifying the expression for the function
$e_{\ell_1}(\rmX,\rmL(t))$ and its derivative according to $t$. We
denote the projection from $\mathbb{R}^D$ onto
$\Sp(\mathbf{v}_j,\mathbf{u}_j)$, where $1 \leq j \leq d$, by $P_j$
and the projection from $\mathbb{R}^D$ onto
$(\dot{\rmL}+\hat{\rmL})^{\perp}$ by $P^{\perp}$ and use this notation
to express the following components of the function
$e_{\ell_1}(\rmX_0,\rmL(t))$ for $i=1,\ldots,N_0$ (we later express the components of $e_{\ell_1}(\rmX_1,\rmL(t))$):
\begin{align} &\dist(\mathbf{y}_i,\rmL(t))=
\sqrt{\sum_{j=1}^{d}\dist^2(P_j(\mathbf{y}_i),\rmL(t))+\dist^2(P^{\perp}(\mathbf{y}_i),\rmL(t))}\nonumber\\
&=\sqrt{\sum_{j=1}^{d}\left((-\sin(t\theta_j)\mathbf{v}_j+\cos(t\theta_j)\mathbf{u}_j
)\cdot \mathbf{y}_i\right)^2+\dist^2(P^{\perp}(\mathbf{y}_i),\rmL(t))}.
\label{eq:dist_for_y}\end{align}%

We differentiate the expression for $\dist(\mathbf{y}_i,\rmL(t))$ in~\eqref{eq:dist_for_y} for all $1 \leq i \leq
N_0$ as follows (note that we use the fact that $\dist^2(P^{\perp}(\mathbf{y}_i),\rmL(t))$ is independent of $t$):
\begin{align}
&\frac{\di}{\di t}
\left(\dist(\mathbf{y}_i,\rmL(t))\right)
\nonumber\\
=&-\frac{\sum_{j=1}^{d}\theta_j\left((\cos(t\theta_j)\mathbf{v}_j+\sin(t\theta_j)\mathbf{u}_j
)\cdot
\mathbf{y}_i\right)\left((-\sin(t\theta_j)\mathbf{v}_j+\cos(t\theta_j)\mathbf{u}_j
)\cdot \mathbf{y}_i\right)}
{\dist(\mathbf{y}_i,\rmL(t))}.\label{eq:outlier}
\end{align}
At $t=0$ it becomes
\begin{align}
\left.\frac{\di}{\di t} \left(\dist(\mathbf{y}_i,\rmL(t))\right)
\right|_{t=0}&=-\frac{\sum_{j=1}^{d}\theta_j
(\mathbf{v}_j\cdot\mathbf{y}_i)(\mathbf{u}_j\cdot\mathbf{y}_i)}{\dist(\mathbf{y}_i,\rmL(0))}.
\label{eq:outlier01}
\end{align}

We form the following matrices:
$\bC=\diag(\theta_1,\theta_2,\cdots,\theta_d)$, $\tilde{\bV} \in \Or(d,D)$
with $j$th row $\mathbf{v}_j^{T}$ and $\tilde{\bU} \in \Or(d,D)$ with
$j$th row $\mathbf{u}_j^T$. We then
reformulate~\eqref{eq:outlier01} using these matrices as follows:
\begin{align}
\left. \frac{\di}{\di t} \left(\dist(\mathbf{y}_i,\rmL(t))\right)
\right|_{t=0}=-\frac{\tr(\bC\tilde{\bV}\mathbf{y}_i\mathbf{y}_i^{T}\tilde{\bU}^{T})}{\dist(\mathbf{y}_i,\dot{\rmL})}\label{eq:outlier1}.
\end{align}

Similarly, we express the components of $e_{\ell_1}(\rmX_1,\rmL(t))$ for all $\bx_i\in \dot{\rmL}$, where $i=1,2,\cdots,N_1$, by
\begin{equation*}
\dist(\bx_i,\rmL(t))=\sqrt{\sum_{j=1}^{d}|(\mathbf{v}_j\cdot
\bx_i)|^2\sin^2(t\theta_j)}
\end{equation*}
and differentiate these expressions as follows
\begin{equation}
\frac{\di}{\di t}
\left(\dist(\bx_i,\rmL(t))\right)=\frac{\sum_{j=1}^{d}\theta_j
|\mathbf{v}_j\cdot\bx_i|^2\sin(t\theta_j)\cos(t\theta_j)}{\dist(\bx_i,\rmL(t))}.\label{eq:inlier}
\end{equation}
At $t=0$, these derivatives become
\begin{equation}
\left.\frac{\di}{\di t} \left(\dist(\bx_i,\rmL(t))\right)
\right|_{t=0}= \sqrt{\sum_{j=1}^{d}|(\mathbf{v}_j\cdot \bx_i)|^2 \,
\theta_j^2}= \|\bC\tilde{\bV}\bx_i\|.\label{eq:inlier1}
\end{equation}
Combining~\eqref{eq:outlier1} and~\eqref{eq:inlier1} and using
$$\mathbf{A}:=\sum_{i=1}^{N_0}\mathbf{y}_i \mathbf{y}_i^{T} / \dist(\mathbf{y}_i,\dot{\rmL}),$$
we obtain the following expression for the derivative of the
$\ell_1$ energy of~\eqref{eq:def_error_1sub}:
\begin{align}
\left.\frac{\di}{\di t} \left(e_{\ell_1}(\rmX,\rmL(t))\right)
\right|_{t=0}=\sum_{i=1}^{N_1}\|\bC\tilde{\bV}\bx_i\|-\tr(\bC\tilde{\bV}\mathbf{A}\tilde{\bU}^T).
\end{align}

Replacing $\tilde{\bV}$ with
$\bV \in \Or(d)$, whose $j$th row is
$P_{\dot{\rmL}}(\mathbf{v}_j)^T$ and $\tilde{\bU}$ with
$\bU \in \reals^{d \times (D-d)}$, where
$\bU^T = [\bU_1,\bU_2]$,
$\bU_1 \in \Or(D-d,k)$, whose $j$th row is
$P_{\dot{\rmL}}^{\perp}(\mathbf{u}_j)^T$, and $\bU_2 = \b0_{(D-d) \times (d-k)}$,
we may rewrite this expression as follows:
\begin{align}
\left.\frac{\di}{\di t}
\left(e_{\ell_1}(\rmX,\rmL(t))\right)\right|_{t=0}=
\sum_{i=1}^{N_1}\|\bC\bV P_{\dot{\rmL}}\bx_i\|-\tr(\bC\bV\mathbf{B}_
{\dot{\rmL},\rmX}\bU^T)\label{eq:result1}.
\end{align}

We note that
\begin{equation}\max_{\bU^T}(\tr(\bC\bV\mathbf{B}_{\dot{\rmL},\rmX}\bU^T))=
\|\bC\bV\mathbf{B}_{\dot{\rmL},\rmX}\|_*.\label{eq:max}\end{equation}
Indeed, denoting the thin SVD decomposition of
$\bC\bV\mathbf{B}_{\dot{\rmL},\rmX}$ by
$\mathbf{U}_0\mathbf{\Sigma}_0\bV_0^{T}$ we have that
\begin{align}
\nonumber
\tr(\bC\bV\mathbf{B}_{\dot{\rmL},\rmX}\bU^T)&=\tr(\mathbf{U}_0\mathbf{\Sigma}_0
\bV^T_0\bU^T)=\tr(\mathbf{\Sigma}_0
\bV^T_0\bU^T\mathbf{U}_0)\leq \tr(\mathbf{\Sigma}_0)\\
\label{eq:timber_utah}
&=\|\bC\bV\mathbf{B}_{\dot{\rmL},\rmX}\|_*
\end{align}
and equality is achieved in~\eqref{eq:timber_utah} when $\bU^T = \bV_0\mathbf{U}_0^T$. The theorem is
thus concluded by combining \eqref{eq:result1} and \eqref{eq:max}.

The theorem is now easily concluded. Indeed, if~\eqref{eq:main} is satisfied then it follows from
\eqref{eq:result1} and \eqref{eq:timber_utah} that the derivative of $e_{\ell_1}(\rmX,\rmL(t))$
at $t=0$ is positive and thus $\dot{\rmL}$ is a local $\ell_1$ subspace. If on the other hand $\dot{\rmL}$ is
a local $\ell_1$ subspace, then the derivative of $e_{\ell_1}(\rmX,\rmL(t))$ at $t=0$ is nonnegative for any geodesic line.
It thus follows from \eqref{eq:result1} and \eqref{eq:max} that \eqref{eq:main_necessary} is satisfied.

\subsubsection{Simultaneous Proof for Both Propositions~\ref{prop:comb_cond2} and~\ref{prop:comb_cond3}}
\label{subsec:comb_cond23}

For the $d$-subspace $\dot{\rmL}$ and an arbitrary $d$-subspace
$\hat{\rmL} \in \ballG(\dot{\rmL},1)$, we form the geodesic line
parametrization $\rmL(t)$ and the corresponding matrices $\bC$,
$\tilde{\bV}$, $\tilde{\bU}$, $\bV$ and $\bU$ as in the proof of
Theorem~\ref{thm:comb_cond}.

We assume first that $p>1$ (and thus start with proving the main part of Proposition~\ref{prop:comb_cond3}).
We note for $\bz \in \reals^D$
\begin{align}
\frac{\di}{\di t}
\dist(\bz,\rmL(t))^p
=p\,\dist(\bz,\rmL(t))^{p-1}
\frac{\di}{\di t}
\dist(\bz,\rmL(t)),
\label{eq:outlier_p}
\end{align}
where if $\bz=\bx_i$, $i=1,2,\cdots,N_1$, or $\bz =\by_i$, $i=1,2,\cdots,N_0$,
then the derivative in the RHS of~\eqref{eq:outlier_p} can be formulated using~\eqref{eq:outlier}
or~\eqref{eq:inlier} respectively.
Applying \eqref{eq:outlier01}, \eqref{eq:inlier1}, \eqref{eq:outlier_p} and the fact that
$\dist(\bx_i,\dot{\rmL})=0$, for $i=1,2,\cdots,N_1$, we obtain that
\begin{align}
\label{eq:main2_a1}
\left. \frac{\di}{\di t}
\left(e_{\ell_p}(\rmX,\rmL(t))\right)
\right|_{t=0}
&=-p\sum_{i=1}^{N_0}\dist(\mathbf{y}_i,\dot{\rmL})^{p-2}\tr(\bC\tilde{\bV}\mathbf{y}_i\mathbf{y}_i^{T}\tilde{\bU}^{T})\\
\nonumber
&=-p\sum_{i=1}^{N_0}\dist(\mathbf{y}_i,\dot{\rmL})^{p-2}\tr(\bC\bV
P_{\dot{\rmL}}(\by_i)P^{\perp}_{\dot{\rmL}}(\by_i)^T\bU^{T}).\nonumber
\end{align}

If $\dot{\rmL}$ is a local minimum of $e_{\ell_p}(\rmX,\rmL)$, then
the LHS of~\eqref{eq:main2_a1} is nonnegative.
Fixing $\bC=\bV=\mathbf{I}_{d}$ in the RHS of~\eqref{eq:main2_a1} and using its nonnegativity and then applying~\eqref{eq:max}, we conclude that
\begin{align} \label{eq:main2_ineq_last} 0&\geq
\max_{\bU}{p\sum_{i=1}^{N_0}\dist(\mathbf{y}_i,\dot{\rmL})^{p-2}
\tr(
P_{\dot{\rmL}}(\by_i)P^{\perp}_{\dot{\rmL}}(\by_i)^T\bU^{T})}
\\&=p \,
\left\|\sum_{i=1}^{N_0}\dist(\mathbf{y}_i,\dot{\rmL})^{p-2}P_{\dot{\rmL}}(\by_i)P^{\perp}_{\dot{\rmL}}(\by_i)^T\right\|_*
\end{align}
and consequently that~\eqref{eq:necessary} holds.
That is, Proposition~\ref{prop:comb_cond3} is proved when $p>1$. Proposition~\ref{prop:comb_cond3} can be similarly proved when $\rmX_1=\emptyset$ and $0<p\leq 1$. Indeed, \eqref{eq:main2_a1} still holds in this case ($\rmX=\rmX_0$).

Next, assume that $p<1$. We note that the derivative of $e_{\ell_p}(\rmX,\rmL(t))$ at $t=0$ is only defined when $p \geq 1$
(indeed, in view of~\eqref{eq:inlier1} the limit of the derivative in \eqref{eq:outlier_p} when $t \rightarrow 0$ and $\bz=\bx_i$, $i=1,2,\cdots,N_1$, is infinite).
To overcome this, we use the following derivative according to the variable $t^p$:
\begin{align}
\label{eq:def_deriv_tp} \left. \frac{\di}{\di t^p}
\left(\dist(\bz,\rmL(t)^p)\right)
\right|_{t=0}
=
\lim_{t\rightarrow 0} \Big( \frac{t^{1-p}}{p}  \frac{\di}{\di t}
\Big(\dist(\bz,\rmL(t))^p\Big) \Big).
\end{align}
It follows from~\eqref{eq:outlier_p},~\eqref{eq:def_deriv_tp} and~\eqref{eq:outlier01} that
\begin{align}
\label{eq:twins_angels1}
\frac{\di}{\di t^p}
\left(\dist(\by_i,\rmL(t))^p\right)\Big|_{t=0}
=&\lim_{t\rightarrow 0} \left( \frac{t^{1-p}}{p} \cdot p \cdot \dist(\by_i,\rmL(t))^{p-1} \right) \cdot \frac{\di}{\di t}\dist(\by_i,\rmL(t))\Big|_{t=0}
\\=&0\cdot \frac{\di}{\di t}\dist(\by_i,\rmL(t))\Big|_{t=0}=0.\nonumber
\end{align}
Furthermore, it follows from~\eqref{eq:outlier_p},~\eqref{eq:def_deriv_tp}  and~\eqref{eq:inlier1} (and also its derivation from~\eqref{eq:inlier}) that
\begin{align}
\nonumber
&\frac{\di}{\di t^p}
\left(\dist(\bx_i,\rmL(t))^p\right)\Big|_{t=0}
\!\!\!\!\!\!=\lim_{t\rightarrow 0}\!\! \left(\! \frac{t^{1-p}}{p} \cdot p \cdot \dist(\bx_i,\rmL(t))^{p-1}\! \right) \!\!\cdot\! \frac{\di}{\di t}\dist(\bx_i,\rmL(t))\Big|_{t=0}
\\ \label{eq:twins_angels2}
&=
\left. \left(\lim_{t\rightarrow 0}\dist(\bx_i,\rmL(t))/t\right)^{p-1} \cdot \frac{\di}{\di t}\dist(\bx_i,\rmL(t))  \right|_{t=0}
=
\|\bC\bV P_{\dot{\rmL}}(\bx_i)\|^p.
\end{align}
Combining~\eqref{eq:twins_angels1} and~\eqref{eq:twins_angels2}
we obtain that
\begin{align}
\left. \frac{\di}{\di t^p}
\left(e_{\ell_p}(\rmX,\rmL(t))\right)
\right|_{t=0} =\sum_{i=1}^{N_1}\|\bC\bV P_{\dot{\rmL}}(\bx_i)\|^p.
\label{eq:main2_a_p<1}\end{align}


Now, if $\Sp(\{\bx_i\}_{i=1}^{N_1})=\dot{\rmL}$, then there exists $1\leq j\leq N_1$ such that $\bv_1^T\bx_j\neq 0$ and
thus $\|\bC \bV P_{\dot{\rmL}}(\bx_i)\|=\|\bC \tilde{\bV} \bx_i\| \geq \theta_1 \, \|\bv_1^{T}\bx_i\| > 0$.
Combining this observation with \eqref{eq:main2_a_p<1} we conclude that $\dot{\rmL}$ is a local minimum of
$e_{\ell_p}(\rmX,\rmL(t))$ and thus prove Proposition~\ref{prop:comb_cond2}.

\subsection{Proof of Theorem~\ref{thm:dist_1subs_loc1}: Combination of Combinatorial Estimates (\S\ref{sec:proofs_counter+cond})
with Probabilistic Estimates}
\label{sec:proofs_local_global_lp}

To find the probability that $\rmL^*_1$ is a local $\ell_1$ subspace
we will estimate the probabilities of large LHS and small RHS
of~\eqref{eq:main} for arbitrary $\hat{\rmL} \in \ballG(\rmL^*_1,1)$.
We denote the $N_1$ inliers and $N_0$ outliers by
$\{\mathbf{x_i}\}_{i=1}^{N_1}$ and $\{\mathbf{y_i}\}_{i=1}^{N_0}$ respectively. Due to the homogeneity
of~\eqref{eq:main} in $\bC$, we will assume WLOG that $\|\bC\|_2=1$,
i.e., $\theta_1=1$.

We start with estimating the probability that the RHS
of~\eqref{eq:main} is small. Applying the above assumption that
$\|\bC\|_2=1$ we have that
\[\|\mathbf{CVB}_{\rmL^*_1,\rmX}\|_F\leq\|\mathbf{VB}_{\rmL^*_1,\rmX}\|_F=\|\mathbf{B}_{\rmL^*_1,\rmX}\|_F\]
and consequently
\begin{align*}
\Pr\left(\frac{\|\bC\bV\mathbf{B}_{\rmL^*_1,\rmX}\|_*}{N_0}<\epsilon\right)
&\geq\Pr\left(\frac{\|\bC\bV\mathbf{B}_{\rmL^*_1,\rmX}\|_F}{N_0}<\frac{\epsilon}{\sqrt{d}}\right)\\
\geq\Pr\left(\frac{\|\mathbf{B}_{\rmL^*_1,\rmX}\|_F}{N_0}<\frac{\epsilon}{\sqrt{d}}\right)
\geq&\Pr\left(\frac{\max_{1 \leq p,l \leq d}|(\mathbf{B}_{\rmL^*_1,\rmX})_{p,l}|}{N_0}<\frac{\epsilon}{d\sqrt{D}}\right).
\end{align*}
We further estimate this probability by Hoeffding's inequality as
follows: we view the matrix $\mathbf{B}_{\rmL^*_1,\rmX}$ as the sum of
random variables $P_{\rmL^*_1}(\mathbf{y}_i)
P_{\rmL^*_1}^{\perp}(\mathbf{y}_i)^{T}/\|P_{\rmL^*_1}^{\perp}(\mathbf{y}_i)\|$,
$i=1, \ldots, N_0$. Since the distribution of outliers is uniform on the unit sphere, the coordinates of both
$P_{\rmL^*_1}(\mathbf{y}_i)$ and
$P_{\rmL^*_1}^{\perp}(\mathbf{y}_i)^{T}/\|P_{\rmL^*_1}^{\perp}(\mathbf{y}_i)\|$
have expectations $0$ and take values in [-1,1]. We can thus
apply Hoeffding's inequality to the sum defining
$\mathbf{B}_{\rmL^*_1,\rmX}$ and consequently obtain that
\begin{equation}
\Pr\left(\frac{\max_{1 \leq p,l \leq d}|(\mathbf{B}_{\rmL^*_1,\rmX})_{p,l}|}{N_0}<\frac{\epsilon}{d\sqrt{D}}\right)\geq
1-2 dD
\exp\left(-\frac{N_0\epsilon^2}{2d^2D}\right).\label{eq:right}
\end{equation}

Next, we estimate the probability that the LHS of~\eqref{eq:main} is
sufficiently large. We first note that
\begin{align}
\sum_{i=1}^{N_1}\|\bC\bV P_{\rmL^*_1} (\bx_i)\| &\geq
\sum_{i=1}^{N_1}|\theta_1\mathbf{v}_1^{T}P_{\rmL^*_1} (\bx_i)|
=\sum_{i=1}^{N_1}|\mathbf{v}_1^{T}P_{\rmL^*_1} (\bx_i)| \nonumber\\
& \geq\sqrt{\sum_{i=1}^{N_1}|\mathbf{v}_1^{T}P_{\rmL^*_1} (\bx_i)|^2}
\geq
\min_{t}\sigma_t\left(\sum_{i=1}^{N_1}P_{\rmL^*_1}(\bx_i)P_{\rmL^*_1}(\bx_i)^{T}\right).\label{eq:temp}
\end{align}
Second of all, since $\mu_1$ is uniform on $\rmL^*_1 \cap \mathbb{S}^{D-1}$
\begin{equation} \label{eq:def_delta}
E_{\mu_1}(P_{\rmL^*_1}(\bx)P_{\rmL^*_1}(\bx)^{T}) =
\delta_{*}\mathbf{I}_{d}, \ \text{ where $\delta_{*}=\frac{1}{d}$.} \end{equation}

We will prove in \S\ref{app:verify_if} the following statement:
\begin{multline}\label{eq:verify_if} \text{If }  \max_{1 \leq j \leq d} \sigma_j
\left(\sum_{i=1}^{N_1}P_{\rmL^*_1}(\bx_i)P_{\rmL^*_1}(\bx_i)^{T}
-\delta_{*} \mathbf{I}_{d}\right) <
\eta,\\ \text{ then } \min_{1 \leq j \leq d} \sigma_j
\left(\sum_{i=1}^{N_1}P_{\rmL^*_1}(\bx_i)P_{\rmL^*_1}(\bx_i)^{T}\right)
> \delta_{*}-\eta. \end{multline}%
We combine~\eqref{eq:temp}-\eqref{eq:verify_if} and Hoeffding's
inequality to obtain the following probabilistic estimate for the
LHS of \eqref{eq:main}:
\begin{align}
&\Pr\left(\frac{\sum_{i=1}^{N_1}\|\bC\bV P_{\rmL^*_1} (\bx_i)\|}{N_1}>\delta_{*}-\eta\right)\label{eq:left}\\&\geq\Pr
\left(\min_{1 \leq j \leq d} \sigma_j \left(\frac{\sum_{i=1}^{N_1}P_{\rmL^*_1} (\bx_i )P_{\rmL^*_1} (\bx_i )^{T}}{N_1}\right)>\delta_{*}-\eta\right)\nonumber\\
&\geq \Pr\left(\max_{1 \leq j \leq d} \sigma_j \left(\frac{\sum_{i=1}^{N_1}P_{\rmL^*_1} (\bx_i )P_{\rmL^*_1} (\bx_i )^{T}}{N_1}-\delta_{*} \mathbf{I}_{d}\right)<\eta\right)\nonumber\\
&\geq \Pr\left(\left\|\frac{\sum_{i=1}^{N_1}P_{\rmL^*_1} (\bx_i
)P_{\rmL^*_1} (\bx_i )^{T}}{N_1}-\delta_{*}\mathbf{I}_{d}\right\|_F
<\eta\right)\nonumber\\
&\geq \Pr\left(\max_{1 \leq p,l \leq d} \left| \left( \frac{\sum_{i=1}^{N_1}P_{\rmL^*_1}
(\bx_i )P_{\rmL^*_1} (\bx_i )^{T}}{N_1}-
\delta_{*}\mathbf{I}_{d}\right)_{1\leq p,l \leq d} \right|<\frac{\eta}{d}\right) \nonumber
\\
&\geq1-2d^2\exp\left(-\frac{N_1\eta^2}{2d^2}\right).\nonumber
\end{align}

From \eqref{eq:right} and \eqref{eq:left}, \eqref{eq:main} is valid
with probability at least
\begin{equation}
1-2d^2\exp\left(-\frac{N_1\eta^2}{2d^2}\right)-2dD
\exp\left(-\frac{N_0\epsilon^2}{2d^2D}\right) \,\,\,\,\forall \,
\epsilon, \, \eta \ \text{ s.t. } \eta+\frac{N_0}{N_1}\epsilon<\frac{1}{d}.
\end{equation}
We can choose $\epsilon=N_1/2dN_0$,
$\eta= 1/d\sqrt{4.005}$ and obtain that if $N_0=o(N_1^2)$
then~\eqref{eq:main} is valid with the probability specified
in~\eqref{eq:dist_1subs_loc1}.


\subsection{Proof of Theorem~\ref{thm:hlm}: From Local Probabilistic Estimates to Global Ones}
\label{sec:proofs_intro}

\subsubsection{Outline of the Proof}
The proof verifies three different propositions and then combines them to conclude Theorem~\ref{thm:hlm}. We use the following notation: For any  subspace $\hat{\rmL} \in \GDd$
such that $\dG(\hat{\rmL},\rmL^*_1)=1$, we let $\rmL(t): [0,1]
\rightarrow \GDd$ denote the parametrization of the geodesic
line from $\rmL^*_1$ to $\hat{\rmL}$ such that $\rmL(0)=\rmL^*_1$ and $\rmL(1)=\hat{\rmL}$.
Using this notation and the setting of Theorem~\ref{thm:hlm}, the propositions are formulated as follows:
\begin{proposition}
\label{prop:aux1}
For any fixed $0<p \leq 1$ there exists $\gamma_1 = \gamma_1(p,D,d, \alpha_1, \alpha_0)$ such that
w.o.p.~for any $\hat{\rmL} \in \GDd$ satisfying $\dG(\hat{\rmL},\rmL^*_1)=1$ with the corresponding geodesic parametrization $\rmL(t)$ from $\rmL^*_1$ to $\hat{\rmL}$:
\begin{equation} \left. \ddtp \left(\frac{\sum_{\bx\in \rmX}\dist(\bx,\rmL(t))^p}{N}\right)\right|_{t=0}
>\gamma_1. \label{eq:derivativeii}\end{equation}
\end{proposition}
\begin{proposition}
\label{prop:aux2}
For any fixed $0<p \leq 1$ there exists $0<\gamma_2 = \gamma_2(p,D,d, \alpha_1, \alpha_0)<1$ such that w.o.p.~for any $t_0\in
[0,\gamma_2]$ and any $\hat{\rmL} \in \GDd$ satisfying $\dG(\hat{\rmL},\rmL^*_1)=1$ with the corresponding geodesic parametrization $\rmL(t)$ from
$\rmL^*_1$ to $\hat{\rmL}$:
\begin{equation}\left. \ddtp \left(\frac{\sum_{\bx\in \rmX}\dist(\bx,\rmL(t))^p}{N}\right)\right|_{t=0}-\left.\ddtp
\left(\frac{\sum_{\bx\in \rmX}\dist(\bx,\rmL(t))^p}{N}\right)
\right|_{t=t_0}<\frac{\gamma_1}{2},
\label{eq:local6}
\end{equation}
where $\gamma_1$ is the constant guaranteed by Proposition~\ref{prop:aux1} for this value of $p$.
\end{proposition}
\begin{proposition}
\label{prop:aux3}
For any fixed $0<p\leq 1$ and $\gamma_2$, the constant guaranteed by Proposition~\ref{prop:aux2}:
\begin{equation}\label{eq:global0}\text{
$\rmL^*_1$ is a global $\ell_p$ subspace
w.o.p.~in $\GDd\setminus \ballG(\rmL^*_1,\gamma_2)$. } \end{equation}
\end{proposition}

Theorem~\ref{thm:hlm} immediately concludes from these three propositions.
Indeed, Propositions~\ref{prop:aux1} and~\ref{prop:aux2} imply that
the function $e_{\ell_p}(\rmX, \rmL(t))$: [0,1] $\rightarrow \mathbb{R}$
of~\eqref{eq:def_error_1sub} has a positive derivative w.o.p.~at any
$t\in[0,\gamma_2]$ (as explained in \S\ref{subsec:comb_cond23}
we use the derivative with respect to the variable $t^p$). That is,
\begin{equation}\ddtp \left(\frac{\sum_{\bx\in \rmX}\dist(\bx,\rmL(t))^p}{N}\right)>0
\  \text{ for all $t\in[0,\gamma_2]$ \
w.o.p.}\label{eq:derivative5}\end{equation}

Equation \eqref{eq:derivative5} implies that
w.o.p.~$\rmL^*_1$ is the global $\ell_p$ subspace in
$\ballG(\rmL^*_1,\gamma_2)$. Combining it with \eqref{eq:global0}, we conclude Theorem~\ref{thm:hlm}.


We prove Proposition~\ref{prop:aux1} when $p=1$ in \S\ref{sec:proof_aux1_p1}
and when $0<p<1$ in \S\ref{sec:proof_aux1_smallp}; Proposition~\ref{prop:aux2}
in  \S\ref{sec:proof_aux2}; and Proposition~\ref{prop:aux3} in \S\ref{sec:global0}.
At last, \S\ref{sec:dependence_on_D} estimates the asymptotic dependence of the overwhelming probability in
Theorem~\ref{thm:hlm} and the minimal size $N$ on $d$ and $D$.

\subsubsection{Proof of Proposition~\ref{prop:aux1} for \boldsymbol{$p=1$}}
\label{sec:proof_aux1_p1}
We decompose the sampled data set as follows: $\rmX=\cup_{i=0}^K\rmX_i$, where $\rmX_i$ is the set of points sampled from $\mu_i$ for all $0\leq i\leq K$. It follows from~\eqref{eq:main} that the event in \eqref{eq:derivativeii} is the
same as the event
\begin{align}&\frac{\sum_{\bx\in\rmX_1}\|\bC\bV P_{\rmL^*_1}(\bx)\|-\|\bC\bV\bB_{\rmL^*_1,\rmX\setminus \rmX_1}\|_*}{N}
>\gamma_1 \ \ \label{eq:derivative1}\\ &\nonumber\ \forall \bC \in \NScp(d) \text{ and }
\bV \in \Or(d). \end{align}

We will prove~\eqref{eq:derivative1} in two steps. In the first step  we will fix matrices $\bC_0 \in
\NScp(d)$ and $\bV_0 \in \Or(d)$ and show that
\begin{align}&\frac{\sum_{\bx \in
\rmX_1}\|\bC_0\bV_0 P_{\rmL^*_1}(\bx)\|-\|\bC_0 \bV_0
\mathbf{B}_{\rmL^*_1,\rmX\setminus \rmX_1}\|_{*}}{N}>2\gamma_1
\,\,\,\label{eq:deduction2}\\ &\text{w.p.}\geq\,1-(2D^2+1)\exp(-2N\gamma_1^2),\nonumber
\end{align}
where
$$\gamma_1:=\beta_0 \min_{\bC \in \NScp(d),\bV \in \Or(d)}
E_{\mu_1} \|\bC \bV P_{\rmL^*_1}(\bx) \|/6 \ \
\text{ and } \ \ \beta_0 = \alpha_1-\sum_{j=2}^{K}\alpha_j.$$
In the second step we will combine a  covering argument and \eqref{eq:deduction2} to prove~\eqref{eq:derivative1}.

\subsubsection*{Step 1: Proof of~\eqref{eq:deduction2}}
We will first verify the following two probabilistic inequalities:
\begin{equation}\label{eq:first_step_a}
\frac{\|\bC_0 \bV_0
\mathbf{B}_{\rmL^*_1,\rmX_0}\|_{*}}{N}<2\gamma_1\,\,\text{w.p.~$1 - 2 \,D^2 \exp(2\gamma_1^2\, N)$}
\end{equation}
and
\begin{align} \label{eq:taekwon} &\frac{\sum_{\bx \in
\rmX_1}\|\bC_0\bV_0 P_{\rmL^*_1}(\bx)\|-\sum_{\bx \in \rmX\setminus
\{\rmX_1\cup\rmX_0\}}\|\bC_0 \bV_0P_{\rmL^*_1}(\bx)\|}{N}>4\gamma_1
\,\,\,\\ &\text{w.p.}\geq\,1-\exp(-2N\gamma_1^2).\nonumber
\end{align}

\subsubsection*{Part I of Step 1: Proof of~\eqref{eq:first_step_a}}

We define 
$J_0(\bx)=I(\bx\in\rmX_0)\,P_{\rmL^*_1}(\bx)P^\perp_{\rmL^*_1}
(\bx)^{T}/\dist(\bx,\rmL^*_1)$.
We note that its elements lie in $[-1,1]$ and $E_{\mu_0}(J_0(\bx)) = \b0$. Indeed, denoting $R_{\rmL^*_1}(\bx) = \hat{P}_{\rmL^*_1}(\bx)-\hat{P}^\perp_{\rmL^*_1}(\bx)$ (i.e.,
$R_{\rmL^*_1}(\bx)$ is the reflection of $\bx$ w.r.t.~the $d$-subspace $\rmL^*_1$) we obtain that
\begin{equation}2E_{\mu_0}(J_0(\bx))=E_{\mu_0} (J_0(\bx))-E_{\mu_0}(J_0(R_{\rmL^*_1}(\bx)))=
\b0,
\nonumber
\end{equation}
where the first equality is clear since $P_{\rmL^*_1}(\bx)P^\perp_{\rmL^*_1} (\bx)^{T}=
-P_{\rmL^*_1}(R_{\rmL^*_1}(\bx))P^\perp_{\rmL^*_1} (R_{\rmL^*_1}(\bx))^{T}$ and the second one follows from
the symmetry of $\mu_0$.
Therefore, combining the fact that
\[\bD_{i,j}=\be_i^T\bD\be_j\leq\max_{\bu,\bv\in\mathbb{R}^D}\bu^T\bD\bv/\|\bu\|\|\bv\|=\|\bD\|_*,\]
for any $\bD\in\mathbb{R}^{D\times D}$ and $1\leq i,j\leq N$, and
Hoeffding's inequality for the random variable $J_0(\bx)$, we establish the following inequality, which clearly implies~\eqref{eq:first_step_a}:
\begin{align}
&\Pr\left(\|\sum_{\bx
\in  \rmX_0}P_{\rmL^*_1}(\bx)P^\perp_{\rmL^*_1}
(\bx)^{T}/\dist(\bx,\rmL^*_1)\|_{*}/N<2\gamma_1\right)\nonumber\\
\geq&\Pr\left(\|\sum_{\bx
\in  \rmX_0}P_{\rmL^*_1}(\bx)P^\perp_{\rmL^*_1}
(\bx)^{T}/\dist(\bx,\rmL^*_1)\|_{\infty}/N<2\gamma_1\right)
\geq1 - 2 \,D^2 \exp(2\gamma_1^2\, N).
\label{eq:nuclear_sX_0}
\end{align}

\subsubsection*{Part II of Step 1: Proof of~\eqref{eq:taekwon}}

We define the random variable $J_1(\bx)=(I(\bx\in
\rmX_1)-I(\bx\in
\rmX\setminus\{\rmX_1\cup\rmX_0\}))\|\bC_0\bV_0P_{\rmL^*_1}(\bx)\|$ and using the spherical symmetry of $\{\mu_i\}_{i=1}^{K}$, we have
\begin{align}\label{eq:distri} E_{\mu}&(J_1(\bx)) = E_{\mu^N}\left(\frac{\sum_{\bx \in \rmX_1}\|\bC_0\bV_0
P_{\rmL^*_1}(\bx)\|-\sum_{\bx \in \rmX\setminus \{\rmX_1\cup\rmX_0\}}\|\bC_0
\bV_0P_{\rmL^*_1}
(\bx)\|}{N}\right) \\\nonumber
&=\alpha_1E_{\mu_1}\|\bC_0\bV_0P_{\rmL^*_1}(\bx)\|
-\sum_{j=2}^{K}\alpha_jE_{\mu_j}\|\bC_0\bV_0P_{\rmL^*_1}(\bx)\|\\\nonumber
&\geq\alpha_1E_{\mu_1}\|\bC_0\bV_0P_{\rmL^*_1}(\bx)\|
-\sum_{j=2}^{K}\alpha_jE_{\mu_1}\|\bC_0\bV_0P_{\rmL^*_1}(\bx)\|
\\\nonumber
&= \beta_0 E_{\mu_1}\|\bC_0\bV_0P_{\rmL^*_1}(\bx)\| \geq 6 \gamma_1.
\end{align}
We conclude \eqref{eq:taekwon} by applying Hoeffding's inequality to the random
variable $J_1(\bx)$, while using the facts that its expectation is larger than $6\gamma_1$ and its values are
in $[-1,1]$.

\subsubsection*{Part III of Step 1: Conclusion of~\eqref{eq:deduction2} via~\eqref{eq:first_step_a} and~\eqref{eq:taekwon}}

We first observe that
\begin{equation}
\|\bC_0\bV_0\mathbf{B}_{\rmL^*_1,\rmX\setminus \rmX_1}\|_{*}\leq \|\bC_0\bV_0\mathbf{B}_{\rmL^*_1,\rmX\setminus \{\rmX_1\cup\rmX_0\}}\|_{*}+
\|\bC_0\bV_0\mathbf{B}_{\rmL^*_1,\rmX\setminus \rmX_0}\|_{*}
\label{eq:barea} \end{equation}
and
\begin{align}&\|\bC_0 \bV_0
\mathbf{B}_{\rmL^*_1,\rmX\setminus \{\rmX_1\cup\rmX_0\}}\|_{*}=\|\bC_0 \bV_0\sum_{\bx
\in \rmX\setminus \{\rmX_1\cup\rmX_0\}}P_{\rmL^*_1}(\bx)P^\perp_{\rmL^*_1}
(\bx)^{T}/\dist(\bx,\rmL^*_1)\|_{*}\label{eq:B_CV}\\
&\leq \!\!\!\sum_{\bx \in \rmX\setminus \{\rmX_1\cup\rmX_0\}}\!\!\!\|\bC_0 \bV_0
P_{\rmL^*_1}(\bx)P^\perp_{\rmL^*_1} (\bx)^{T}/\|P^\perp_{\rmL^*_1}
(\bx)\|\,\|_{*}\leq \!\!\!\sum_{\bx \in \rmX\setminus \{\rmX_1\cup\rmX_0\}}\!\!\!\|\bC_0 \bV_0
P_{\rmL^*_1}(\bx)\|.\nonumber
\end{align}
Applying~\eqref{eq:barea} and~\eqref{eq:B_CV}, we bound the LHS of~\eqref{eq:deduction2} by the difference
between the LHS of \eqref{eq:first_step_a} and the LHS of \eqref{eq:taekwon} as follows:
\begin{align}
\label{eq:barea2}
&\frac{\sum_{\bx \in
\rmX_1}\|\bC_0\bV_0 P_{\rmL^*_1}(\bx)\|-\|\bC_0 \bV_0
\mathbf{B}_{\rmL^*_1,\rmX\setminus \rmX_1}\|_{*}}{N}\\\ \nonumber \geq& \frac{\sum_{\bx \in
\rmX_1}\|\bC_0\bV_0 P_{\rmL^*_1}(\bx)\|-\|\bC_0\bV_0\mathbf{B}_{\rmL^*_1,\rmX\setminus \{\rmX_1\cup\rmX_0\}}\|_{*}-
\|\bC_0\bV_0\mathbf{B}_{\rmL^*_1,\rmX\setminus \rmX_0}\|_{*}}{N}
\\ \nonumber \geq & \frac{\sum_{\bx \in
\rmX_1}\|\bC_0\bV_0 P_{\rmL^*_1}(\bx)\|-\sum_{\bx \in \rmX\setminus
\{\rmX_1\cup\rmX_0\}}\|\bC_0 \bV_0P_{\rmL^*_1}(\bx)\|}{N}-\frac{\|\bC_0 \bV_0
\mathbf{B}_{\rmL^*_1,\rmX_0}\|_{*}}{N}.
\end{align}
Equation~\eqref{eq:deduction2} is thus an immediate consequence of \eqref{eq:first_step_a}, \eqref{eq:taekwon} and \eqref{eq:barea2}.

\subsubsection*{Step 2: Conclusion of~\eqref{eq:derivative1} via \eqref{eq:deduction2} and a covering argument}

We recall that~\eqref{eq:derivative1} needs to be verified for all matrices $\bC \in \NScp(d)$ and $\bV \in \Or(d)$. We define
\begin{align}
\label{eq:dist_mat_pair}
\dist_{\NScp(d) \times \Or(d)}((\bC_1,\bV_1),(\bC_2,\bV_2)):=\max(\|\bC_1-\bC_2\|_2,\|\bV_1-\bV_2\|_2)
\end{align}
and note that whenever
$\dist_{\NScp(d) \times \Or(d)}((\bC_1,\bV_1),(\bC_2,\bV_2))<\gamma_1/2$ and $\bx\in
\ball(\mathbf{0},1)$ we have that
\begin{align}&\|\bC_1\bV_1
P_{\rmL^*_1}(\bx)\|-\|\bC_2\bV_2
P_{\rmL^*_1}(\bx)\|\nonumber\\
&=(\|\bC_1\bV_1 P_{\rmL^*_1}(\bx)\|-\|\bC_2\bV_1
P_{\rmL^*_1}(\bx)\|)+(\|\bC_2\bV_1 P_{\rmL^*_1}(\bx)\|-\|\bC_2\bV_2
P_{\rmL^*_1}(\bx)\|)\nonumber\\&\leq
\|\bC_1-\bC_2\|_2+\|\bC_2\|_2\|\bV_1-\bV_2\|_2\leq \gamma_1
.\label{eq:distance}\end{align}
Combining~\eqref{eq:deduction2} and~\eqref{eq:distance} we obtain
that for $(\bC,\bV)$ in a ball in $\NScp(d) \times \Or(d)$ of radius $\gamma_1/2$ and center
$(\bC_0,\bV_0)$:
\begin{equation}\label{eq:for_ball}\frac{\sum_{\bx \in \rmX_1}\|\bC\bV
P_{\rmL^*_1}(\bx)\|-\sum_{\bx \in \rmX\setminus \rmX_1}\|\bC \bV
P_{\rmL^*_1}(\bx)\|}{N}>\gamma_1
\,\,\,\text{w.p.}\geq\,1-\exp(-2N\gamma_1^2).\end{equation}%

We easily extend~\eqref{eq:for_ball} for all pairs of matrices
$(\bC,\bV)$ in the compact space $\NScp(d) \times \Or(d)$ (with the
distance specified in~\eqref{eq:dist_mat_pair}). Indeed,
it follows from~\cite[Theorem 7]{Szarek97metricentropy}
that $\Or(d)$ can be covered by
$C_1'^{d(d-1)/2}$ $/(\gamma_1/2)^{d(d-1)/2}$ balls of radius
$\gamma_1/2$ for some $C_1'>0$
(note that the dimension of $\Or(d)$ is $d(d-1)/2$). Since $\NScp(d)$ is isomorphic to $\mathbb{S}^{d-1}$,
it follows from \cite[Lemma 5.2]{vershynin_book} that it can
be covered by $3^{d}$ $/(\gamma_1/2)^{d}$ balls of radius
$\gamma_1/2$.
Therefore, the product space $\NScp(d) \times \Or(d)$ with norm
defined in \eqref{eq:dist_mat_pair} can be covered by $C_1^{d(d+1)/2}/(\gamma_1/2)^{d(d+1)/2}$ balls of radius
$\gamma_1/2$, where $C_1:=\max(C_1',3)$, and consequently
\begin{multline}
\label{eq:est_prob_derivative1} \text{\eqref{eq:derivative1} is
valid for any $\bC \in \NScp(d)$ and $\bV \in \Or(d)$}
\\\text{w.p.}\,\,1-C_1^{d(d+1)/2}\exp(-2N\gamma_1^2)/(\gamma_1/2)^{d(d+1)/2},
\end{multline}
which means that \eqref{eq:derivativeii} with $p=1$ holds with the probability specified in \eqref{eq:est_prob_derivative1}.

\subsubsection{Proof of Proposition~\ref{prop:aux1} for \boldsymbol{$0<p<1$}}
\label{sec:proof_aux1_smallp}

When $0<p<1$, it follows from \eqref{eq:main2_a_p<1} and Hoeffding's
inequality  that~\eqref{eq:derivativeii} holds for any
$\bC \in \NScp(d)$ and $\bV \in \Or(d)$
w.p.~$1-\exp(-2N\gamma_1^2)$, where
$$\gamma_1:=\alpha_1 \cdot \min_{\bC\in\NScp(d), \bV\in
\Or(d)}E_{\mu_0}(\|\bC \bV
P_{\rmL^*_1}(\bx)\|^p)/2.$$

Following the same covering argument as in the proof of \eqref{eq:est_prob_derivative1}, we conclude that \eqref{eq:derivativeii} holds with the same probability specified in \eqref{eq:est_prob_derivative1} (though $\gamma_1$ is defined differently for $p=1$ and $0<p<1$).

\subsubsection{Proof of Proposition~\ref{prop:aux2}}
\label{sec:proof_aux2}

We verify~\eqref{eq:local6} by separating $\rmX$ into three parts:  $\rmX_1=\rmX\cap\rmL^*_1$, $\hat{\rmX}:=\{\bx
\in\rmX\setminus\rmX_1:\dist(\bx,\rmL^*_1)\leq 2\,\gamma_3\}$ ($\gamma_3$ will be clarified later) and
$\rmX \setminus (\rmX_1 \cup \hat{\rmX})$. Specifically, we will prove  that there exists $0<\gamma_2 <1$ such that for any
$t_0 \in [0,\gamma_2]$:
\begin{align}& \frac{1}{N} \sum_{\bx\in \rmX_1}
\left(\left.\ddtp\dist(\bx,\rmL(t))^p
 \right|_{t=0}-\left.\ddtp \dist(\bx,\rmL(t))^p
 \right|_{t=t_0}\right)<\frac{\gamma_1}{6},
\ \ \label{eq:local1}\end{align}
\begin{align}\frac{1}{N} \sum_{\bx\in \hat{\rmX}}
\left(\left.\ddtp\dist(\bx,\rmL(t))^p
 \right|_{t=0}-\left.\ddtp \dist(\bx,\rmL(t))^p
 \right|_{t=t_0}\right)
<\frac{\gamma_1}{6}. \label{eq:local2}\end{align}
and
\begin{equation}
\frac{1}{N} \sum_{\bx\in \rmX \setminus (\rmX_1 \cup \hat{\rmX})}
\left(\left.\ddtp\dist(\bx,\rmL(t))^p
 \right|_{t=0}-\left.\ddtp \dist(\bx,\rmL(t))^p
 \right|_{t=t_0}\right)
<\frac{\gamma_1}{6}.\label{eq:local3}
\end{equation}
We prove \eqref{eq:local1} and \eqref{eq:local3} deterministically and~\eqref{eq:local2} w.o.p.
Then \eqref{eq:local6} follows from the summation of  \eqref{eq:local1}, \eqref{eq:local2} and \eqref{eq:local3}.

In order to prove~\eqref{eq:local1}, we uniformly bound from above the terms of the sum in~\eqref{eq:local1} by a term of order $O(t_0^2)$.
For simplicity, let us assume that $p=1$. It follows from~\eqref{eq:inlier} and the fact
that the $\sinc$ function is decreasing
that for any
$\bx\in \rmX_1$ and any $0 \leq t_0 \leq 1$, the derivative $\frac{\di}{\di t}
\left(\dist(\bx,\rmL(t))\right)$ at $t=t_0$ is bounded 
below by
\begin{equation}
\label{eq:deriv_lower_bound}
\frac{\frac{\sin t_0}{t_0}\sum_{j=1}^{d}\theta_j
|\mathbf{v}_j\cdot\bx|^2 t_0
\theta_j\cos(t_0)}{\sqrt{\sum_{j=1}^{d}|(\mathbf{v}_j\cdot
\bx)|^2(t\theta_j)^2}}=\frac{\sin t_0 \cos
t_0}{t_0}\sqrt{\sum_{j=1}^{d}|(\mathbf{v}_j\cdot
\bx)|^2 \theta_j^2}.
\end{equation}
We note that $\sqrt{\sum_{j=1}^{d}|(\mathbf{v}_j\cdot
\bx)|^2 \theta_j^2} \leq 1$
(indeed, the assumption $\dG(\hat{\rmL},\rmL^*_1)=1$ implies that
$\sum_{i=1}^d \theta_i^2=1$). Combining this observation with \eqref{eq:inlier1}
and~\eqref{eq:deriv_lower_bound}
we derive the following bound on the terms in the sum of~\eqref{eq:local1} when $p=1$:
\begin{equation}
\left.\ddtp\dist(\bx,\rmL(t))
 \right|_{t=0}-\left.\ddtp \dist(\bx,\rmL(t))
 \right|_{t=t_0} \leq \left( 1-\frac{\sin t_0 \cos
t_0}{t_0} \right)=O(t_0^2).
\end{equation}
Similarly, one can also uniformly bound these terms by an $O(t_0^2)$ term when $p<1$.
Therefore, one can choose a sufficiently small $\gamma_2$ such that \eqref{eq:local1} holds.

Next, we verify~\eqref{eq:local2}. Here we can bound the terms of the sum in~\eqref{eq:local2} by 2
(using an additional assumption; see below). However, we cannot bound them by a term that approaches zero
when $t_0$ approaches zero.
We thus control w.o.p.~the fraction of the cardinality of $\hat{\rmX}$ over $N$ by a sufficiently small constant. We fix $\gamma_3\equiv
\gamma_3(D,d,\gamma_1)\equiv \gamma_3(D,d,\alpha_0,\alpha_1,p)$
a sufficiently small constant such that
\begin{align}\label{eq:gamma3}
\mu(\bx \in \mathbb{S}^{D-1}: 0<\dist(\bx, \rmL^*_1)\leq 2\,\gamma_3)\leq \gamma_1/24.
\end{align}
By  applying Hoeffding's inequality to the indicator function of ${\hat{\rmX}}$,
$I_{\hat{\rmX}}(\bx)$, while using the facts that $E(I_{\hat{\rmX}}(\bx))= \mu(\bx:
\bx\in\hat{\rmX}) \leq \gamma_1/24$ and $I_{\hat{\rmX}}(\bx)$ takes values
in $[0,1]$, we bound the size of $\hat{\rmX}$ as follows:
\begin{equation}\label{eq:control_ratios}
\frac{\#(\hat{\rmX})}{N} = \frac{\#(\hat{\rmX})}{\#(\rmX)}\leq
\gamma_1/12 
\ \ \text{w.p. $1-\exp(-2N(\gamma_1/24)^2)$}. \end{equation}%

Now for $\bx \in \hat{\rmX}$, we claim that the derivative according to $t^p$ of $\dist(\bx,\rmL(t))^p$
takes values in $[-1,1]$ (this requires an additional assumption when $p<1$).
When $p=1$ it is easiest to see it by directly applying the definition of the derivative to $\di (\dist(\bx,\rmL(t)))/\di t$ and then using Lemma~\ref{lemma:distance} to control the corresponding difference of distances. When $p<1$, we introduce the harmless assumption: $\gamma_2<\gamma_3$.
One may conclude the bound in this case by applying~\eqref{eq:def_deriv_tp},
the former bound on the derivative (when $p=1$) and bounding $t^{1-p}/\dist(\bx,\rmL(t))^p$ by 1.
The latter bound follows from the observation that for any $t \in [0, \gamma_2]$:
$t \leq \gamma_3 \leq \dist(\bx,\rmL(t))$, which can be concluded by the followings: Application of
Lemma~\ref{lemma:distance} with $\rmL_1 = \rmL(0)$ and
$\rmL_2 = \rmL(t)$, basic estimates, the definitions of $\gamma_2$, $\gamma_3$ and $\hat{\rmX}$ and the assumption $\gamma_2<\gamma_3$.
Thus the elements in the sum of~\eqref{eq:local2} are bounded by 2 (assuming $\gamma_2<\gamma_3$).
This observation and~\eqref{eq:control_ratios} imply that \eqref{eq:local2} holds for $t_0 \in [0,1]$ with
the probability specified in~\eqref{eq:control_ratios}.

Finally, in order to verify~\eqref{eq:local3} we
apply the fundamental theorem of calculus
and rewrite~\eqref{eq:local1} as follows:
\begin{align}
& \frac{1}{N}
\int_{t=0}^{t_0}
\sum_{\bx\in \rmX \setminus (\rmX_1 \cup \hat{\rmX})}
\frac{\di^2}{\di (t^p)^2}   \dist(\bx,\rmL(t))^p
 \di t  <\frac{\gamma_1}{6}.
\label{eq:local1_aux}\end{align}
Differentiating~\eqref{eq:outlier} and \eqref{eq:main2_a_p<1} one more time, we obtain that for
$\bx\in \rmX \setminus (\rmX_1 \cup \hat{\rmX})$,  the
second derivative of $\dist(\bx,\rmL(t))$ with respect to $t^p$ is bounded by
$C(d)/\gamma_3^3$, where $C(d)$ is in the order of $d^2$. Thus we can choose
$\gamma_2\equiv\gamma_2(\gamma_1,\gamma_3,d)\equiv
\gamma_2(\alpha_0,\alpha_1,d,D,p)$ such that $\gamma_2C(d)/\gamma_3^3<\gamma_1/6$ and then
\eqref{eq:local3} holds. Equation~\eqref{eq:local6}
is thus verified by combining~\eqref{eq:local1}, \eqref{eq:local2}
and \eqref{eq:local3}, and it holds with the probability specified in \eqref{eq:control_ratios}.


%

\subsubsection{Proof of Proposition~\ref{prop:aux3}}
\label{sec:global0}

Applying
Lemma~\ref{lemma:mean} we obtain that for all $2\leq i\leq K$:
\begin{equation}\label{eq:to_final1}
E_{\mu_1}\left(\dist(\bx,\rmL)^p-\dist(\bx,\rmL^*_1)^p)+E_{\mu_i}(\dist(\bx,\rmL)^p-\dist(\bx,\rmL^*_1)^p\right)\geq
0. \end{equation}
Further application of Lemma~\ref{lemma:ell_dist_est} with  $\rmL\in\GDd\setminus \ballG(\rmL^*_1,\gamma_2)$  results in the
inequality:
\begin{equation}E_{\mu_1}(\dist(\bx,\rmL))>0.88 \cdot 2^{\frac{3p}{2}} \cdot \pi^{-\frac{(2p+1)}{2}}\cdot ({d+p})^{-p/2} \cdot\gamma_2^p.
\label{eq:to_final2} \end{equation}%
Now, combining~\eqref{eq:to_final1} and~\eqref{eq:to_final2} we have
that
\begin{align}
&E_{\mu}(\dist(\bx,\rmL)^p-\dist(\bx,\rmL^*_1)^p)\nonumber\\
=&\sum_{i=2}^{K}\alpha_i ((E_{\mu_1}(\dist(\bx,\rmL)^p-\dist(\bx,\rmL^*_1)^p)+E_{\mu_i}(\dist(\bx,\rmL)^p-\dist(\bx,\rmL^*_1)^p))\nonumber\\&+\beta_0 E_{\mu_1}(\dist(\bx,\rmL)^p-\dist(\bx,\rmL^*_1)^p)
\label{eq:to_final3}
\geq
\beta_0\cdot0.88 \cdot 2^{\frac{3p}{2}} \cdot \pi^{-\frac{(2p+1)}{2}}\cdot ({d+p})^{-p/2} \cdot\gamma_2^p\,.
\end{align}
We define
\begin{equation}
\gamma_4 = 0.88 \cdot 2^{\frac{3p}{2}} \cdot \pi^{-\frac{(2p+1)}{2}}\cdot ({d+p})^{-p/2} \cdot\gamma_2^p\label{eq:to_final4}
\end{equation}
and note that it depends on $d$, $D$, $K$, $\alpha_0$, $\alpha_1$ and
$\min_{2\leq i\leq K}(\dG(\rmL^*_1,\rmL^*_i))$.
Applying Hoeffding's inequality to
$\dist(\bx,\rmL)-\dist(\bx,\rmL^*_1)$, whose absolute values are uniformly bounded by $1$ and its expectation is at least $\gamma_4$ (which follows from \eqref{eq:to_final3} and \eqref{eq:to_final4}), we obtain that for any
$\rmL\in\GDd\setminus \ballG(\rmL^*_1,\gamma_2)$:
\begin{align}
e_{\ell_p}(\rmX,\rmL)-e_{\ell_p}(\rmX,\rmL^*_1)>\gamma_4N/2 \ \
\text{w.p.} \, \geq 1-\exp(-N\gamma_4^2/8)\,.
\label{eq:global}\end{align}

By Lemma~\ref{lemma:distance} we have that for any $\rmL'\in\GDd$
satisfying $\dG(\rmL,\rmL')<(\gamma_4/4)^{1/p}$ and any $\bx\in
\ball(\mathbf{0},1)$:
\[|\dist(\bx,\rmL')^p-\dist(\bx,\rmL)^p|<\gamma_4/4.\]
Consequently, for any $\rmL\in\GDd\setminus \ballG(\rmL^*_1,\gamma_2)$
and all $\rmL'\in \ballG(\rmL,(\gamma_4/4)^{1/p})$:
\begin{align}
e_{\ell_p}(\rmX,\rmL')-e_{\ell_p}(\rmX,\rmL^*_1)>0 \,\,\, \text{w.p.}
\, \geq 1-\exp(-N\gamma_4^2/8)\,.\label{eq:global1}
\end{align}

We can cover $\GDd\setminus
\ballG(\rmL^*_1,\gamma_2)$ by $(C_2\sqrt{d})^{d(D-d)}/\gamma_4^{d(D-d)/p}$ balls
of radius $(\gamma_4/4)^{1/p}$ (this follows from Remark~8.4 of~\cite{szarek83}).
Now, for each such ball we have
that~\eqref{eq:global} is valid for its center
w.p.~$1-\exp(-N\gamma_4^2/8)$ and consequently~\eqref{eq:global1} is
valid for subspaces in that ball with the same probability. We thus
conclude that \begin{equation}\label{eq:global0_prob}\text{~\eqref{eq:global0} holds
w.p.~$1-\exp(-N\gamma_4^2/8) (C_2\sqrt{d})^{d(D-d)/p}/\gamma_4^{d(D-d)}$.}
\end{equation}

\subsubsection{Dependence of the Probability and \boldsymbol{$N$} on \boldsymbol{$d$} and \boldsymbol{$D$}}
\label{sec:dependence_on_D}
Applying the union bound for the events specified in \eqref{eq:derivativeii}, \eqref{eq:local6} and \eqref{eq:global0}, whose probabilities are specified in \eqref{eq:est_prob_derivative1}, \eqref{eq:control_ratios} and \eqref{eq:global0_prob} respectively,  we conclude that $\rmL^*_1$ is a global $\ell_1$ subspace w.p. at least
\begin{align}\label{eq:probability_Theorem1}
1-&C_1^{d(d+1)/2}\exp(-2N\gamma_1^2)/(\gamma_1/2)^{d(d+1)/2}-\exp(-2N(\gamma_1/24)^2)
\\-&\exp(-N\gamma_4^2/8) (C_2\sqrt{d})^{d(D-d)/p}/\gamma_4^{d(D-d)}.
\nonumber\end{align}
We bound~\eqref{eq:probability_Theorem1} from below by $1-C'\exp(-N/C)$, where
$C=1/\min(2(\gamma_1/24)^2,\gamma_4^2/8)$
and
$C'=C_1^{d(d+1)/2}/(\gamma_1/2)^{d(d+1)/2}+1+(C_2\sqrt{d})^{d(D-d)/p}/\gamma_4^{d(D-d)}.$

We cannot formulate nice expressions for $\gamma_1$ and $\gamma_4$, however, we can express their dependence on $D$ and $d$ as follows (assuming the rest of the parameters are fixed).
The definition of $\gamma_1$ in \eqref{eq:distri} implies that it is in the order of $1/d$. In order to estimate $\gamma_4$, we first need to estimate
$\gamma_3$ and $\gamma_2$.
The defining equation of $\gamma_3$, i.e.,
\eqref{eq:gamma3}, implies that $\gamma_3$ is in the order of $d^{-1}D^{-1/2}$ (a rigorous argument appears in \S\ref{sec:invertible_gamma}). We claim that
$\gamma_2$ is in the order
of $d^{-6}D^{-1.5}$. Indeed, when proving~\eqref{eq:local1} we required that $\gamma_2 C(d)/\gamma_3^3<\gamma_1/6$ and $C(d)=O(d^2)$. 
At last, applying \eqref{eq:to_final4} and the estimate above for $\gamma_2$, we conclude that $\gamma_4$ is in the order of $d^{-6.5\,p}D^{-1.5 p}$. Therefore, $C=O(d^{\max(13p,2)}D^{3p})$ and $C'=O(d^{d(d+1)/2}+d^{6.5d(D-d)}D^{1.5d(D-d)})$.

We can use these estimates for $C$ and $C'$ and thus for the probability $1-C'\exp(-N/C)$ to obtain an estimate of the dependence of the minimal size $N$ on $D$ and $d$ in the asymptotic case. Assume that $N, D\rightarrow\infty$ and
$N / (d^{\max(13p,2)+1}D^{3p}\max(D-d,d+1)\log(D))\rightarrow \infty,
$
then the probability $1-C'\exp(-N/C)$ approaches $0$.
That is, asymptotically $N= \Omega(d^{\max(13p,2)+1}D^{3p}\max(D-d,d+1)\log(D))$.
This estimate, which indicates significant oversampling for the single subspace recovery,
is not tight and tighter estimates are left for future work.

\subsection{Proof of Theorem~\ref{thm:noisy_hlm}: Stability Analysis}
\label{sec:proofs_noise}

\subsubsection{Reduction of Theorem~\ref{thm:noisy_hlm}}\label{sec:noisy_hlm1}

We first explain how to reduce the proof of
Theorem~\ref{thm:noisy_hlm} when $0<p\leq 1$ to the verification of
a simpler statement. We then adapt this idea for proving the same
theorem when both $p> 1$ and $K=1$.

In order to prove Theorem~\ref{thm:noisy_hlm} when $0<p\leq 1$,
i.e., prove that the global minimum of $e_{\ell_p}(\rmX, \rmL)$ is
in $\ballG(\rmL^*_1,f)$ w.o.p., we only need to show that there exists
a constant $\rho_1>0$ such that for any $\rmL \notin
\ballG(\rmL^*_1,f)$:
\begin{equation}\label{eq:noise}
E_{\mu_\eps}(e_{\ell_p}(\bx, \rmL))>E_{\mu_\eps}(e_{\ell_p}(\bx,
\rmL^*_1))+\rho_1.
\end{equation}
Indeed, we cover the compact space $\GDd\setminus \ballG(\rmL^*_1,f)$
with small balls of radius $\rho_1/2$. Then by using
\eqref{eq:noise} and Hoeffding's inequality, we obtain that
$e_{\ell_p}(\rmX, \rmL)>e_{\ell_p}(\rmX, \rmL^*_1)$ for any $\rmL$ in
each such ball w.o.p. Therefore, $e_{\ell_p}(\rmX,
\rmL)>e_{\ell_p}(\rmX, \rmL^*_1)$ for  $\rmL\in\GDd\setminus
\ballG(\rmL^*_1,f)$ w.o.p.  Equivalently, $\GDd\setminus
\ballG(\rmL^*_1,f)$ does not contain the global minimum of
$e_{\ell_p}(\rmX, \rmL)$ w.o.p. By a similar argument as in \S\ref{sec:global0}, the probability is at least $1-\exp(-N\rho_1^2/8) (C_2\sqrt{d})^{d(D-d)/p}/\rho_1^{d(D-d)}$.

We will prove \eqref{eq:noise} with
\begin{equation}
\label{eq:noise_rho} \rho_1=2\eps^p
\end{equation}
and thus obtain the probability specified in \eqref{eq:thm2_prob}.

We further reduce~\eqref{eq:noise} by using the measure $\mu$ instead of $\mu_\eps$ (see \S\ref{subsec:precise}).
Combining the triangle inequality and the concavity of $x^p$ we obtain that
\begin{align}
\nonumber &|E_{\mu_i+\nu_{i,\eps}}(e_{\ell_p}(\bx,
\rmL))-E_{\mu_i}(e_{\ell_p}(\bx, \rmL))|=|E_{\mu_i+\nu_{i,\eps}} (\|P_{\rmL^\perp}(\bx)\|^p-\|P_{\rmL^\perp}(\hat{P}_{\rmL_i^*}(\bx))\|^p)|
\\ &\leq E_{\mu_i+\nu_{i,\eps}}\|P_{\rmL^\perp}(\hat{P}_{\rmL_i^{*\perp}}(\bx))\|^p
\leq E_{\mu_i+\nu_{i,\eps}}\|P_{\rmL_i^{*\perp}}(\bx)\|^p
=E_{\nu_{i,\eps}}\|\bx\|^p \leq \eps^p.
\label{eq:reduce_for_thm2_1}
\end{align}
Summing \eqref{eq:reduce_for_thm2_1} over all $1\leq i\leq K$, we have
\begin{equation}
\label{eq:reduce_for_thm2_2}
|E_{\mu_\eps}(e_{\ell_p}(\bx,
\rmL))-E_{\mu}(e_{\ell_p}(\bx, \rmL))|\leq \eps^p.
\end{equation}
Hence, in order to prove \eqref{eq:noise} and thus
Theorem~\ref{thm:noisy_hlm} for $p\leq 1$, the following equation is
sufficient:
\begin{equation}\label{eq:noise1}
E_{\mu}(e_{\ell_p}(\bx,
\rmL))>E_{\mu}(e_{\ell_p}(\bx,
\rmL^*_1))+\rho_1+2\eps^p,\,\,\,\text{for any $\rmL\in\GDd\setminus
\ballG(\rmL^*_1,f)$}.
\end{equation}

We can similarly reduce Theorem~\ref{thm:noisy_hlm} when $K=1$ and
$p>1$. However, \eqref{eq:reduce_for_thm2_1} needs to be modified since $x^p$ is not
concave when $p>1$. For this purpose we note that for any $\bx_1,\bx_2 \in \ball(\mathbf{0},1)$
\begin{equation}
\dist(\bx_1,{\rmL}^*_1)^p-\dist(\bx_2,{\rmL}^*_1)^p<1-(1-\dist(\bx_1,\bx_2))^p < p \cdot \dist(\bx_1,\bx_2).\label{eq:p>1_error_control1}
\end{equation}
Indeed, when $p=1$ \eqref{eq:p>1_error_control1} is immediate (it is equivalent to $\|P_{{\rmL}^*_1}(\bx_2-\bx_1)\| \leq \|\bx_2-\bx_1\|$)
and it extends to $p>1$ by the following proposition: if $0\leq y_1,y_2\leq
1$, $y_1-y_2<\eta$
and $p>1$, then $y_1^p-y_2^p<1-(1-\eta)^p$.
Combining~\eqref{eq:p>1_error_control1} with the derivation of~\eqref{eq:reduce_for_thm2_1}, we conclude the following analog of \eqref{eq:reduce_for_thm2_1} in the current case:
\begin{equation}
\label{eq:bound_p_eps}
|E_{\mu_\eps}(e_{\ell_p}(\bx,
\rmL))-E_{\mu}(e_{\ell_p}(\bx, \rmL))|\leq p \cdot \eps.
\end{equation}
Consequently, we reduce \eqref{eq:noise} and \eqref{eq:noise_rho}  (and thus Theorem~\ref{thm:noisy_hlm}) when $K=1$ and
$p>1$ to the following equations:
\begin{equation}
\label{eq:noise3}
E_{\mu}(e_{\ell_p}(\bx,\rmL))
>
E_{\mu}(e_{\ell_p}(\bx,
\rmL^*_1))+\rho_1+2p\eps,\,\,\,\text{for any $\rmL\in\GDd\setminus
\ballG(\rmL^*_1,f)$}
\end{equation}
and
\begin{equation}
\label{eq:noise_rho3} \rho_1 = 2 \cdot p \cdot \eps.
\end{equation}

\subsubsection{Proof of \eqref{eq:noise1} and \eqref{eq:noise3} and Conclusion of Theorem~\ref{thm:noisy_hlm} }
We arbitrarily fix $\rmL\in\GDd\setminus \ballG(\rmL^*_1,f)$. We assume
first that $0<p \leq 1$ and apply Lemma~\ref{lemma:mean} to obtain
that
\begin{align*}
&E_{\mu-(\alpha_1-\sum_{i=2}^{K}\alpha_i){\mu}_{1}}e_{\ell_p}(\bx,
\rmL)-E_{\mu-(\alpha_1-\sum_{i=2}^{K}\alpha_i){\mu}_{1}}e_{\ell_p}(\bx,
\rmL^*_1)\\
=&\sum_{i=2}^{K}\alpha_i\left(E_{{\mu}_{1}+\mu_i}e_{\ell_p}(\bx,
\rmL)-E_{{\mu}_{1}+\mu_i}e_{\ell_p}(\bx,
\rmL^*_1)\right) \geq 0.
\end{align*}
Consequently, we prove~\eqref{eq:noise1} with $\rho_1:= 2 \eps^p$ by Lemma~\ref{lemma:ell_dist_est}
as follows:
\begin{align}
\label{eq:super_bowl} &E_{\mu}(e_{\ell_p}(\bx,
\rmL))-E_{\mu}(e_{\ell_p}(\bx, \rmL^*_1))\geq
\left(\alpha_1-\sum_{i=2}^{K}\alpha_i\right)E_{{\mu}_{1}}(e_{\ell_p}(\bx,
\rmL))\\& \geq \frac{ \left(\alpha_1-\sum_{i=2}^{K}\alpha_i\right)\cdot f^p\cdot 2^{3p/2}\cdot 0.88}{  (d+p)^{p/2}\cdot \pi^{(2p+1)/2} } =4\eps^p,\nonumber
\end{align}
where the second inequality applies Lemma~\ref{lemma:ell_dist_est} and the last equality uses the fact that the term $\alpha_0 + 2 \cdot \alpha_1 - 1$ in the definition of $f$ equals $(\alpha_1-\sum_{i=2}^{K}\alpha_i)$.
Equation~\eqref{eq:definition_of_f} is obtained by solving for $f$ in the last equality of~\eqref{eq:super_bowl}.


Equation~\eqref{eq:noise3} (with $p>1$) follows from the same argument
of~\eqref{eq:super_bowl}, where $\eps^p$ is now replaced by $p\eps$.
Equation~\eqref{eq:definition_of_f2} is deduced in a similar way to~\eqref{eq:definition_of_f}, while using~\eqref{eq:noise_rho3} instead of~\eqref{eq:noise_rho}.

\subsection{Proof of Theorem~\ref{thm:phase_hlm}: Symmetry Arguments}
\label{sec:proofs_phase}

\subsubsection{Structure of the Proof}
We proceed with several reductions of the statement of the theorem. The first reduction (see \S\ref{sec:reduction1}) practically states that it is enough to prove w.p.~1 (under the measure $\gamma_{D,d}^K$) that $\rmL^*_1$ is not a global $\ell_p$ subspace in expectation, or equivalently,
$\rmL^*_1 \neq \argmin_{\rmL\in\GDd}E_{\mu}(e_{\ell_p}(\bx,\rmL))$.
In order to be able to prove this, we condition the probability measure on other events and thus ``reduce randomness''.
In the second reduction (see \S\ref{sec:reduction2}) we condition on $\rmL_1^*$, $\rmL_3^*$, $\rmL_4^*$, $\cdots$, $\rmL_K^*$
and in the third reduction (see \S\ref{sec:reduction3}) we condition on the principal vectors and principal angles of $\rmL_2^*$. We then
prove the final reduced statement in \S\ref{sec:reduction_conclusion}. At last, \S\ref{sec:delta_kappa} estimates the sizes of
$\delta_0$ and $\kappa_0$ and \S\ref{sec:no_asymptotic} uses the results of this section to show that exact asymptotic recovery is impossible in our setting with $K>1$ and any noise level $\eps>0$.

\subsubsection{First Reduction of Theorem~\ref{thm:phase_hlm}}
\label{sec:reduction1}
Theorem~\ref{thm:phase_hlm} states that the global $\ell_p$ subspace
is not in $\ballG(\rmL^*_1,\kappa_0)$ w.o.p.~for almost every
$\{\rmL^*_i\}_{i=1}^{K}\in\GDd^K$. We claim that it reduces to (or equivalently, implied by) the
following statement:
\begin{equation}
\label{eq:noise1p>1}
\gamma_{D,d}^{K}\left(\{\rmL^*_i\}_{i=1}^{K}\subset\GDd:
\rmL^*_1=\argmin_{\rmL\in\GDd}E_{\mu}(e_{\ell_p}(\bx,
\rmL))\right)=0.
\end{equation}
Indeed, if \eqref{eq:noise1p>1} is satisfied, then for
$\rmL_0=\argmin_{\rmL\in\GDd}E_{\mu}(e_{\ell_p}(\bx, \rmL))$ and any $K$
$d$-subspaces $\{\rmL^*_i\}_{i=1}^{K}$ in a subset of $\GDd^K$ with
nonzero $\gamma_{D,d}^K$ measure, the constant
$$\zeta_1:=E_{\mu}(e_{\ell_p}(\bx,
\rmL^*_1))-E_{\mu}(e_{\ell_p}(\bx, \rmL_0))$$
is positive.

For any $\rmL^*\in \ballG(\rmL^*_1,\kappa_0)$ and $\bx \in \supp(\mu) \subseteq \ball(\b0,1)$
\[
\dist(\bx,\rmL^*)^p-\dist(\bx,\rmL^*_1)^p
\leq 1^p- (1-\dist_G(\rmL^*,\rmL^*_1))^p \leq p \cdot \dist_G(\rmL^*,\rmL^*_1)
\]
and therefore
\begin{equation}\label{eq:p>1_error_control}
E_{\mu}(e_{\ell_p}(\bx,
\rmL^*))>E_{\mu}(e_{\ell_p}(\bx,
\rmL^*_1))-\kappa_0 \cdot p.
\end{equation}
Letting $\delta_0=\kappa_0=\zeta_1/4p\eps$, we obtain
from~\eqref{eq:bound_p_eps} (using the fact that $\eps < \delta_0$) and~\eqref{eq:p>1_error_control} that
\begin{align*}
&E_{{\mu}_{\eps}}(e_{\ell_p}(\bx,
\rmL^*))-E_{{\mu}_{\eps}}(e_{\ell_p}(\bx,
\rmL_0))>E_{\mu}(e_{\ell_p}(\bx,
\rmL^*))-E_{\mu}(e_{\ell_p}(\bx, \rmL_0))-2 \delta_0
p\\&> E_{\mu}(e_{\ell_p}(\bx,
\rmL^*_1))-E_{\mu}(e_{\ell_p}(\bx, \rmL_0))-2 \delta_0
p-\kappa_0 p=\frac{\zeta_1}{4}.
\end{align*}
Therefore, by Hoeffding's inequality:
\begin{equation}
\label{eq:in_israel}
e_{\ell_p}(\rmX, \rmL^*)-e_{\ell_p}(\rmX,
\rmL_0)>\frac{\zeta_1N}{8}\,\,\,\text{w.p. $1-\exp(-\frac{N\zeta_1^2}{32})$}.
\end{equation}

At last, we prove w.o.p.~that
\begin{equation}
\label{eq:before_sab}
e_{\ell_p}(\rmX, \rmL^*)-e_{\ell_p}(\rmX,
\rmL_0)>0\,\,\,\text{ for all $\rmL^*\in \ballG(\rmL^*_1,\kappa_0)$.}
\end{equation}
To do this, we cover $\ballG(\rmL^*_1,\kappa_0)$ with small balls of radius
$\zeta_1/16$  so that $e_{\ell_p}(\rmX,
\rmL)>e_{\ell_p}(\rmX,
\rmL_0)$ for all $\rmL$ in each such ball w.o.p. Therefore,
$e_{\ell_p}(\rmX, \rmL)>e_{\ell_p}(\rmX, \rmL_0)$ for  all $\rmL\in
\ballG(\rmL^*_1,\kappa_0)$ w.o.p. Equivalently, $\ballG(\rmL^*_1,\kappa_0)$
will not contain the global minimum of
$e_{\ell_p}(\rmX, \rmL)$ w.o.p. This implies Theorem~\ref{thm:phase_hlm}.

We note that the number of covering balls can be the
$(\zeta_1/16)$-covering number of $\GDd$, which is $(C_2\sqrt{d})^{D(D-d)}$
$/(\zeta_1/16)^{D(D-d)}$ (see Section~\ref{sec:global0}). The combination of this observation with the probabilistic estimate
in \eqref{eq:in_israel} implies the following expression for
the probability of~\eqref{eq:before_sab} (which is the unspecified
failure probability of~\eqref{thm:phase_hlm}):
\begin{equation}
\label{eq:failure_prob}
1-\frac{(C_2\sqrt{d})^{D(D-d)}}{(\zeta_1/16)^{D(D-d)}}\exp\Big(-\frac{N\zeta_1^2}{32}\Big),\end{equation}
where $\zeta_1$ is later estimated in \eqref{eq:gamma1}.

\subsubsection{Second Reduction of Theorem~\ref{thm:phase_hlm}}
\label{sec:reduction2}
We define the operator
\begin{equation}\label{eq:def_B_p1}\mathbf{D}_{\rmL,\bx,p}=P_{\rmL}(\bx)P^\perp_{\rmL}
(\bx)^{T} \dist(\bx,\rmL)^{(p-2)}\end{equation}
and the function
\[
h(\rmL^*_1,\rmL^*_i)=E_{\mu_i}(\mathbf{D}_{\rmL^*_1,\bx,p}), \ \ 0\leq i \neq 1 \leq K.
\]

In view of Proposition~\ref{prop:comb_cond3}, \eqref{eq:noise1p>1} follows from the condition:
\begin{align}
\label{eq:reduce_noise1p>1_0}
\gamma_{D,d}^{K}\left(\{\rmL^*_i\}_{i=1}^{K}\subset\GDd:
E_{\mu}\left( \mathbf{D}_{\rmL^*_1,\bx,p}\right)=
0\right) =0,
\end{align}
%
which we rewrite as follows:
\begin{align}
\nonumber &\gamma_{D,d}^{K}\left(\{\rmL^*_i\}_{i=1}^{K}\subset\GDd:
E_{\mu}\left( \mathbf{D}_{\rmL^*_1,\bx,p}\right)= 0
\right)  \\=& \gamma_{D,d}^{K}\left(\{\rmL^*_i\}_{i=1}^{K}\subset\GDd:
E_{\sum\limits_{\latop{i=0}{i \neq 1}}^{K}\alpha_i\mu_i}\left(
\mathbf{D}_{\rmL^*_1,\bx,p}\right)= 0\right) \nonumber\\=&
\gamma_{D,d}^{K}\left(\{\rmL^*_i\}_{i=1}^{K}\subset\GDd:
\sum\limits_{\latop{i=0}{i \neq 1}}^{K}\alpha_i \,h(\rmL^*_1,\rmL^*_i)=0\right)=
0.\label{eq:reduce_noise1p>1}
\end{align}
Since $\{\rmL^*_i\}_{i=1}^{K}$ are identically and independently distributed according
to $\gamma_{D,d}$, Fubini's Theorem implies
that~\eqref{eq:reduce_noise1p>1} follows from the equation:
\begin{equation}\gamma_{D,d}\left(\rmL^*_2\in\GDd: h(\rmL^*_1,\rmL^*_2)=
\bH(\rmL^*_1,\rmL^*_3,\cdots,\rmL^*_K)\right)=0,\label{eq:prob_equal_0}
\end{equation}%
where
\begin{equation}
\label{eq:def_bH}
\bH(\rmL^*_1,\rmL^*_3,\cdots,\rmL^*_K)=-\sum\limits_{\latop{i=0}{i \neq 1,2}}^{K}\alpha_i
\,h(\rmL^*_1,\rmL^*_i)/\alpha_2.
\end{equation}

\subsubsection{Third Reduction of Theorem~\ref{thm:phase_hlm}}
\label{sec:reduction3}
We denote the principal angles
between $\rmL^*_2$ and $\rmL^*_1$ by $\{\theta_j\}_{j=1}^{d}$, the
principal vectors of $\rmL^*_2$ and $\rmL^*_1$ by
$\{\hat{\bv}_j\}_{j=1}^{d}$ and $\{\bv_j\}_{j=1}^{d}$ respectively
and the complementary orthogonal system for $\rmL^*_2$ w.r.t. $\rmL^*_1$
by $\{\bu_j\}_{j=1}^{d}$. Note that $h(\rmL^*_1,\rmL^*_2)$, as a function of $\bx$, maps $\Sp(\{\bu_i\}_{i=1}^{d})$ to
$\Sp(\{\bv_i\}_{i=1}^{d})$. Now,  transforming $\bx\in\rmL^*_2\cap
\ball(\mathbf{0},1)$ to $\{a_i\}_{i=1}^{d}$ in a $d$-dimensional unit
ball by $\bx=\sum_{i=1}^{d}a_i\hat{\bv}_i$, we have that for any
$1\leq i_1,i_2\leq d$:
\begin{align*}
&\bv_{i_1}^T h(\rmL^*_1,\rmL^*_2) \bu_{i_2}=E_{\mu_2}(\bv_{i_1}^T
\hat{P}_{\rmL^*_1}(\bx) \hat{P}_{\rmL^*_1}^\perp(\bx)^T\bu_{i_2}\dist(\bx,\rmL^*_1)^{p-2})\\
=&\int_{\sum_{i=1}^d{a_i}^2\leq
1}\cos(\theta_{i_1})a_{i_1}\sin(\theta_{i_2})a_{i_2}\left(\sum_{i=1}^{d}a_i^2\sin^2\theta_i\right)^{\frac{p-2}{2}}\di \mu_2.
\end{align*}

When $i_1\neq
i_2$, the function
$$\cos(\theta_{i_1})a_{i_1}\sin(\theta_{i_2})a_{i_2}\left(\sum_{i=1}^{d}a_i^2\sin^2\theta_i\right)^{\frac{p-2}{2}}$$
is odd w.r.t.~$a_{i_1}$ and consequently \begin{align*} \bv_{i_1}^T
h(\rmL^*_1,\rmL^*_2) \bu_{i_2}
=\int_{\sum_{i=1}^d{a_i}^2\leq
1}\cos(\theta_{i_1})a_{i_1}\sin(\theta_{i_2})a_{i_2}\left(\sum_{i=1}^{d}a_i^2\sin^2\theta_i\right)^{\frac{p-2}{2}} \di \mu_2 =0.
\end{align*} Therefore, when we form $\bV$ and $\bU$ as in
\eqref{eq:outlier1}, the $d\times d$ matrix
$\bV h(\rmL^*_1,\rmL^*_2)\bU^T$
is diagonal with the elements
$$\int_{\sum_{i=1}^d{a_i}^2\leq
1}\cos(\theta_{j})\sin(\theta_{j})a_{j}^2\left(\sum_{i=1}^{d}a_i^2\sin^2\theta_i\right)^{\frac{p-2}{2}} \di \mu_2,
\ \ \ j=1,\cdots,d.$$
Notice that $\bV h(\rmL^*_1,\rmL^*_2)=h(\rmL^*_1,\rmL^*_2)=h(\rmL^*_1,\rmL^*_2)\bU^T$ and that
$h(\rmL^*_1,\rmL^*_2)$ has the following singular values, where $j=1,\cdots,d$:
\[
\lambda_j(h(\rmL^*_1,\rmL^*_2))=\int_{\sum_{i=1}^d{a_i}^2\leq
1}\cos(\theta_{j})\sin(\theta_{j})a_{j}^2\left(\sum_{i=1}^{d}a_i2\sin2\theta_i\right)^{\frac{p-2}{2}}\di \mu_2.
\]
We arbitrarily fix
$\rmL^*_1$, $\rmL^*_3$, $\rmL^*_4$, $\cdots$, $\rmL^*_K$ and denote the
singular values of $\bH$ (which is defined in~\eqref{eq:def_bH})
by $\{\sigma_i\}_{i=1}^{D}$ and observe that
\eqref{eq:prob_equal_0} is implied by the following equation:
\begin{equation}\gamma_{D,d}\left( \rmL^*_2\in\GDd:
\lambda_1(h(\rmL^*_1,\rmL^*_2))\in \{\sigma_i\}_{i=1}^{D}\right)=0,
\end{equation}which we express as:
\begin{align} \label{eq:prob_equal_0_1} &\gamma_{D,d}\left(\int_{\sum_{i=1}^d{a_1}^2\leq
1}\cos(\theta_1)\sin(\theta_1)
a_1^2\left(\sum_{i=1}^{d}a_i^2\sin^2\theta_i\right)^{\frac{p-2}{2}}
\di \mu_2
\in\{\sigma_i\}_{i=1}^{D}\right)\\\nonumber &=0.\end{align}

\subsubsection{Proof of~\eqref{eq:prob_equal_0_1} and Conclusion of Theorem~\ref{thm:phase_hlm}}
\label{sec:reduction_conclusion}
We first conclude \eqref{eq:prob_equal_0_1} when $p=2$. In this case
\begin{multline}\lambda_1(h(\rmL^*_1,\rmL^*_2)) \equiv \int_{\sum_{i=1}^d{a_1}^2\leq 1}\cos(\theta_1)\sin(\theta_1)
a_1^2\left(\sum_{i=1}^{d}a_i^2\sin^2\theta_i\right)^{\frac{p-2}{2}} \di \mu_2\\
\equiv\int_{\sum_{i=1}^d{a_1}^2\leq
1}\cos(\theta_1)\sin(\theta_1) a_1^2 \di \mu_2 \end{multline}
is a monotone
function of $\theta_1$ on $[0,\pi/4]$ as well as $[\pi/4,\pi/2]$.
That is, the requirement that
$\lambda_1(h(\rmL^*_1,\rmL^*_2))\in\{\sigma_i\}_{i=1}^{D}$ can occur
only at discrete values of $\theta_1$ (at most $2D$) and consequently has
$\gamma_{D,d}$ measure 0, that is, \eqref{eq:prob_equal_0_1} (and
consequently \eqref{eq:noise1p>1}) is verified in this case.

If $p\neq 2$ and $\{\theta_i\}_{i=1}^{d-1}$ are fixed, then
\begin{equation}
\label{eq:monotone}
\int_{\sum_{i=1}^d{a_1}^2\leq
1}\cos(\theta_1)\sin(\theta_1)
a_1^2\left(\sum_{i=1}^{d}a_i^2\sin^2\theta_i\right)^{\frac{p-2}{2}} \di \mu_2
\end{equation}
is a monotone function of $\theta_d$. Following a similar
argument, we obtain that
\begin{equation}\gamma_{D,d}\left( \lambda_1(h(\rmL^*_1,\rmL^*_2))\in
\{\sigma_i\}_{i=1}^{D}|\{\theta_i\}_{i=1}^{d-1}\right)=0.\label{eq:fubini_fix_theta}
\end{equation}%
Combining~\eqref{eq:fubini_fix_theta} with Fubini's Theorem, we
conclude~\eqref{eq:prob_equal_0_1}.

\subsubsection{Remark on the Size of \boldsymbol{$\delta_0$} and \boldsymbol{$\kappa_0$}}
\label{sec:delta_kappa}

The above constants $\delta_0$ and $\kappa_0$ depend on other
parameters of the underlying spherically uniform HLM model in particular the
underlying subspaces $\{\rmL^*_i\}_{i=1}^{K}$. We recall that $\kappa_0=\delta_0=\zeta_1/4p$, where $\zeta_1=E_{\mu}(e_{\ell_p}(\bx,
\rmL_1^*))-\min_{\rmL\in\GDd}E_{\mu}(e_{\ell_p}(\bx,
\rmL))$. Therefore, in order to bound $\kappa_0$ and $\delta_0$ from below, we bound $\zeta_1$ from below as follows:
\begin{equation}\label{eq:gamma1}
\zeta_1 \geq \begin{cases}
\frac{p}{2}\| E_{\mu}(\mathbf{D}_{\rmL^*_1,\bx,p})\|_F^2,
 &\text{if $p\geq 2$;}\\
(p-1)p^\frac{1}{p-1} 2^\frac{p-4}{p-1} \| E_{\mu}(\mathbf{D}_{\rmL^*_1,\bx,p})\|_F^\frac{p}{p-1}
,&\text{if $1< p< 2$}.
\end{cases}\end{equation}
We include the proof of~\eqref{eq:gamma1} in \S\ref{sec:gamma1}.
It also leads to a lower bound for the constants $\delta_0$ and $\kappa_0$ of~\cite{lp_recovery_part2_11}, which is better than the one mentioned
there (\S4.5.5).

We derive~\eqref{eq:gamma1_example} from~\eqref{eq:gamma1} as follows.
We recall that~\eqref{eq:gamma1_example} applies to the case where $K=2$, $\alpha_0=0$,  $\dim(\rmL^*_1)=\dim(\rmL^*_2)=1$, $D=2$ and where $\mu_1$ and $\mu_2$ are uniform distributions on line segments centered on the origin and of length $2$ within $\rmL^*_1$ and $\rmL^*_2$. If $\theta$ is the angle between $\rmL^*_1$ and $\rmL^*_2$, then
\begin{equation}
\label{eq:gamma4_bound_example}
\|E_{\mu}(\mathbf{D}_{\rmL^*_1,\bx,p})\|=\alpha_2\cos(\theta)\sin(\theta)^{p-1}/(p+1).
\end{equation}
The lower bound for both $\kappa_0$ and $\delta_0$ in~\eqref{eq:gamma1_example} thus follows from~\eqref{eq:gamma1}, \eqref{eq:gamma4_bound_example} and the fact that $\kappa_0=\delta_0=\zeta_1/4p$.

\subsubsection{Implication of Proof: A Counterexample for Exact Asymptotic Recovery}
\label{sec:no_asymptotic}

Theorem~\ref{thm:noisy_hlm} established near recovery of $\rmL_1^*$ for a spherically uniform HLM measure $\mu_{\eps}$ when $\eps>0$ and $0<p \leq 1$. It is sometimes more desirable to have exact asymptotic $\ell_p$ recovery of $\rmL_1^*$. It means that
if $\rmX=\{\bx_1,\bx_2,\cdots,\bx_N\}$ is an i.i.d.~sample from $\mu_{\eps}$ and $\rmL_{(N)}$ is the minimizer of $e_{\ell_p}(\rmX,\rmL)$, then $\rmL_{(N)}$ converges to $\rmL_1^*$ w.p.~1 as $N$ approaches infinity.
However, this is generally not true for any $p>0$ when $K>1$ and $\eps>0$. Indeed, we provide here a simple counterexample, whose verification follows the proof of Theorem~\ref{thm:phase_hlm}.

We assume a measure
$\tilde{\mu}=\alpha_1\tilde{\mu}_1+\sum_{i=2}^K\alpha_i\mu_i$, where
$\{\mu_i\}_{i=2}^K$ are the uniform measures on $S^{D-1}\cap \rmL_i^*$
and $\tilde{\mu}_1$ is the uniform measure on the strip $\{\bx\in
S^{D-1}: \dist(\bx,\rmL_1^*)\leq \eps\}$. The symmetry of this strip
w.r.t.~$\rmL_1^*$ implies that
\begin{equation}
\int P_{\rmL^*_1}(\bx)P^{\perp}_{\rmL^*_1}(\bx)^T\dist(\bx,\rmL^*_1)^{p-2}
\di \tilde{\mu}_1(\bx) = 0.
\label{eq:counterexample1}
\end{equation}

Besides, it  follows from Proposition \ref{prop:comb_cond3} that a
necessary condition for $\rmL_1^*$ to be a local $\ell_p$ subspace in
expectation is
\begin{equation}\label{eq:counterexample}\int
P_{\rmL^*_1}(\bx)P^{\perp}_{\rmL^*_1}(\bx)^T\dist(\bx,\rmL^*_1)^{p-2}\di
\mu(x)=
0.\end{equation}

Combining \eqref{eq:counterexample1} and \eqref{eq:counterexample}, we conclude that
\begin{equation}
\sum_{i=0,i=2}^K \int
P_{\rmL^*_1}(\bx)P^{\perp}_{\rmL^*_1}(\bx)^T\dist(\bx,\rmL^*_1)^{p-2}\di
\mu_i(x)=
0.
\label{eq:counterexample2}
\end{equation}

However, the proof of~\eqref{eq:reduce_noise1p>1} implies that the
measure $\gamma_{D,d}^K$ of \eqref{eq:counterexample2}
w.r.t.~$\{\rmL^*_i\}_{i=1}^{K}$ is zero.
That is, a.e.~$\rmL_1^*$ (w.r.t.~$\gamma_{D,d}^K=1$) is not the global
$\ell_p$ subspace in expectation.
Consequently, a.e.~$\rmL_1^*$ is not the asymptotic global $\ell_p$
subspace (since exact asymptotic recovery is stronger than recovery in
expectation).

\section{Discussion}
\label{sec:conclusion}
We studied the effectiveness of $\ell_p$
minimization for recovering and nearly recovering the most significant subspace within outliers
w.o.p. Our setting assumed i.i.d.~sampling from a spherically uniform
HLM measure (and sometimes weakly spherically uniform HLM measure) with noise level $\eps \geq 0$.
A restricted setting is necessary and indeed we described some
typical cases where the global $\ell_p$ subspace is different than the most
significant subspace for all $0 < p < \infty$. In our particular study of significantly large fraction of outliers, we need the rather
strong restriction of spherically symmetric outliers, which is not necessary when limiting this fraction (see e.g., \cite{lp_recovery_part2_11}).

Our analysis provided some guarantees for the robustness to spherically uniform outliers
(or spherically symmetric outliers) of the single subspace recovery advocated in~\cite{Ding+06}. The recovery established here is for the theoretical minimizer of the energy and not for any algorithmic output. Both~\cite{robust_pca_ZL} and~\cite{LMTZ2012} followed some basic ideas of this paper in their analysis of a convex relaxation of~\eqref{eq:def_error_1sub} when $p=1$, while incorporating many other ideas. However, the theoretical guarantees of the latter works require a bound on the fraction of outliers and it is unclear if their algorithms can always recover the most significant subspace in our setting when $K>1$.

We further discuss possible and impossible extensions of this theory, some other implications and open problems.

\subsection{Beyond Spherically Uniform Distributions}

We can easily replace spherically uniform distributions with sub-Gaussian spherically symmetric distributions. For this purpose, we may apply the Hoeffding-type inequality for sub-Gaussian measures of Proposition 5.10 in \cite{vershynin_book}. Alternatively, if the data is projected onto the unit sphere than spherically symmetric distributions (and even some more general distributions) are mapped into spherically uniform distributions.

We may even relax the spherical symmetry of
$\{\mu_i\}_{i=1}^K$ within $\{\rmL_i^*\}_{i=1}^K$ and require instead approximate spherical symmetry within $\{\rmL_i^*\}_{i=1}^K$. That is, we require for $i=2, \cdots, K$ that there exist $\{\tilde{\mu}_i\}_{i=1}^K$
spherically symmetric distributions within $\rmL_i^*$ such that the derivatives $f_i:= \di \mu_i/ {\di \tilde{\mu}_i}$, $i=1, \ldots, K$ are bounded away from 0 and $\infty$. In this case, \eqref{eq:cond_alpha}
is replaced with
\begin{equation}\label{eq:cond_alpha_const}  \alpha_1 > \sum_{i=2}^K \frac{\sup(f_i)}{\inf(f_1)} \, \alpha_i. \end{equation}

On the other hand, a symmetry-type property of $\mu_0$ is crucial for the proof of
Theorems~\ref{thm:hlm} and~\ref{thm:noisy_hlm}, unless one may tolerate a restricted fraction of outliers~\cite{lp_recovery_part2_11}.

In Theorem~\ref{thm:dist_1subs_loc1} it is enough to assume that $\mu_0$ is symmetric with respect to $\rmL^*_1$.
It is even possible to assume a slightly weaker assumption: $E_{\mu_0}(\mathbf{D}_{\rmL^*_1,\bx,p})=0$, where $\mathbf{D}_{\rmL^*_1,\bx,p}$ is defined in~\eqref{eq:def_B_p1}.

\subsection{Affine Subspaces}
We restrict the theory of this paper to linear subspaces, since affine subspaces do not fit within the framework of spherically uniform (or spherically symmetric) measures.
The common strategy of using homogenous coordinates, which transform $d$-dimensional affine
subspaces in $\reals^D$ to $(d+1)$-dimensional linear subspaces in $\reals^{D+1}$, is not useful to us since it distorts the structure
of both noise and outliers.
On the other hand, the theory of~\cite{lp_recovery_part2_11}
can be generalized to affine subspaces (see \S5.6 of~\cite{lp_recovery_part2_11}).

\subsection{$\boldsymbol{p=1}$ Versus $\boldsymbol{0<p<1}$}
\label{sec:p1_min}
Our main theorems do not distinguish between $p=1$ and $0<p<1$.
However, Proposition~\ref{prop:comb_cond2} shows that many subspaces can be local $\ell_p$ subspaces when $p<1$ (in particular, $d$-subspaces spanned by subsets of outliers).
Such wealth of local minima clearly does not occur when $p = 1$.
An open problem is to estimate the number and
depth of local minima when $p = 1$ for spherically uniform HLM measures.

\subsection{The Non-convexity of Our Strategy}
\label{sec:non_convex}
Our setting is non-convex and we are not aware of efficient and theoretically guaranteed strategies to approximate the global minimizer.
Both Ding et al.~\cite{Ding+06} and Zhang et al.~\cite{MKF_workshop09} suggested heuristic methods to approximate a minimizer to this problem when $p=1$, but they did not guarantee them. It will be interesting to develop even partial theoretical guarantees, possibly for another strategy. It will also be interesting to know whether there is any practical advantage of trying to minimize~\eqref{eq:def_error_1sub} with $p=1$ instead of using a convex relaxation of this minimization.

In \S\ref{subsec:background} we discussed the result of
Hardt and Moitra~\cite{moitra_pca2012}, which implies that if the small set expansion problem has no efficient algorithm (which is unknown), then under some circumstances
(different than the ones here) subspace recovery with sufficiently high percentage of outliers cannot be done by an efficient algorithm.
It is interesting to know if it is true that any procedure that can exactly recover the underlying subspace in our setting with
arbitrarily large percentage of outliers must be inefficient. If it is true, we are then curious about the upper bound on the fraction of outliers
in our setting. After all, \cite{robust_pca_ZL} and~\cite{LMTZ2012} indicated higher fraction of outliers than Hardt and Moitra~\cite{moitra_pca2012} for our setting with $K=1$.

\appendix
\section{Supplementary Details}

\subsection{The auxiliary function $\boldsymbol{\psi}$ }
\label{sec:invertible}

We define the function $\psi$ and bound it from above. We later use this function and its upper bound to estimate $\gamma_3$ (in \S\ref{sec:invertible_gamma}).

We assume that $\rmL$ is a $d$-subspace of $\reals^D$, where $0 \leq d \leq D$, and that $\nu$ is the uniform measure on $\rmL \cap \mathbb{S}^{D-1}$. We define
\begin{equation}
\label{eq:def_psi_mu1}
\psi_{\nu}(\tt)  =  \nu(\bx\in\rmL: |\bx^T\bv|<\tt),
\end{equation}
where $\bv$ is an arbitrarily fixed vector in $\rmL \cap \mathbb{S}^{D-1}$ (since $\nu$ is uniform
on $\mathbb{S}^{D-1}\cap\rmL$, $\psi_{\nu}$ is independent of $\bv$).
We establish the following upper bound on $\psi_{\nu}$:
\begin{equation}
\psi_{\nu}(\tt)<\sqrt{\frac{\pi\,{d}}{2}}\,\tt.\label{eq:upper_bound}
\end{equation}

Let us denote the surface area measure on $\mathbb{S}^{d-1}$ by $\area_{d-1}$. Using this notation,
we conclude~\eqref{eq:upper_bound} as follows:
\begin{align*}&\psi_{\mu_1}(\tt)=\area_{d-1}\left\{\bx\in\mathbb{S}^{d-1}: |x_1|<\tt \right\}
\Big/\area_{d-1}\left\{\mathbb{S}^{d-1}\right\}\\
&=\frac{\int_{\cos^{-1}(t)}^{\frac{\pi}{2}} \sin^{d-2}(\theta)\di\theta}{\int_{0}^{\frac{\pi}{2}} \sin^{d-2}(\theta)\di\theta}\!\!\leq\!\! \frac{\int_{\cos^{-1}(t)}^{\frac{\pi}{2}} 1 \di\theta}{\int_{0}^{\frac{\pi}{2}} \sin^{d-2}(\theta)\di\theta}
= \frac{\frac{\pi}{2}-\cos^{-1}(t)}{\int_{0}^{\frac{\pi}{2}} \sin^{d-2}(\theta)\di\theta}=\frac{\sin^{-1}(t)}{\frac{\sqrt{\pi}\Gamma(\frac{d-1}{2})}{2\Gamma(\frac{d}{2})}}
\leq \frac{t\sqrt{\pi d}}{\sqrt{2}}.
\end{align*}
We remark that the second equality follows from the well-known formula for the surface area of the spherical cap of ``half-angle'' $\beta$, $Cap(\beta) \subset \mathbb{S}^{D-1}$:
$\area_{d-1} (Cap(\beta))= C(d) \cdot \int_{0}^{\beta} \sin^{d-2}(\theta)\di\theta$ (we do not use the value of $C(d)$ since it cancels in both the numerator and denominator); the fourth (and last) equality follows from a basic trigonometric identity (for the numerator) and the following well-known integration formula (for the denominator): $\int_{0}^{\frac{\pi}{2}} \sin^{d-2}(\theta)\di\theta = B((d-1)/2,1/2)/2=
\sqrt{\pi} \Gamma((d-1)/2)/(2 \Gamma (d/2))$; and the last inequality is obtained by applying the inequality $\sin^{-1}(t)\leq \pi t/2$ for $0 \leq t \leq \pi/2$ (for the numerator) and the following immediate consequence of Gautschi-Kershaw's inequality~\cite{gautschi_kershaw06} for the gamma function:
$\Gamma(\frac{d}{2})/\Gamma(\frac{d-1}{2}) \leq \sqrt{d/2}$ (for the denominator), which is obtained by substituting $s=0.5$, $x=d/2-1$ in (1) of~\cite{gautschi_kershaw06} and using a looser upper bound.


\subsection{Asymptotic Dependence of $\boldsymbol{\gamma_3}$ on $\boldsymbol{D}$ and $\boldsymbol{d}$}
\label{sec:invertible_gamma}

We upper bound the constant $\gamma_3$, which is determined by
\eqref{eq:gamma3},
by applying the function $\psi$ and its upper bound in~\eqref{eq:upper_bound}.
We note that for any $\bv \in \mathbb{S}^{D-1}$ orthogonal to $L^*_1$:
\begin{equation}
\label{eq:gamma3_2}
\{ \bx \in \mathbb{S}^{D-1} : 0<\dist(\bx,L^*_1)<2\gamma_3 \} \subset
\{\bx \in \mathbb{S}^{D-1}:
0<|\bx^T \bv| < 2\gamma_3 \}.
\end{equation}
Therefore, we arbitrarily fix $\bv \in \mathbb{S}^{D-1}  \cap L^{*\perp}_1$ (we will adapt this choice throughout the construction) and estimate a constant $\gamma_3$ that
will satisfy the equation
\begin{equation}
\label{eq:gamma3_3}
(\mu -\alpha_1\mu_1)(\{\bx \in \mathbb{S}^{D-1}: |\bx^T \bv| < 2\gamma_3
\}) \leq \gamma_1/24.
\end{equation}
Indeed, it follows from~\eqref{eq:gamma3_2} and the fact that
$\dist(\bx,L^*_1)=0$ if and only if
$\bx \in \supp(\mu_1)$ that if $\gamma_3$
satisfies~\eqref{eq:gamma3_3} then it also
satisfies~\eqref{eq:gamma3}.

Since $\alpha_0+\sum_{i=2}^K\alpha_i<1$, we only need to find
$\gamma_3$ such that
\begin{equation}
\label{eq:gamma3_4}
\max_{i=0,2,3,\cdots,K}\mu_i(\{\bx \in \mathbb{S}^{D-1}: |\bx^T \bv| <
2\gamma_3 \}) \leq \gamma_1/24.
\end{equation}
Applying \eqref{eq:upper_bound} (with $\rmL= \reals^D$, where $\rmL$ is the subspace defining $\nu$), we obtain that
\begin{equation}
\label{eq:gamma3_5}
\mu_0(\{\bx \in \mathbb{S}^{D-1}: |\bx^T \bv| <
2\gamma_3 \})=\psi_{\mu_0}(2\gamma_3)< \gamma_3\sqrt{2\pi\,D}.
\end{equation}
Fixing $2 \leq i \leq K$ and applying again \eqref{eq:upper_bound}
(with $\rmL= \rmL^*_i$), we obtain that
\begin{align}
\label{eq:gamma3_6}&
\mu_i(\{\bx \in \mathbb{S}^{D-1}: |\bx^T \bv| =
2\gamma_3 \})\leq \mu_i(\{\bx \in \rmL^*_i: |\bx^T
(P_{\rmL^*_i}\bv)| < 2\gamma_3 \})\\=&\nonumber\mu_i(\{\bx \in \rmL^*_i: |\bx^T
(P_{\rmL^*_i}\bv)/\|P_{\rmL^*_i}\bv\|| < 2\gamma_3/\|P_{\rmL^*_i}\bv\|
\})\\\nonumber=&\psi_{\mu_i}(2\gamma_3/\|P_{\rmL^*_i}\bv\|)<
\gamma_3\sqrt{2\pi\,d}/\|P_{\rmL^*_i}\bv\|.
\end{align}
Since the subspaces $\{\rmL_i^*\}_{i=1}^K$ are distinct we may adapt $\bv$ such that $\|P_{\rmL^*_i}\bv\| \neq 0$ for all $2 \leq i \leq K$ (we discuss the optimal choice of $\bv$
below).
Combining~\eqref{eq:gamma3_5} and~\eqref{eq:gamma3_6} and
using the fact that
\eqref{eq:gamma3_4} implies \eqref{eq:gamma3_3} and thus \eqref{eq:gamma3}, we
conclude that
$\gamma_3=\gamma_1\min_{i=2}^K\|P_{\rmL^*_i}\bv\|/(24\sqrt{2\pi D})$
will satisfy \eqref{eq:gamma3_4}.

We can choose $\bv$ to maximize the
$\min_{i=2}^K\|P_{\rmL^*_i}\bv\|$ and therefore
\[\gamma_3=\gamma_1\max_{\bv\in\rmL^{*\perp}_1,\|\bv\|=1}\min_{i=2,\cdots,
K}\|P_{\rmL^*_i}\bv\|/(24\sqrt{2\pi D}).\] We remark that
$\max_{\bv\in\rmL^{*\perp}_1,\|\bv\|=1}\min_{i=2}^K\|P_{\rmL^*_i}\bv\|$
is similar to $\min_{i=2}^K\dG(\rmL_1^*,\rmL_i^*)$ since both of them
measure the difference between $\rmL_1^*$ and $\{\rmL_i^*\}_{i=2}^K$ and in particular,
their value is $0$ only when $\rmL_1^*=\rmL_i^*$ for some $i\geq 2$.

\subsection{Proof of Lemma~\ref{lemma:ell_dist_est}}\label{app:lemma_proof2}

We assume first that $p=2$.  We denote the principal angles between
$\rmL_1$ and $\hat{\rmL}_1$ by $\{\theta_i\}_{i=1}^d$ and the principle
vectors of  $\rmL_1$ and $\hat{\rmL}_1$ by $\{\bv_i\}_{i=1}^d$ and
$\{\hat{\bv}_i\}_{i=1}^d$ respectively. We express every point $\bx$ in $\rmL_1$ by $\bx=(x_1,x_2,\cdots,x_d)=(\bv_1^T\bx,\bv_2^T\bx,\cdots,\bv_d^T\bx)$. We note that
\begin{equation}\label{eq:ell_dist_est1}
\dist(\bx,\hat{\rmL}_1)^2=\sum_{i=1}^d x_i^2\sin^2\theta_i\geq \frac{4}{\pi^2}\sum_{i=1}^d x_i^2\theta_i^2.
\end{equation}
Combining~\eqref{eq:ell_dist_est1} with the observation that $E_{{\mu}_1}(x_i^2)=1/d$ for all $1\leq i\leq d$, we conclude Lemma~\ref{lemma:ell_dist_est} in this case as follows:
\begin{align}
\label{eq:proof_p_2}
\nonumber
&E_{{\mu}_1}\left(e_{\ell_2}(\bx,\hat{\rmL}_1)\right)
=E_{{\mu}_1}  \dist(\bx,\hat{\rmL}_1)^2\geq E_{{\mu}_1} \left(\frac{4}{\pi^2}\sum_{i=1}^d x_i^2\theta_i^2\right) \\
=&\frac{4}{\pi^2\,d}\sum_{i=1}^d \theta_i^2=\frac{ 4}{\pi^2\cdot d} \cdot \dG(\rmL_1,\hat{\rmL}_1)^2.
\end{align}

Next, we assume that $p>2$. Applying~\eqref{eq:proof_p_2} and Jensen's inequality with the convex function $\phi(x)=x^{p/2}$, we conclude Lemma~\ref{lemma:ell_dist_est} in this case as follows:
\[
E_{{\mu}_1}\left(e_{\ell_p}(\bx,\hat{\rmL}_1)\right)\geq
\left(E_{{\mu}_1}\left(e_{\ell_2}(\bx,\hat{\rmL}_1)\right)\right)^{\frac{p}{2}}
\geq \pi^{-p} \cdot 2^{p} \cdot d^{-\frac{p}{2}} \cdot \dG(\rmL_1,\hat{\rmL}_1)^p.
\]

Finally, we assume that $0<p<2$. Using the above parametrization
$\bx=(x_1,x_2,\cdots,x_d)$ for points in $\rmL_1^* \cap \mathbb{S}^{D-1}$, we view
the restriction of $\mu_1$ onto $\rmL_1^*$ (expressed in these coordinates) as the uniform measure onto $\mathbb{S}^{d-1}$.
It follows from~\eqref{eq:ell_dist_est1} that
\begin{equation}\label{eq:ell_dist_est2}
E_{{\mu}_1}\left(e_{\ell_p}(\bx,\hat{\rmL}_1)\right)\geq
E_{{\mu}_1}\left(\frac{4}{\pi^2}\sum_{i=1}^d x_i^2\theta_i^2\right)^{p/2}.
\end{equation}
The main argument of the proof, which we delay to \S\ref{sec:pf_min_of_angles}, is to verify (via Karamata's inequality) that when $\dG(\rmL_1,\hat{\rmL}_1)$ (equivalently, $\sum_{i=1}^d\theta_i^2$) is fixed, then
\begin{equation}
E_{\mu_1} \left( \sum_{i=1}^d x_i^2\theta_i^2 \right)^{p/2}
\geq
E_{{\mu_1}} x_1^p \cdot \left( \sum_{i=1}^d\theta_i^2 \right)^{\frac{p}{2}}
=
E_{{\mu_1}} x_1^p \cdot \dG(\rmL_1,\hat{\rmL}_1)^p.
\label{eq:min_of_angles}
\end{equation}
We estimate $E_{{\mu_1}} x_1^p$ as follows:
\begin{align}
\label{eq:apply_gautschi_again}
\nonumber
&E_{{\mu_1}} x_1^p = \frac{\int
(\sin\theta)^{d-2}(\cos\theta)^p\di\theta}{\int
(\sin\theta)^{d-2}\di\theta}\\=&\frac{B(\frac{d-1}{2},\frac{p+1}{2})}{B(\frac{d-1}{2},\frac{1}{2})}
=\frac{\Gamma(\frac{p+1}{2})\,\Gamma(\frac{d}{2})}{\Gamma(\frac{1}{2})\,\Gamma(\frac{d+p}{2})}
> \frac{0.88}{\sqrt{\pi}} \cdot \Big(\frac{d+p}{2}\Big)^{-p/2} .
\end{align}
The last inequality uses the following equalities and inequalities:
$\Gamma(1/2)=\sqrt{\pi}$; $\Gamma(\frac{p+1}{2}) \geq 0.88$, which follows from the well-known estimate: $\min_{x\geq 0}\Gamma(x) \approx 0.885603$ (see e.g., \cite{Edwards1935}); and the inequality $\Gamma(\frac{d+p}{2})/\Gamma(d/2)<(\frac{d+p}{2})^{p/2}$, which follows from Gautschi-Kershaw's inequality~\cite{gautschi_kershaw06} (indeed, apply (1) of~\cite{gautschi_kershaw06}
with $x=(d+p-2)/2$ and $s=1-p/2$, while using a looser upper bound, and then invert the inequality while taking the power of -1 of both its LHS and RHS).

Therefore, the case $0<p<2$ is concluded by combining~\eqref{eq:ell_dist_est2},
\eqref{eq:min_of_angles} (which is proved in the following subsection) and
\eqref{eq:apply_gautschi_again}.

\subsubsection{Proof of \eqref{eq:min_of_angles}}
\label{sec:pf_min_of_angles}
We will prove a more general statement, which requires the
following notation:
For $1 \leq j \leq d$ and $1 \leq i \leq d$
$$
\theta_{i,j} =
\begin{cases}
\sqrt{\sum_{i=1}^j \theta_{i,j}^2}, &\text{if $i=1$};\\
0, &\text{if $2 \leq i \leq j$};\\
\theta_i, &\text{if $j+1 \leq i \leq d$}.
\end{cases}
$$
The more general statement is
\begin{equation}
\label{eq:min_of_angles_gen}
E_{\mu_1}\left(\sum_{i=1}^d x_i^2\theta_{i,j}^2\right)^{p/2}
\geq E_{\mu_1}\left(\sum_{i=1}^d x_i^2\theta_{i,j+1}^2\right)^{p/2} \ \
\text{ for } 1 \leq j \leq d-1.
\end{equation}
Clearly, successive application of~\eqref{eq:min_of_angles_gen} implies \eqref{eq:min_of_angles}.

In order to prove~\eqref{eq:min_of_angles_gen}, we introduce additional notation, formulate two sequences with the majorization property and then apply Karamata's inequality.
For $1 \leq j \leq d-1$, let
$$
x_{i,j} =
\begin{cases}
x_i, &\text{if $i \neq 1, j+1$};\\
x_{j+1}, &\text{if $i = 1$};\\
x_{1},   &\text{if $i = j+1$}.
\end{cases}
$$
It follows from elementary algebraic manipulations that for any $1 \leq j \leq d$:
\begin{equation}
\label{eq:alg_manipulation_1}
\sum_{i=1}^d x_i^2\theta_{i,j+1}^2 + \sum_{i=1}^d {x}_{i,j}^2\theta_{i,j+1}^2= \sum_{i=1}^d x_i^2 {\theta_{i,j}}^2 + \sum_{i=1}^d {x}_{i,j}^2 {\theta_{i,j}}^2.
\end{equation}
One can also verify that
\begin{equation}
\label{eq:alg_manipulation_2}
\max\big(\sum_{i=1}^d x_i^2{\theta_{i,j+1}}^2,\sum_{i=1}^d {x}_{i,j}^2
\theta_{i,j+1}^2\big)
\geq
\max\big(\sum_{i=1}^d x_i^2{\theta_{i,j}}^2,\sum_{i=1}^d {x}_{i,j}^2 {\theta_{i,j}}^2\big).
\end{equation}
This is done by showing (again by simple algebra) that each one of the terms in the argument of the maximum function in the LHS of~\eqref{eq:alg_manipulation_2} is controlled by one of the terms in the RHS of~\eqref{eq:alg_manipulation_2}. Equations~\eqref{eq:alg_manipulation_1} and~\eqref{eq:alg_manipulation_2}
imply that
$(\sum_{i=1}^d x_i^2{\theta_{i,j+1}}^2,\sum_{i=1}^d {x}_{i,j}^2 \theta_{i,j+1}^2)$ majorizes $(\sum_{i=1}^d x_i^2{\theta_{i,j}}^2,\sum_{i=1}^d {x}_{i,j}^2 {\theta_{i,j}}^2)$. Combining this observation, the concavity of $f(x)=x^{p/2}$ and Karamata's inequality, we conclude
that
\begin{equation}
\label{eq:via_karamata}
\left(\sum_{i=1}^d x_i^2\theta_{i,j}^2\right)^{p/2}
+\left(\sum_{i=1}^d x_{i,j}^2\theta_{i,j}^2\right)^{p/2}
\geq \left(\sum_{i=1}^d
x_i^2\theta_{i,j+1}^2\right)^{p/2} +
\left(\sum_{i=1}^d
x_{i,j}^2\theta_{i,j+1}^2\right)^{p/2}.
\end{equation}
Integrating~\eqref{eq:via_karamata} over $\mu_1$ and using the invariance of $\mu_1$ to permutations
of $\bx$ (in particular, invariance to replacing $x_i$ with $x_{i,j}$ for all $1\leq i \leq d$), we
obtain~\eqref{eq:min_of_angles_gen} and consequently~\eqref{eq:min_of_angles}.

\subsection{Proof of Lemma~\ref{lemma:distance}}\label{app:lemma_distance_proof}
We denote the principal angles between the $d$-subspaces $\rmL_1$ and
$\rmL_2$ by
$\theta_1\geq\theta_2\geq\theta_3\geq\cdots\geq\theta_d$.
Arbitrarily choosing $\bQ_1$, $\bQ_2 \in \Or(D,d)$, representing
$\rmL_1$, $\rmL_2$ respectively, we note that
\begin{align*}
&|\dist(\bx,\rmL_1)-\dist(\bx,\rmL_2)|=|\,\|\bx-\bx \bQ_1\bQ_1^T\|-\|\bx-\bx \bQ_2\bQ_2^T\|\,|\\\leq &\|\bx-\bx \bQ_1\bQ_1^T-\bx+\bx \bQ_2\bQ_2^T\|
\leq\|\bx\|\, \normF{\bQ_1\bQ_1^T-\bQ_2\bQ_2^T}\\=&
\|\bx\|\,\sqrt{\sum_{i=1}^{d}\sin(\theta_i)^2}\leq\|\bx\|\,\sqrt{\sum_{i=1}^{d}\theta_i^2}
=\|\bx\|\,\dG(\rmL_1,\rmL_2).
\end{align*}

\subsection{Local $\boldsymbol{\ell_p}$ subspace for $\boldsymbol{0< p<1}$ and $\boldsymbol{K=1}$}\label{sec:prop:p_geq_1}
\begin{proposition}
\label{prop:p_geq_1} Assume that $D>d+1$, $\rmL^*_1\in \GDd$, $\mu_0$ is a
distribution on $\reals^D$ such that $\mu_0(\{\rmL\})\neq 0$ for any affine subspace $\rmL$, where $\rmL \subset \reals^D$,
$\mu_1$ a distribution on $\rmL_1^*$ and $\mu=\alpha_0 \mu_0+\alpha_1 \mu_1$, where $\alpha_0$, $\alpha_1$ are
nonnegative numbers summing to $1$. If $\rmX$ is a data set
sampled identically and independently from $\mu$ and $p>1$, then the probability
that $\rmL^*_1$ is a local $\ell_p$ subspace of $\rmX$ is 0.
\end{proposition}

Let $\{\by_i\}_{i=1}^{N_0}$ denote the i.i.d.~outliers sampled from $\mu_0$.
We will prove that for any $\bW\in \reals^{d \times D-d}$:
\begin{equation}
\label{eq:bV}
\mu_0\left(\by_1 \in \reals^D :
P_{\rmL^*_1}(\by_1) P^{\perp}_{\rmL^*_1}(\by_1)^T\dist(\by_1,\rmL^*_1)^{p-2}=
\bW\right)=0.
\end{equation}
Proposition~\ref{prop:p_geq_1} follows by substituting
$\bW = -\sum_{i=2}^{N_0}P_{\rmL^*_1}(\by_i)P^{\perp}_{\rmL^*_1}(\by_i)^T\dist(\by_i,\rmL^*_1)^{p-2}$
in \eqref{eq:bV} and applying Proposition~\ref{prop:comb_cond3}.

We may assume that $\by_1 \nin \rmL^*_1 \cup {\rmL^*_1}^{\perp}$ since $\mu_0(\{\rmL_1^*\})=\mu_0(\{\rmL_1^{*\perp}\})=0$.
We note that for any $\by_1 \nin \rmL^*_1 \cup {\rmL^*_1}^{\perp}$ the rank of $P_{\rmL^*_1}(\by_1)P^{\perp}_{\rmL^*_1}(\by_1)^T$ is 1.
Therefore, \eqref{eq:bV} is obvious if $\mathrm{rank}(\bW) \neq 1$.
Furthermore, if $\ker(\bW)\not\supset\rmL_1^{*}$ then \eqref{eq:bV} is also obvious since the kernel of
$P_{\rmL^*_1}(\by_1)P^{\perp}_{\rmL^*_1}(\by_1)^T$ contains $\rmL_1^{*}$.

At last, we assume that $\mathrm{rank}(\bW)=1$ and
$\ker(\bW)\supset\rmL_1^{*}$ and denote $\bv=\ker(\bW)^\perp$. Applying the assumption that proper affine subspaces of $\reals^D$ have measure $\mu_0$ zero and the assumption $D>d+1$,
we obtain that $\mu_0(\Sp(\rmL_1^*,\bv))= 0$. We thus
conclude~\eqref{eq:bV} (and consequently Proposition~\ref{prop:p_geq_1}) as follows.
\begin{align*}
& \mu_0\left(\by_1 \in \reals^D:
P_{\rmL^*_1}(\by_1)P^{\perp}_{\rmL^*_1}(\by_1)^T\dist(\by_1,\rmL^*_1)^{p-2}=
\bW\right)
 \\   \leq
&  \mu_0\left(\by_1 \in \reals^D : P_{\rmL^*_1}(\by_1)= c\bv\,\,\text{for some
$c\in\reals$} \right) \\
= &
\mu_0\left(\by_1 \in \reals^D : \by_1\in\Sp(\rmL_1^*,\bv)
\right)= 0.
 \end{align*}

\subsection{Proof of Lemma~\ref{lemma:mean}} \label{app:lemma_proof}
We assume WLOG that $i=1$ in \eqref{eq:lemma2}. We thus need to
prove that for all $\hat{\rmL} \in \GDd$:
\begin{align}
&\mathbb{E}_{\mu_1}(\dist(\bx_1,\hat{\rmL})^p)+\mathbb{E}_{\mu_2}(\dist(\bx_2,\hat{\rmL})^p)
\nonumber\\\label{eq:contradiction}\geq&
\mathbb{E}_{\mu_1}(\dist(\bx_1,\rmL_1)^p)+\mathbb{E}_{\mu_2}(\dist(\bx_2,\rmL_1)^p).\end{align} We denote the
principal angles between $\rmL_1$ and $\rmL_2$ by
$\{\theta_i\}_{i=1}^d$, the principle vectors of  $\rmL_1$ and
$\rmL_2$ by $\{\bv_i\}_{i=1}^d$ and $\{\hat{\bv}_i\}_{i=1}^d$ and
the complementary orthogonal system for $\rmL_2$ w.r.t.~$\rmL_1$ by
$\{\mathbf{u}_i\}_{i=1}^{d}$.

We notice that we can restrict the set of subspaces $\hat{\rmL}$
satisfying~\eqref{eq:contradiction}. First of all, we only need to
consider subspaces
\begin{equation}\label{eq:subs_cond1} \hat{\rmL}\in \rmL_1+\rmL_2 \,.\end{equation}Indeed,
the LHS of \eqref{eq:contradiction} is the same if we replace
$\hat{\rmL}$ by $\hat{\rmL}\cap (\rmL_1+\rmL_2)$.

Second of all, we claim that it is sufficient to assume that
\begin{equation}\Sp(\hat{\mathbf{v}}_i,\mathbf{v}_i) \nsubseteq
\hat{\rmL} \,\,\text{ for all $1 \leq i\leq
k$}.\label{eq:subs_cond2}\end{equation}
We first show this for $i=1$. We suppose on the contrary
to~\eqref{eq:subs_cond2} that
$\hat{\mathbf{v}}_1,\mathbf{v}_1\in\hat{\rmL}$. Since $\hat{\rmL}$
is $d$-dimensional, there exists $2\leq j\leq d$ (assume WLOG $j=2$)
such that it
does not contain both $\hat{\bv}_j$ and $\bv_j$. For any pair of
points $\bx=\sum_{i=1}^{d}a_i\mathbf{v}_i \in \rmL_1$ and
$\hat{\bx}=\sum_{i=1}^{d}a_i\hat{\mathbf{v}}_i \in \rmL_2$:
\[\dist(\bx,\hat{\rmL})=\sqrt{\sin(\theta_2)^2a_2^2+\tau_1^2}
\text{\,\,\,and\,\,\,}
\dist(\hat{\bx},\hat{\rmL})=\sqrt{\sin(\theta_1)^2a_1^2+\tau_2^2},\]
where
\[
\tau_1=\dist\left(\sum_{i=3}^{d}a_i\mathbf{v}_i,\hat{\rmL}\right)\text{\,\,\,and\,\,\,}
\tau_2=\dist\left(\sum_{i=3}^{d}a_i\hat{\mathbf{v}}_i,\hat{\rmL}\right).
\]
Now, for
$\tilde{\rmL}=\Sp(\hat{\rmL}\setminus\{\mathbf{v}_1,\hat{\mathbf{v}}_1\},\mathbf{v}_1,\mathbf{v}_2)$,
  we obtain that
\[\dist(\hat{\bx},\tilde{\rmL})=\sqrt{\sin(\theta_1)^2a_1^2+\sin(\theta_2)^2a_2^2+\tau_2^2}\text{\,\,\,and\,\,\,} \dist(\bx,\tilde{\rmL})=
\tau_1 .\] Therefore
\[
\dist(\bx,\tilde{\rmL})^p+\dist(\hat{\bx},\tilde{\rmL})^p\leq
\dist(\bx,\hat{\rmL})^p+\dist(\hat{\bx},\hat{\rmL})^p
\]
and by direct integration we have that
\begin{align}&\mathbb{E}_{\mu_1}(\dist(\bx_1,\tilde{\rmL})^p)+\mathbb{E}_{\mu_2}(\dist(\bx_2,\tilde{\rmL})^p)\nonumber\\\leq&
\mathbb{E}_{\mu_1}(\dist(\bx_1,\hat{\rmL})^p)+
\mathbb{E}_{\mu_2}(\dist(\bx_2,\hat{\rmL})^p).\label{eq:contradiction1}\end{align}

Since $\tilde{\rmL}$ satisfies~\eqref{eq:subs_cond2} for $i=1$ and
satisfies \eqref{eq:contradiction1}, we conclude that proving \eqref{eq:contradiction}
only for $\hat{\rmL}$
satisfying~\eqref{eq:subs_cond2} with $i=1$ implies it for all $\hat{\rmL} \in \GDd$.
Similarly, we can assume
that $\hat{\rmL}$ satisfies~\eqref{eq:subs_cond2} for all $1\leq i\leq
k$, by verifying \eqref{eq:contradiction1} for
$\tilde{\rmL}=\Sp(\hat{\rmL}\setminus\{\mathbf{v}_i,\hat{\mathbf{v}}_i\},\mathbf{v}_i,\mathbf{v}_j)$
for some $1 \leq j\neq i \leq k$ such that $\Sp(\hat{\mathbf{v}}_j,\mathbf{v}_j) \nsubseteq
\hat{\rmL}$.

It follows from~\eqref{eq:subs_cond1} and~\eqref{eq:subs_cond2} that
$\hat{\rmL}$ can be represented as follows:
\[\hat{\rmL}=\Sp(\mathbf{v}_1^*,\mathbf{v}_2^*,\cdots,\mathbf{v}_d^*),\]
where \[\mathbf{v}_i^*=\cos{\theta _i^*}\mathbf{v}_i+\sin{\theta_i
^*}\mathbf{u}_i.\]
Thus, for any pair of points
$\bx=\sum_{i=1}^{d}a_i\mathbf{v}_i \in \rmL_1$ and
$\hat{\bx}=\sum_{i=1}^{d}a_i\hat{\mathbf{v}}_i \in \rmL_2$:
\begin{equation}\dist(\bx,\hat{\rmL})=\sqrt{\sum_{i=1}^{d}\sin^2\theta_i
^*a_i^2}, \ \ 
\dist(\hat{\bx},\hat{\rmL})=\sqrt{\sum_{i=1}^{d}\sin^2(\theta_i-\theta_i^*)a_i^2},
\label{eq:nikos1} \end{equation}%
\begin{equation}\dist(\bx,\rmL_1)=0\text{\,\,\,and\,\,\,}
\dist(\hat{\bx},\rmL_1)=\sqrt{\sum_{i=1}^{d}\sin^2{\theta_i}a_i^2}.
\label{eq:nikos2} \end{equation}%
Applying~\eqref{eq:nikos1}, \eqref{eq:nikos2}, the triangle
inequality (for ``sine vectors'' in $\reals^d$) and then the
subadditivity of the sine function, we obtain that
\begin{align*}
\dist(\bx,\hat{\rmL})+\dist(\hat{\bx},\hat{\rmL})&\geq \sqrt{\sum_{i=1}^{d}\big(\sin{\theta_i ^* } +\sin{(\theta_i-\theta_i ^*)}\big)^2a_i^2}\\
\geq
\sqrt{\sum_{i=1}^{d}\sin^2{\theta_i}a_i^2}&=\dist(\hat{\bx},\rmL_1)
+ \dist({\bx},\rmL_1).
\end{align*}
Since $p\leq 1$, this inequality clearly implies that
\begin{equation}\label{eq:nikos3}
\dist(\bx,\hat{\rmL})^p+\dist(\hat{\bx},\hat{\rmL})^p\geq
\dist(\hat{\bx},\rmL_1)^p=\dist(\hat{\bx},\rmL_1)^p+\dist(\bx,\rmL_1)^p.
\end{equation}%
We conclude~\eqref{eq:contradiction} by appropriately integrating~\eqref{eq:nikos3} and consequently prove the lemma. 

\subsection{Proof of~\eqref{eq:verify_if}} \label{app:verify_if}
We denote
$\mathbf{B}=\sum_{i=1}^{N_1}P_{\rmL^*_1}(\bx_i)P_{\rmL^*_1}(\bx_i)^{T}$ and note that if $\max_{1 \leq j \leq d} \sigma_j \left(\bB -\delta_{*}
\mathbf{I}_{d}\right) < \eta$, then
$$\frac{\|\bB \bv-\delta_{*}\bv\|}{\|\bv\|} < \eta \ \text{ for all }
\bv \in \reals^{d}\setminus \{\mathbf{0}\},$$
and consequently
$$
\delta_{*} - \eta <\frac{\|\bB \bv \|}{\|\bv\|} \ \text{ for all } \bv
\in \reals^{d}\setminus \{\mathbf{0}\},$$
that is, $\min_{1 \leq j \leq d} \sigma_j(\mathbf{B}) > \delta_{*} - \eta$. 

\subsection{Proof of~\eqref{eq:gamma1}}
\label{sec:gamma1}
We first prove the following two lemmata.
\begin{lemma}\label{lemma:pertubation_lp}
For $p>1$ and any $\bx,\by\in\ball(\b0,1)$,
\[\|\|\bx\|^{p-2}\bx-\|\by\|^{p-2}\by\|\leq
\begin{cases} 2^{3-p}\|\bx-\by\|^{p-1}, &\text{if $1<p\leq 2$};\\
(p-1)\|\bx-\by\|, &\text{if $p>2$}.
\end{cases}
\]
\end{lemma}
\begin{proof}
First we consider the case where either $\|\bx\|=1$ or $\|\by\|=1$. WLOG we assume that $\|\bx\|=1$. When $p> 2$
\begin{align*}
&\|\|\bx\|^{p-2}\bx-\|\by\|^{p-2}\by\|=\|\bx-\|\by\|^{p-2}\by\|\leq \|\bx-\by\|\frac{\|\bx-\by\|+\|\by-\|\by\|^{p-2}\by\|}{\|\bx-\by\|}
\\\leq& \|\bx-\by\|\frac{1-\|\by\|^{p-1}}{1-\|\by\|}\leq (p-1) \|\bx-\by\|,
\end{align*}
where the second inequality follows from the identity
$1-\|\by\|+\|\by-\|\by\|^{p-2}\by\|=1-\|\by\|^{p-1}$, the inequality $\|\bx-\by\|\geq 1-\|\by\|$ and the fact that the function $f(t)=(t+c)/t$ is non-increasing for $c \geq 0$.

On the other hand, when $1 < p \leq 2$
\begin{equation}
\label{eq:last_ineq}
\|\|\bx\|^{p-2}\bx-\|\by\|^{p-2}\by\|=\|\bx-\|\by\|^{p-2}\by\|\leq \|\bx-\by\|+\|\by-\|\by\|^{p-2}\by\| \leq 2 \|\bx-\by\|,
\end{equation}
where the last inequality of~\eqref{eq:last_ineq} follows from the inequality
\begin{equation}
\label{eq:last_ineq2}
\|\by-\|\by\|^{p-2}\by\|\leq \|\frac{\by}{\|\by\|}-\by\| \leq  \|\bx-\by\|,
\end{equation}
which we explain as follows.
Since $\by$, $\|\by\|^{p-2}\by$  and  ${\by}/{\|\by\|}$ lie on the same line
through the origin and  since $\|\by\| \leq \|\|\by\|^{p-2}\by\| \leq 1$, $\|\by\|^{p-2}\by$ is located
between $\by$ and ${\by}/{\|\by\|}$ and this clarifies the first inequality in~\eqref{eq:last_ineq2}.
The second inequality in~\eqref{eq:last_ineq2} follows from the following observation:
$\|{\by}/{\|\by\|}-\by\| =1 - \|\by\| = \|\bx\| - \|\by\| \leq \|\bx-\by\|$.

The main idea of the proof for the general case is to arbitrarily fix $\|\bx-\by\|$ and maximize  $\|\|\bx\|^{p-2}\bx-\|\by\|^{p-2}\by\|$
We transform the problem into maximization over the two variables: $r=\log({\|\bx\|}/{\|\by\|})$ and $t=2\bx^T\by/(\|\bx\|\|\by\|)$ of the function
\[
h(r,t):=\frac{e^{(p-1)r}+e^{-(p-1)r}+t}{(e^{r}+e^{-r}+t)^{p-1}}=\left(\frac{\|\|\bx\|^{p-2}\bx-\|\by\|^{p-2}\by\|}{\|\bx-\by\|^{p-1}}\right)^2,
\]
when $\|\bx-\by\|>0$ is fixed (if $\|\bx-\by\|=0$ then~\eqref{lemma:pertubation_lp} is trivial).

We first find the boundary of the domain of this function when $c_0:=\|\bx-\by\|$ is fixed. We then maximize the function on the boundary and later find a local maximizer within the interior of this domain. The variable $t$ obtains values in $[-2,2]$. For any fixed $t$, we find the values that $r$ may obtain.
We note that $\|\bx\|^2+\|\by\|^2-t\|\bx\|\|\by\|=c_0^2$ and
$e^{2r}+1-t e^r=c_0^2/\|\by\|^2$. Since $\|\by\|\leq 1$,
if $t$ is fixed and $r\leq 0$, then $r$ is in the domain $e^{2r}+1-te^r\geq c_0^2$,
whose boundary is $e^{2r}+1-te^r =
c_0^2$. That is, when $r \leq 0$ (i.e., $\|\bx\| \leq \|\by\|$), then $\|\by\|=1$.
Similarly, when $r\geq 0$ the boundary of the domain of $h(r,t)$ corresponds to the case $\|\bx\|=1$.

Next, we verify \eqref{eq:gamma1} for points on the boundary of the domain of $h(r,t)$ (it is sufficient to verify it for maximizers on this boundary). For fixed $-2<t<2$, points on the boundary correspond to $\|\bx\|=1$ or $\|\by\|=1$ and we have already verified \eqref{eq:gamma1} in this case. We also need to consider the boundary points $t=-2$ or $t=2$, equivalently, $\bx/\|\bx\|=-\by/\|\by\|$ or $\bx/\|\bx\|=\by/\|\by\|$. We thus find the maximal values of $h(r,2)$ and $h(r,-2)$ (when its denominator is fixed).
The function $\sqrt{h(r,-2)}$ (i.e., with $\bx$ and $\by$ satisfying $\bx/\|\bx\|=-\by/\|\by\|$) is equivalent to
\[
\frac{a^{p-1}+b^{p-1}}{(a+b)^{p-1}},\,\,\,\text{where $a=\|\bx\|$ and $b=\|\by\|$.}
\]
Its maximum is obtained when $a=b$ if $1< p\leq 2$ and when $a=0$ or $b=0$ if $p>2$.
The function $\sqrt{h(r,2)}$ (i.e., with $\bx$ and $\by$ satisfying $\bx/\|\bx\|=\by/\|\by\|$) is equivalent to
\[
\frac{a^{p-1}-b^{p-1}}{(a-b)^{p-1}}.
\]
Using the convexity/concavity of the power function $x^{p-1}$ for different values of $p$ we note that if $p\geq 2$ then its maximum is obtained when $a=1$ or $b=1$ and if $1<p\leq 2$ then its maximum is obtained when $b=0$.
It is immediate to note that~\eqref{eq:gamma1} is satisfied when $a=0$ (i.e., $\bx=0$) or $b=0$ (i.e., $\by=0$). We have also verified above that it is satisfied when $a=1$ or $b=1$.
We also show that~\eqref{eq:gamma1} is satisfied when $a=b$ and $1<p\leq 2$. Indeed, $\|\bx-\by\|\leq
\|\bx\|+\|\by\|=2\|\bx\|$ and thus $\|\bx-\by\|^{p-2}\geq
(2\|\bx\|)^{p-2}$, which implies that
\begin{align*}
&2^{3-p}\|\bx-\by\|^{p-1}=2^{3-p}\|\bx-\by\|^{p-2}\,\|\bx-\by\|
\geq 2^{3-p} (2\|\bx\|)^{p-2}\,\|\bx-\by\|\\=&
2 \|\bx\|^{p-2}\,\|\bx-\by\|
=2 \|\|\bx\|^{p-2}\bx-\|\by\|^{p-2}\by\|
\geq \|\|\bx\|^{p-2}\bx-\|\by\|^{p-2}\by\|.
\end{align*}
We therefore verified~\eqref{eq:gamma1} for points corresponding to the boundary of $h$.

At last, we consider the interior of the domain of $h$. If $(r_0,t_0)$ is a local maximizer of $h(r,t)$, then
\[
0=\frac{\di}{\di t}h(r,t)\Big|_{(r,t)=(r_0,t_0)}= \frac{(e^{r_0}+e^{-r_0}+t_0)-(p-1)(e^{(p-1)r_0}+e^{-(p-1)r_0}+t_0)}{(e^{r_0}+e^{-r_0}+t_0)^{p}}
\]
and  $(e^{r_0}+e^{-r_0}+t_0)=(p-1)(e^{(p-1)r_0}+e^{-(p-1)r_0}+t_0)$.  Therefore
\[
h(r_0,t_0)=\frac{p-1}{(e^{r_0}+e^{-r_0}+t_0)^{p-2}}.
\]
Furthermore, its maximal value (when $t_0$ is fixed) is obtained when $r_0=0$ or $r_0 =\infty$ or $r_0=-\infty$. Equivalently, it is obtained when $a=b$ or $a=0$ or $b=0$. To conclude the proof we only need to verify that~\eqref{eq:gamma1} is satisfied when $a=b$ and any $1<p\leq 2$ (all the other cases were discussed above).
In this case, we use the fact that $a\leq 1$ and $a^{p-2}\leq 1$ and consequently note that
\[\|\|\bx\|^{p-2}\bx-\|\by\|^{p-2}\by\|=a^{p-2}\|\bx-\by\|\leq
\|\bx-\by\|\leq (p-1)\|\bx-\by\|.\]

\end{proof}
\begin{lemma}\label{lemma:deriative}
If $f, g:\reals\rightarrow \reals$, $g(0)=0$, $g$ is increasing and
\begin{equation}
\text{$|f'(x_1)-f'(x_2)|\leq g(|x_1-x_2|)$ \ for any $x_1,x_2\in\reals$,}\label{eq:derivative_diff}
\end{equation}
then the following inequality is satisfied for all $x_0 \in \reals$ and for $\hat{x}:=\min_{x\in\reals}f(x)$:
\[
f(x_0)-f(\hat{x})\geq |f'(x_0)|g^{-1}(|f'(x_0)|) - \int_{0}^{g^{-1}(|f'(x_0)|)}g(x)\di x.
\]
\end{lemma}
\begin{proof}
WLOG we assume that $f'(x_0) \geq 0$. Applying this assumption, \eqref{eq:derivative_diff} and the definition of $\hat{x}$,
we conclude the lemma as follows:
\begin{align*}
&f(x_0)-f(\hat{x})\geq f(x_0)-f(x_0-g^{-1}(f'(x_0)))=\int_{x_0-g^{-1}(f'(x_0))}^{x_0}f'(x)\di x \geq
\\& \int\limits_{x_0-g^{-1}(f'(x_0))}\limits^{x_0}\left(f'(x_0)-g(x_0-x)\right)\di x= f'(x_0)g^{-1}(f'(x_0)) - \int\limits_{0}\limits^{g^{-1}(f'(x_0))}g(x)\di x.
\end{align*}
\end{proof}

To prove \eqref{eq:gamma1}, we restrict   $E_{\mu}(e_{\ell_p}(\bx,\rmL))$ to a geodesic line $\rmL: [0,\infty)\rightarrow \GDd$ with $\rmL(0)=\rmL_1^*$. Then we use the following inequality to find the lower bound of $\zeta_1$:
\begin{align}
\zeta_1 =  &E_{\mu}(e_{\ell_p}(\bx,
\rmL_1^*))-\min_{\rmL\in\GDd}E_{\mu}(e_{\ell_p}(\bx,
\rmL))\nonumber\\\geq& E_{\mu}(e_{\ell_p}(\bx,
\rmL(0)))-\min_{t\geq 0}E_{\mu}(e_{\ell_p}(\bx,
\rmL(t))).\label{eq:gamma_1_lower_bound}
\end{align}
The lower bound of the RHS of \eqref{eq:gamma_1_lower_bound} will  be obtained by applying Lemma~\ref{lemma:deriative} to $f(t)=E_{\mu}(e_{\ell_p}(\bx,
\rmL(t)))$ with a specific $L(t)$.

We choose this $\rmL(t)$ such that $\dist_G(\rmL(0),\rmL(1))=1$ and
\begin{equation}\label{eq:choose_direction}
\frac{\di}{\di t}E_{\mu}(e_{\ell_p}(\bx,
\rmL(t)))\Big|_{t=0}= -p\|E_{\mu}(\mathbf{D}_{\rmL^*_1,\bx,p})\|_F.
\end{equation}
To show that this is possible, we recall (see~\eqref{eq:main2_a1}) that
\begin{equation}\label{eq:choose_direction2}
\frac{\di}{\di t}E_{\mu}(e_{\ell_p}(\bx,
\rmL(t)))\Big|_{t=0}=
-p\tr(\bC\bV  E_{\mu}(\mathbf{D}_{\rmL^*_1,\bx,p})\bU^T),
\end{equation}
where $\|\bC\|_F=1$ (since we use the distance defined in \eqref{eq:dist_grassman}).
Let us denote the thin SVD of $E_{\mu}(\mathbf{D}_{\rmL^*_1,\bx,p})$ by $\bV_0\bSigma_0\bU_0^T$.
We choose the matrices $\bV$, $\bU$ and $\bC$, which determine $\rmL(t)$ as follows: $\bV=\bV_0^T$, $\bU=\bU_0^T$ and
$\bC=\bSigma_0/ \|\bSigma_0\|_F.$ This choice indeed implies~\eqref{eq:choose_direction}
as a consequence of~\eqref{eq:choose_direction2} and the following observation:
\begin{align*}
p\tr(\bC\bV E_{\mu}(\mathbf{D}_{\rmL^*_1,\bx,p})\bU^T)=
p\tr(\bSigma_0^2)/\|\bSigma_0\|_F
=p\|\bSigma_0\|_F
=p\| E_{\mu}(\mathbf{D}_{\rmL^*_1,\bx,p})\|_F.
\end{align*}

We proceed by finding $g$ for $f(t)=E_{\mu}(e_{\ell_p}(\bx,
\rmL(t)))$  so that \eqref{eq:derivative_diff} is satisfied.
It follows from \eqref{eq:choose_direction2}
that for $t_2>t_1\geq 0$:
\begin{align}\label{eq:derivative_diff3}
&\left|f'(t_2)-f'(t_1)\right|
\leq
p \, E_{\mu}
\left\langle\bC,\bV(\mathbf{D}_{\rmL(t_1),\bx,p}-\mathbf{D}_{\rmL(t_2),\bx,p})\bU^T\right\rangle_F
\\\nonumber\leq &
p \, E_{\mu}
 \|\mathbf{D}_{\rmL(t_1),\bx,p}-\mathbf{D}_{\rmL(t_2),\bx,p}\|_F.\end{align}

Combining the following observations
\[
\|P_{\rmL(t_1)}(\bx)-P_{\rmL(t_2)}(\bx)\|\leq \|P_{\rmL(t_1)}-P_{\rmL(t_2)}\|\leq \dG(\rmL(t_1),\rmL(t_2))= t_2-t_1,
\]
\[\text{$\|P_{\rmL(t_1)^\perp}
(\bx) \dist(\bx,\rmL(t_1))^{(p-2)}\|\leq 1$ and $\|P_{\rmL(t_2)}(\bx)\|\leq 1$,}\]
with the following consequence of Lemma~\ref{lemma:pertubation_lp}
\begin{align*}
&\|P_{\rmL(t_1)^\perp}
(\bx) \dist(\bx,\rmL(t_1))^{(p-2)}-P_{\rmL(t_2)^\perp}
(\bx) \dist(\bx,\rmL(t_2))^{(p-2)}\|\\
\leq &\begin{cases} 2^{3-p}\|P_{\rmL(t_1)^\perp}
(\bx)-P_{\rmL(t_2)^\perp}
(\bx)\|^{p-1},&\text{if $1<p\leq 2$};\\
(p-1) \|P_{\rmL(t_1)^\perp}
(\bx)-P_{\rmL(t_2)^\perp}
(\bx)\|,&\text{if $p\geq 2$},
\end{cases}
\end{align*}
we obtain that
\begin{align}
\nonumber&\|\mathbf{D}_{\rmL(t_1),\bx,p}-\mathbf{D}_{\rmL(t_2),\bx,p}\|_F =
\|P_{\rmL(t_1)}(\bx)P_{\rmL(t_1)^\perp}
(\bx)^T \dist(\bx,\rmL(t_1))^{(p-2)}
\\\nonumber&\,
-P_{\rmL(t_2)}(\bx)P_{\rmL(t_2)^\perp}
(\bx)^T \dist(\bx,\rmL(t_2))^{(p-2)}\|_F
\\\nonumber\leq&
\|P_{\rmL(t_1)^\perp}
(\bx) \dist(\bx,\rmL(t_1))^{(p-2)}\|\|P_{\rmL(t_1)}(\bx)-P_{\rmL(t_2)}(\bx)\|
\\\nonumber&+\|P_{\rmL(t_2)}(\bx)\|\|P_{\rmL(t_1)^\perp}
(\bx) \dist(\bx,\rmL(t_1))^{(p-2)}-P_{\rmL(t_2)^\perp}
(\bx) \dist(\bx,\rmL(t_2))^{(p-2)}\|
\\\nonumber\leq&
\|P_{\rmL(t_1)}(\bx)-P_{\rmL(t_2)}(\bx)\|\\\nonumber&+
\|P_{\rmL(t_1)^\perp}
(\bx) \dist(\bx,\rmL(t_1))^{(p-2)}-P_{\rmL(t_2)^\perp}
(\bx) \dist(\bx,\rmL(t_2))^{(p-2)}\|
\\\nonumber\leq&
\begin{cases} (t_2-t_1)+(p-1)\,(t_2-t_1),&\text{if $p\geq 2$};\\
(t_2-t_1)+2^{3-p}\,(t_2-t_1)^{p-1},&\text{if $1<p<2$}
\end{cases}
\\\leq&
\begin{cases} p\,(t_2-t_1),&\text{if $p\geq 2$};\\
2^{4-p}\,\max((t_2-t_1)^{p-1},t_2-t_1),&\text{if $1<p<2$.}
\end{cases}
\label{eq:derivative_diff4}
\end{align}

In view of \eqref{eq:derivative_diff}, \eqref{eq:derivative_diff3}, \eqref{eq:derivative_diff4} and our choice of $f$, we define:
\begin{equation}\label{eq:derivative_diff5}
g(t)=\begin{cases} p\,t,&\text{if $p\geq 2$};\\
2^{4-p}\,\max(t^{p-1},t),&\text{if $1<p<2$.}
\end{cases}
\end{equation}
We note that its inverse function is
\begin{equation}\label{eq:derivative_diff6}
g^{-1}(t)=\begin{cases} \frac{1}{p}\,t,&\text{if $p\geq 2$};\\ \min(2^{p-4}\,t,(2^{p-4}\,t)^{\frac{1}{p-1}}),&\text{if $1<p<2$.}
\end{cases}
\end{equation}

Applying Lemma~\ref{lemma:deriative} with $f$ and $g$ as above and $x_0=0$, we prove \eqref{eq:gamma1} as follows.
We denote $c_1=\|E_{\mu}\left(\tr(\bC\bV\mathbf{D}_{\rmL_1^*,\bx,p}\bU^T)\right)\|_F$. When $p\geq 2$, $f'(x_0)=pc_1$ and
\begin{align*}
&\zeta_1\geq p c_1\cdot   \frac{p c_1}{p} -\int_{0}^\frac{p c_1}{p} p\,x\di x
= \frac{p^2c_1^2}{p}-\frac{p^2c_1^2}{2p}\\=&\frac{p^2c_1^2}{2p}=\frac{p}{2}\|E_{\mu}\left(\tr(\bC\bV\mathbf{D}_{\rmL_1^*,\bx,p}\bU^T)\right)\|_F^2.
\end{align*}
When $1<p<2$, applying
\begin{align*}
&\tr(\bC\bV\mathbf{D}_{\rmL_1^*,\bx,p}\bU^T)
\leq \|\bC\|_F\|\bV\mathbf{D}_{\rmL_1^*,\bx,p}\bU^T\|_F
= \|\bC\|_F\|\mathbf{D}_{\rmL_1^*,\bx,p}\|_F\\\leq&
\|\bC\|_F\|\|P_{\rmL_1^*}(\bx)\|\,\|P_{\rmL_1^{*\perp}}
(\bx)^T \dist(\bx,\rmL(t_1))^{(p-2)}\|\leq 1,
\end{align*}
we conclude that $c_1\leq 1$ and $2^{p-4}\,pc_1\leq 2^{p-3}c_1< 1$. Therefore, $g^{-1}(t)=(2^{p-4}\,t)^{\frac{1}{p-1}}=2^{\frac{p-4}{p-1}} \,t^{\frac{1}{p-1}}$ for $0\leq t \leq pc_1$ and
\begin{align*}
&\zeta_1\geq p c_1\cdot  2^{\frac{p-4}{p-1}} \,(p c_1)^{\frac{1}{p-1}} -\int_{0}^{2^{\frac{p-4}{p-1}} \,(p c_1)^{\frac{1}{p-1}}} 2^{4-p}\,x^{p-1}\di x
= 2^{\frac{p-4}{p-1}} \,(p c_1)^{\frac{p}{p-1}}\\-&p^\frac{1}{p-1} 2^{\frac{p-4}{p-1}}c_1^{\frac{p}{p-1}}=(p-1)p^\frac{1}{p-1} 2^{\frac{p-4}{p-1}}\|E_{\mu}\left(\tr(\bC\bV\mathbf{D}_{\rmL_1^*,\bx,p}\bU^T)\right)\|_F^\frac{p}{p-1}.
\end{align*}

\bibliographystyle{abbrv}
\bibliography{refs_6_19_13}

\end{document}